\documentclass[10pt]{article}
\usepackage[preprint]{Format/tmlr/tmlr}
\usepackage{Format/rbase}
\usepackage{hyperref}
\usepackage{url}
\usepackage{graphicx}
\usepackage{multicol}
\usepackage{makecell}
\usepackage{multirow}
\usepackage{amsmath}
\usepackage{amssymb}
\usepackage{amsthm}
\usepackage{subcaption}
\usepackage{bbm}
\usepackage[table]{xcolor}

\theoremstyle{plain}
\newtheorem{thm}{Theorem}
\newtheorem{prop}{Proposition}

\newtheorem{lem}{Lemma}
\theoremstyle{definition}
\newtheorem{defn}{Definition}
\newtheorem{remark}{Remark}

\newcommand{\gray}{\cellcolor[gray]{0.9}}
\newcommand{\fin}{f^{\textrm{in}}}
\newcommand{\fout}{f^{\textrm{out}}}
\newcommand{\fpre}{f^{\textrm{pre}}}
\newcommand{\fpatch}{f^{\textrm{patch}}}
\newcommand{\diagf}[1]{#1^{\textrm{diag}}}
\newcommand{\SE}{\ensuremath{\mathrm{SE}}}
\newcommand{\e}{\mathbf{e}}
\newcommand{\z}{\mathbf{z}}
\renewcommand{\t}{\ensuremath{\mathbf{t}}}
\renewcommand{\x}{\mathbf{x}}
\renewcommand{\k}{\mathbf{k}}
\renewcommand{\SS}{\mathbb{S}}
\newcommand{\G}{\mathcal{G}}
\newcommand{\I}{\mathcal{I}}
\renewcommand{\Y}{\mathbf{Y}}
\renewcommand{\vec}[1]{\operatorname{vec}\left(#1\right)}

\title{A Geometric Approach to Steerable Convolutions}

\author{
    \name Soumyabrata Kundu \email soumyabratakundu@uchicago.edu \\
    \addr Department of Statistics\\University of Chicago
    \AND
    \name Risi Kondor \email risi@cs.uchicago.edu \\
    \addr Department of Computer Science\\ University of Chicago
}

\def\month{09}  
\def\year{2025} 
\def\openreview{\url{https://openreview.net/forum?id=XXXX}} 

\begin{document}

\maketitle

\begin{abstract}
In contrast to the somewhat abstract, group theoretical approach adopted by many papers, our work provides a new and more intuitive derivation of steerable convolutional neural networks in \(d\) dimensions. This derivation is based on geometric arguments and fundamental principles of pattern matching. We offer an intuitive explanation for the appearance of the Clebsch--Gordan decomposition and spherical harmonic basis functions. Furthermore, we suggest a novel way to construct steerable convolution layers using interpolation kernels that improve upon existing implementation, and offer greater robustness to noisy data.
\end{abstract}

\section{Introduction}

Convolutional neural networks (CNNs) are one of the most successful neural network architectures, and an essential part of modern computer vision systems. As the name implies, the distinguishing feature of CNNs is that the learnable filters of the network are combined with its activations via \emph{convolutions}. This guarantees that the same set of filters are applied to the image at every pixel position, and hence the network can recognize a given feature or constellation of features equally well in any part of the visual field. In modern parlance, we say that convolutional neural networks are \emph{equivariant to translations}. 

Steerable neural networks generalize this framework to also encompass rotations, and hence every point in the output carries additional information about the orientation of features. Consequently, the output of a steerable convolutional layer is a function of both position and orientation. Steerable neural networks offer a clear advantage over classical CNNs by eliminating the need to learn separate filters for each spatial orientation when detecting image constituents like edges and curves. As a whole, the inductive bias of steerable neural networks is more specific and faithful to the structure of most types of natural images than that of simple convolutional networks. This is especially true in applications such as medical imaging, where equivariance to rotations is essential for finding physiologically relevant features.

\subsection{Related Work}

To achieve invariance to specific transformations, one common approach is to apply these transformations to the input and then pool their responses, as shown in studies by \citet{Sohn2012LearningIR, Kanazawa2014LocallySC, zhang2015discriminative}. Recently, however, there has been a growing focus on rotation equivariant CNNs, which are designed to handle general rotations rather than being limited to specific transformations. Steerable filters, originally explored in the 1990s by \citet{FreemanAdelson1991, ReisertBurkhardt}, have seen renewed interest as key components in modern neural networks.

\citet{dieleman2016exploiting} introduced four core operations that seamlessly integrate into conventional convolutional networks by augmenting both the batch and feature dimensions with transformed versions of the input. Subsequently, \citet{marcos2017rotation} enforced rotational equivariance by pooling feature responses across multiple orientations. \citet{cohen2016group} developed a theoretical framework for group equivariant convolutional networks, focusing on discrete transformation groups constrained to the pixel grid, such as the wallpaper group $p4$, which includes translations and $90^\circ$ rotations. Expanding on this, \citet{cohen2017steerable} proposed the use of steerable representations composed of elementary feature types, enabling equivariance to a wider range of orientations. In parallel, \citet{worrall2017harmonic} introduced circular harmonic filters, marking a significant step towards continuous rotation equivariance. This was further refined by \citet{weiler2018learning}, who proposed a weight initialization scheme inspired by \citet{he2015delving}. \citet{weiler2019general} extended the steerable framework to include reflections in addition to rotations, thereby achieving equivariance under the full Euclidean group $\mathrm{E}(2)$.

In three dimensions, \citet{Worrall2018CubeNetET} introduced an architecture that achieves equivariance to translations and right angle rotations. \citet{weiler20183d} broadened the scope by developing models that are equivariant to arbitrary 3D rotations and translations. In a related direction, \citet{thomas2018tensor} proposed Tensor Field Networks, which guarantee equivariance to rotations, translations, and permutations for 3D point cloud data. For data defined on a three dimensional sphere, \citet{CohenSphericalICLR2018} introduced spherical CNNs that are equivariant to rotations. This framework was further refined by \citet{SphericalCNNNeurIPS2018}, who proposed the Clebsch--Gordan non-linearity, a Fourier space non-linearity that mitigates numerical errors caused by repeated transformations between the spatial and spectral domains. More recently, \citet{cesa2022a} generalized these ideas to include rotations, translations, and reflections in arbitrary \(d\) dimensional spaces, providing a principled approach for constructing architectures equivariant to the Euclidean group $\mathrm{E}(d)$.

In a broader and more abstract setting, \citet{krajsek2007unified} introduced a unified theory of steerable filters grounded in Lie group theory, providing a systematic framework for analyzing their structure and behavior. Building on this foundation, \citet{cohen2019general} formulated a general theory of group equivariant convolutional networks on homogeneous spaces, offering a flexible and principled basis for designing equivariant architectures. \citet{kondor18a} formally proved that convolution is the most general form of equivariant linear maps for compact groups, demonstrating that no broader class of equivariant linear operations exists. Extending these theoretical developments, \citet{lang2021a} generalized the classical Wigner--Eckart theorem to the setting of equivariant convolution kernels, advancing the mathematical understanding of symmetry in machine learning.

\subsection{Main Contribution}

Existing works on two and three dimensional steerability \citep{cohen2017steerable, weiler2019general, weiler2018learning, weiler20183d}, along with general theories for \(d\) dimensional steerability \citep{cohen2019general, cesa2022a}, predominantly adopt an algebraic framework to derive the equations for steerable convolutions. While this approach is mathematically rigorous, it often introduces a level of complexity that can hinder the understanding and broader adoption of these methods. 
In this paper, we propose a novel and intuitive method for deriving the equations of steerable convolutions by emphasizing on geometric reasoning over traditional algebraic techniques. This geometric perspective simplifies the underlying concepts, and offers a more nuanced understanding of the operations required for the implementation of steerable neural networks. Starting from the foundational principles of classical CNNs and Fourier transforms, we systematically develop the equations for steerable convolutions, making the methodology more accessible and practical for a wider audience. This derivation also leads to an alternative implementation of steerable convolution which, as demonstrated by our experiments, outperforms existing approaches in the literature and exhibits greater robustness to noise perturbations.

We extend the foundational result of \citet{kondor18a}, which established that convolutions are the most general equivariant linear maps for compact groups, to a broader setting. While elements of this generalization have been discussed in  \citet{weiler20183d} for three dimensional setting and in \citet{cohen2019general} within the context of homogeneous spaces of locally compact groups, our work provides a unified and rigorous formulation. The theorem we present not only consolidates these fragmented insights but also encompasses all prior results as special cases, offering a comprehensive and cohesive extension to the theory.

In the context of three dimensional steerable convolutions, earlier works of \citet{weiler20183d, thomas2018tensor} have recognized spherical harmonics as fundamental building blocks of the architecture. These studies primarily employ spherical harmonics to construct a basis for steerable filters. In contrast, our approach offers a more intuitive interpretation by showing that spherical harmonics naturally arise through the application of the Fourier transform over the rotation group. Furthermore, we extend this perspective to $d$ dimensional spaces, demonstrating that steerable filters can be systematically derived using spherical harmonic basis functions, This leads to a unified framework for developing rotation and translation equivariant architectures in arbitrary \(d\) dimensions.

When applying steerable architectures to real world image data, the discrete nature of images often compromises theoretical equivariance, resulting in small but notable errors. We analyze the \emph{loss of equivariance} in practical implementations, and provide theoretical bounds for these errors, offering deeper insights into the behavior, limitations, and practical applicability of steerable neural networks.
\section{The Principles Behind Classical CNNs}\label{sec: overview CNN}

As a foundation for a ``first principles'' approach to steerable neural networks, we begin by revisiting the fundamental concepts underlying classical two dimensional CNNs. While real world neural networks operate on finite resolution rasterized images, we simplify the mathematical treatment by modeling the inputs, filters, and activations as functions defined on continuous domains. Furthermore, we assume that our input image has a single channel. The generalization to multiple channels is relatively straightforward.

\subsection{Pattern Matching} \label{sec: pattern matching}

In an idealized 2D continuous CNN, the input image is represented as a function $\fin\colon \RR^2 \to \RR$, mapping spatial coordinates to intensity values. The intuitive idea behind CNNs is to find visually meaningful features by comparing local regions of the image against a hierarchy of small, learnable filters \(w\colon\Omega\to\RR\), defined on a compact domain \(\Omega\<=[-h,h]^2\) for some \(h>0\). This pattern matching idea is best expressed as a three stage process. First, around every image location \(\x\) we extract a ``patch function'' \(\fpatch_\x\colon \Omega\to\RR\) defined as
\begin{equation}\label{eq: fpatch0}
 \fpatch_\x(\y)=\fin(\x+\y).
\end{equation}
Second, this patch function is compared with a filter \(w\), by computing their inner product. This gives the so called ``pre-activation function'',
\begin{equation}\label{eq: fpre0}
\fpre(\x)=\inpN{\fpatch_\x,w}_\Omega=\int_{\Omega} \fpatch_\x(\y)\, w(\y)\,d\y. 
\end{equation}
Finally, we get the output by applying a pointwise nonlinearity \(\gamma\), such as the popular ReLU function \(\gamma_{\textrm{ReLU}}(x)=\textrm{max}(0,x)\), and get  
\begin{equation*}
    \fout(\x)=\gamma(\fpre(\x)).
\end{equation*}
Neglecting complications such as pooling and skip connections, classical CNNs simply iterate this pattern matching process over multiple layers. Throughout this paper, we will consistently use the notation $\fin$, $w$, and $\fpatch$, even in more general settings, with the understanding that the underlying domains of the functions will be clear from the context.

\subsection{Convolution on a Group} 

CNNs are named for their close connection between the pattern matching process and the mathematical concept of convolution. Equations \(\eqref{eq: fpatch0}\) and \(\eqref{eq: fpre0}\) imply that \(\fpre\) is the cross-correlation of \(\fin\) and \(w\),
\begin{equation}\label{eq: corr0}
    \fpre(\x)=\int_\Omega \fin(\x+\y)\,w(\y)\,d\y.
\end{equation}
Technically, the convolution of \(\fin\) with \(w\) is given by
\begin{equation}\label{eq: convo0}
    (\fin\ast w)(\x)=\int_\Omega \fin(\x-\y)\,w(\y)\,d\y.
\end{equation}
However, (\ref{eq: corr0}) and (\ref{eq: convo0}) are closely related, and the latter can be transformed into the former by replacing \(w\) with \(\wbar{w}(\y):=w(-\y)\). This conflation of convolution and correlation is common, and algorithms implementing \eqref{eq: corr0} are still called \emph{Convolutional Neural Networks}.

Equation \eqref{eq: convo0} is the \(\RR^2\)-specific formula for convolution. However, the concept of convolution generalizes naturally to any locally compact group \(\G \) (see Appendix \ref{sec: background groups} for definition). For functions \(\fin\) and \(w\) defined on \(\G \), the analogue of \eqref{eq: convo0} would be
\begin{equation}\label{eq: convo_gen}
    (f\ast \wbar{w})(x)=\int_\G \fin(x y^{-1})\,\wbar{w}(y)\,d\mu(y).
\end{equation}
Here, \(\wbar{w}(y) = w(y^{-1})\) and \(\mu\) is the canonical measure on \(\G \), called the Haar measure (see Appendix \ref{sec: background fourier transform} for definition). 

\subsection{Equivariance}\label{sec: equivariance}

The fundamental principle underlying the success of CNNs is that once a set of filters learns to recognize a particular image feature, it should be able to detect that feature equally well at any location in the image plane. Equivalently, if the input to a given CNN layer is translated by some vector \(\t\in \RR^2\), the output must also be translated the same way. This property is called \emph{equivariance} to the translation group \(\RR^2\). 
As we will see now, this idea extends naturally beyond \(\mathbb{R}^2\), allowing for a broader generalization to other domains and symmetry groups.
\begin{defn}\label{defn: group action}
    Given a domain \(\mathcal{X}\) and group \(\G \) acting on it (see Appendix \ref{sec: background group actions} for definition of group action and continuous group action), a natural action of the group on any function on \(\mathcal{X}\) is defined as \([g\cdot f](x) = f(g^{-1}\cdot x)\).
\end{defn}
\begin{defn}
    Suppose \(\mathcal{X}, \mathcal{Y}\) are two domains with group \(\G \) acting on them. Let \(\mathcal{F}(\cdot)\) denote the space of all complex valued functions on a domain. Then \(\phi\colon \mathcal{F}(\mathcal{X})\to \mathcal{F}(\mathcal{Y})\) is said to be \emph{equivariant} if for any \(g\in \G \) and \(f\in \mathcal{F}(\mathcal{X})\), \(g\cdot \phi[f] = \phi[g\cdot f]\).
\end{defn}
Let $\mathcal{X}$ be a measurable space and $\G $ a group acting on $\mathcal{X}$, continuously. Denote by $C_c(\mathcal{X})$ the space of compactly supported continuous functions defined on \(\mathcal{X}\), and by \(C_b(\G )\) the space of all continuous bounded functions on \(\G \). An equivariant linear map $\phi \colon C_c(\mathcal{X}) \to C_b(\G )$ can be defined as
\begin{equation}\label{eq: equivariant linear map}
    \phi[f](g) = \langle g^{-1}\cdot f, {w} \rangle,
\end{equation}
where $w \in \mathcal{L}_1(\mathcal{X})$. Here, $\mathcal{L}_1(\cdot)$ denotes the space of integrable functions on a given domain, and $\langle \cdot, \cdot \rangle$ represents the inner product with respect to the measure on that domain. Proposition \ref{prop: equivariant linear map} confirms that $\phi$ is indeed an equivariant linear map.
\begin{prop}\label{prop: equivariant linear map}
    The map $\phi$ defined in \eqref{eq: equivariant linear map} is an equivariant linear map.
\end{prop}
\begin{proof}
    Linearity of $\phi$ follows directly from the linearity of the inner product. To verify equivariance, observe that for any $g, h \in \G $, we have
    \begin{align*}
           g\cdot \phi[f](h) = \phi[f](g^{-1}h) = \langle (g^{-1}h)^{-1} \cdot f, w\rangle = \langle h^{-1}g \cdot f, w\rangle = \langle h^{-1}\cdot g\cdot f, w\rangle = \phi[g\cdot f](h).
        \end{align*}
    Since $g$ and $h$ were arbitrary, this identity holds for all $g, h \in \G $, completing the proof.
\end{proof}

Equation \eqref{eq: equivariant linear map} offers a way to construct equivariant linear maps. This leads to a natural question: does this representation capture the most general form such maps can take? Under mild regularity assumptions on the space $\mathcal{X}$, the answer is yes. These assumptions are typically satisfied by spaces encountered in practical applications. The theorem below formalizes this statement.

\begin{thm}\label{thme: equivariant linear map}
    Let $\mathcal{X}$ be a locally compact Hausdorff space and $\G $ be a group that acts continuously on it. Suppose $\phi: 
    C_c(\mathcal{X}) \to C_b(\G )$ is a bounded linear map that is equivariant with respect to the action of \(\G \). 
    Then there exists a \emph{unique} complex measure \(\lambda\) such that for all $f \in C_c(\mathcal{X})$ and $g \in \G $,
    \begin{equation*}
        \phi[f](g) = \int_{\mathcal{X}}g^{-1}\cdot f\,d\lambda.
    \end{equation*}
    Furthermore, if \(\mathcal{X}\) is assumed to be \(\sigma\)-compact and equipped with a \(\sigma\)-finite measure, then there exists a \emph{unique} function $w \in \mathcal{L}_1(\mathcal{X})$ and a \emph{unique} singular measure $\nu$ on $\mathcal{X}$ such that for all $f \in C_c(\mathcal{X})$ and $g \in \G $,
    \begin{equation}\label{eq: equivariant theorem}
            \phi[f](g) = \langle g^{-1}\cdot f, w\rangle + \int_\mathcal{X} g^{-1}\cdot f\, d\nu.
    \end{equation}
\end{thm}

\begin{remark}
    Theorem \ref{thme: equivariant linear map} shows that a general equivariant linear map consists of two main parts: an inner product component and a singular component. The singular component accounts for edge cases, such as when \(\phi\) is the identity, but it can be ignored in practice. The construction in equation \eqref{eq: equivariant linear map} corresponds to the special case in which the singular measure \(\nu \equiv 0\). Furthermore, when \(\mathcal{X}\) is the group itself, \eqref{eq: equivariant linear map} reduces to convolution on the group (\ref{eq: convo_gen}). This shows that a classical CNNs can be viewed as convolution on the group \(\RR^2\), or more broadly, as an equivariant linear map under the action of translation on \(\RR^2\).
\end{remark}

\begin{remark}
    \citet{kondor18a} proved a special case of Theorem \ref{thme: equivariant linear map} for compact groups using a Fourier space based argument. However, their method is more intricate and requires stronger assumptions about the underlying function spaces. \citet{cohen2019general} proved a related result in the setting of homogeneous spaces of locally compact groups. Their analysis assumes that the linear maps can be expressed in integral form, which excludes the appearance of singular measures in their formulation. In contrast, Theorem \ref{thme: equivariant linear map} applies more broadly. As established in Lemma \ref{lemma: homogeneous}, if the group $\G $ is locally compact and $\sigma$-compact, then a homogeneous space of \(\G \) (a space where the action of \(\G \) is continuous and transitive) inherits these topological properties, thereby satisfying the assumptions of the theorem. Specific cases involving locally compact groups such as $\SE(3)$ and $\textrm{E}(d)$ were previously addressed by \citet{weiler20183d} and \citet{cesa2022a}, respectively. Theorem \ref{thme: equivariant linear map} offers a cleaner and more general formulation, clearly stating the assumptions on the domains and function spaces, and unifying all of these prior results as special cases.
\end{remark}
\section{Fourier Space Neural Networks}\label{sec: fourier neural network}

The Fourier transform of a function on \m{\RR} is given by 
\begin{equation}\label{eq: fourier transform}
    \h f(\k)=\int f(\x)\,e^{\iota\mathbf{k}\cdot \x}\,d\x.
\end{equation}
In Fourier space, by the well known \emph{convolution theorem}, convolution reduces to 
\begin{equation}\label{eq: Fconvo}
\widehat{(f\ast w)}(\k)=\h f(\k)\,\h w(\k).
\end{equation}
In general, \eqref{eq: Fconvo} is of course much more efficient to compute than \eqref{eq: convo0}. However, the filters in classical CNNs are typically so small that the Fourier approach would afford little or no advantage, so most CNNs still operate in the time domain. In contrast, when functions are defined on nontrivial groups, the computation of convolution in the time domain becomes considerably more challenging. For example, on the rotation group in $d$ dimensions, $\SO(d)$ (see Appendix \ref{sec: background groups} for definition), time domain convolution is not only computationally expensive as it requires an exponential number of discretization points with respect to $d$, but also susceptible to numerical instabilities, particularly due to singularities near the poles. These factors make performing direct convolution operations on such groups particularly challenging, which makes Fourier space methods a more practical option. Remarkably, the Fourier transform generalizes to any compact group \(\G \) in the form 
\begin{equation}\label{eq: Gfourier}
    \h f(\rho)=\int_\G f(g)\,\rho(g)\,d\mu(g). 
\end{equation}
Here, instead of just being complex exponential, \(\rho\) extends over all \emph{irreducible representations (irreps)} of \(\G \) (see Appendix \ref{sec: background group representations} for definition). The irreps of a compact group are finite dimensional unitary complex matrix valued functions \(\rho\colon \G \to\CC^{d_\rho\times d_\rho}\), and hence the Fourier components are now complex matrices instead of a complex scalar. The convolution theorem also generalizes in the elegant form
\begin{equation}\label{eq: Gconvo} 
    \widehat{f\ast g}(\rho)=\h f(\rho)\,\h g(\rho), 
\end{equation}
where the right hand side now involves multiplying \(\h f(\rho)\) and \(\h g(\rho)\) as matrices (the convolution theorem for compact groups is proved in Lemma \ref{lem: group convolution theorem}). Neural networks operating in Fourier space have become very popular \citep{weiler20183d, weiler2018learning, SphericalCNNNeurIPS2018, thomas2018tensor,anderson2019cormorant}. We shall see that it is critical for both understanding and implementing steerable neural networks. 

\subsection{The Clebsch--Gordan Decomposition} 

A key tool required for deriving steerable neural networks is the Clebsch--Gordan decomposition (CG--decomposition). In general, this is used to decompose the tensor product of two irreps of a given group \(\G \) into a direct sum of irreps,
\begin{equation}\label{eq: CGdef}
    \rho_1(g)\<\otimes \rho_2(g)=
    C_{\rho_1,\rho_2} \sqbBig{\bigoplus_{\rho} \bigoplus_{i=1}^{\kappa(\rho_1,\rho_2,\rho)} \rho(g)}C_{\rho_1,\rho_2}^\dagger, \quad g\in  \G .
\end{equation}
Here, \(\kappa(\rho_1,\rho_2,\rho)\) is the multiplicity of \(\rho\) appearing in the CG--decomposition of \(\rho_1\otimes\rho_2\). The unitary matrices \(C_{\rho_1,\rho_2}\) are the so called the Clebsch--Gordan matrices (CG--matrices). While at first sight the CG--decomposition may seem complicated, in some of the relevant special cases it reduces to a simple form. For example, in the case of the two dimensional rotation group \(\SO(2)\), the irreps are all one dimensional and given by \(\rho_k(\theta)=e^{\iota k\theta}\) for \(k\in \ZZ\). In this setting equation \eqref{eq: CGdef} reduces to   
    \begin{equation*}
        \rho_{k_1}(\theta)\,\rho_{k_2}(\theta)=\rho_{k_1+k_2}(\theta).
    \end{equation*}
While the CG--decomposition has appeared in neural nets in different roles, we argue that the fundamental reason that they appear in steerable convolutions is the following result. 
\begin{prop}\label{prop: diagF}
    Let \(\G \) be a compact group and \(f\) a function on the product group \(\G \times \G \) with Fourier transform 
        \begin{equation}\label{eq: diagF1}
        \h f(\rho_1,\rho_2)=\!\!\int_{\G }\int_\G\!\!\!\!f(g_1,g_2)\brbig{\rho_1(g)\otimes \rho_2(g)}\,d\mu(g_1) d\mu(g_2).
        \end{equation}
    Then the Fourier transform of the restriction of \(f\) to the diagonal subgroup,
    \(\diagf{f}(g):=f(g,g)\), is given by
        \begin{align}\label{eq: CG--decomposition}
        \h{\diagf{f}}(\rho)
        &= \sum_{\rho_1,\rho_2}\!\! \sum_{i=1}^{\kappa(\rho_1,\rho_2,\rho)} \fr{d_{\rho_1} d_{\rho_2}}{d_\rho\mu(\G )}\sqbBig{C_{\rho_1,\rho_2}^\dag\, \h f(\rho_1,\rho_2)\,C_{\rho_1,\rho_2}}_{\rho,i},
        \end{align}
    where \([A]_{\rho,i}\) denotes the \(i\)'th diagonal block in the matrix \(A\) corresponding to the 
    irrep \(\rho\). 
\end{prop}
While the CG--decomposition involves the tensor product of group representations, as we will see, it is a crucial component in deriving steerable convolutions. To make the notation more streamlined, we will use the notation \(\substack{\textrm{CG}\\ \rho_1\rho_2\to \rho}\left(\h f(\rho_1,\rho_2)\right)\) to denote the transformation in the right hand side of \eqref{eq: CG--decomposition}.

\begin{remark}\label{remark: CG--matrices}
   In the relevant cases of $\SO(2)$ and $\SO(3)$, the multiplicity of any irrep $\rho$ in the CG--decomposition of $\rho_1 \otimes \rho_2$ satisfies $\kappa(\rho_1, \rho_2, \rho) \leq 1$. This means that the multiplicity of any irrep \(\rho\) appearing in the CG--decomposition has multiplicity exactly one. For simplicity, we adopt this assumption going forward. With this assumption, we denote by $C^{(\rho, \rho_1, \rho_2)} \in \CC^{d_{\rho_1}d_{\rho_2} \times d_\rho}$ the subset of columns of the CG--matrix $C_{\rho_1, \rho_2}$ corresponding to the irrep $\rho$.
\end{remark}
\section{Steerable Convolutions}\label{sec: steerable network}

Steerable neural networks are constructed to enforce equivariance under rigid body transformations. In mathematical terms, this means that in \(d\) dimensional space, the network layers exhibit equivariance to the natural action of the \emph{Special Euclidean group}. Theorem \ref{thme: equivariant linear map} provides the foundation for defining equivariant linear maps in this framework. In this section, we will explore how this general form of equivariant linear maps and convolution theorem can be used to derive the equations for steerable convolutions.

\subsection{The Special Euclidean Group}

The group of encompassing all rotations and translations is referred to as the special euclidean group and is denoted by \(\SE(d)\) in \(d\) dimensions. Any element of \(\SE(d)\) can be expressed as \((\mathbf{t}, R)\), where \(\mathbf{t} \in \mathbb{R}^d\) represents the translational component, and \(R \in \SO(d)\) denotes the rotational component. This group acts on a vector \(\mathbf{x} \in \mathbb{R}^d\) according to the rule \((\mathbf{t}, R) \cdot \mathbf{x} = R \mathbf{x} + \mathbf{t}\), and the group operation is given by  
\begin{equation*}
    (\mathbf{t}', R') (\mathbf{t}, R) = (\mathbf{t}' + R' \mathbf{t}, \, R' R).
\end{equation*}
From an algebraic point of view, \(\SE(d)\) is the so called \emph{semi-direct product} \(\RR^d\rtimes \SO(d)\) (see Appendix \ref{sec: background groups} for more details).
In line with Definition \ref{defn: group action}, the action of \((\mathbf{t}, R)^{-1}\) on a function \(f\) defined over \(\mathbb{R}^d\) would be
\begin{equation*}
    [(\mathbf{t}, R)^{-1} \cdot f](\mathbf{x}) = f(R \mathbf{x} + \mathbf{t}).    
\end{equation*}
On the other hand, when \(f\) is defined on the group itself, the action becomes 
\begin{equation*}
    [(\mathbf{t}, R)^{-1} \cdot f](\mathbf{x}, R') = f((\mathbf{t}, R)(\mathbf{x}, R')) = f(R \mathbf{x} + \mathbf{t}, R R').
\end{equation*}

\subsection{First Layer}\label{sec: first layer}

We will now derive the equations of steerable convolutional networks, by applying the general equivariant linear map \eqref{eq: equivariant linear map} in the context of the group \(\SE(d)\).
The input field \(\fin\) is a function defined on \(\RR^d\). Similar to classical CNNs, the weight function \(w\) is assumed to be compactly supported on \(\Omega := [-h,h]^d\). 
With this setup, the equivariant linear map \eqref{eq: equivariant linear map} becomes
\begin{equation}\label{eq: nsteer1}
    \fpre(\x, R) = \langle {(\x, R)^{-1}}\cdot\fin, w\rangle =  \int_{[-h,h]^d} \fin(\x + R\y)w(\y)d\y.
\end{equation}
Equation \eqref{eq: nsteer1} captures the core operation of steerable convolutions. The key difference from classical CNNs is evident in this formulation: instead of applying the filter to a single patch at each spatial location $\x \in \RR^d$, steerable convolutions apply the filter to rotated versions of the patch, for \emph{every} rotation $R \in \SO(d)$.
Specifically, in two dimensions, the pre-activation function simplifies to
 \begin{equation}\label{eq: fpatch1}
     \fpre(\x,\phi) = \int_{[-h,h]^2} \fin(\x + R_\phi\y)w(\y)d\y,
\end{equation} 
where \(R_\phi\) is the standard 2D rotation matrix
\begin{equation*}
    R_\phi = 
    \begin{bmatrix}
        \cos\phi & -\sin\phi \\
        \sin\phi & \cos\phi
    \end{bmatrix}.
\end{equation*}
Intuitively, \(\fpre(\x, \phi)\) quantifies how well the filter \(w\) matches the input \(\fin\) at location \(\x\) for each orientation \(\phi\). This generalizes the pattern matching mechanism in classical CNNs: a single filter \(w\) is not only applied to all spatial patches, but also to their rotated versions at every possible angle. This added rotational dimension allows steerable neural networks to effectively capture features invariant or equivariant to rotations (see Figure \ref{fig: CNNconvo}).

\begin{figure}
     \centering
     \begin{subfigure}[b]{0.45\textwidth} 
         \centering
         \includegraphics[scale=0.37]{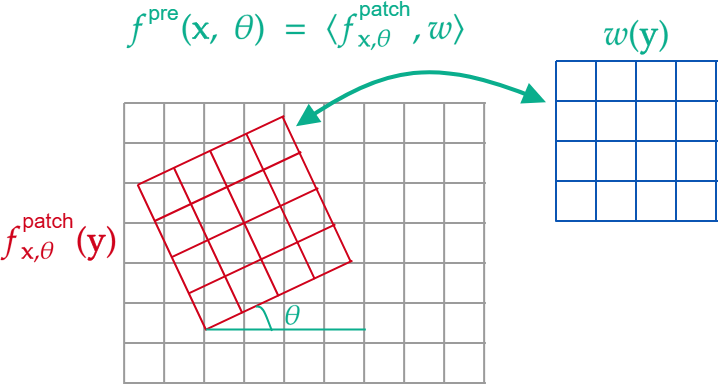}
         \caption{}
         \label{fig: CNNconvo} 
     \end{subfigure}
     \hfil
     \begin{subfigure}[b]{0.45\textwidth}
         \centering
         \includegraphics[scale=0.37]{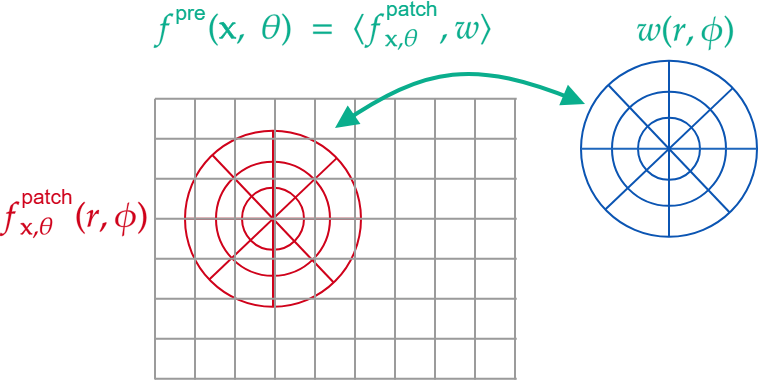}
         \caption{}
         \label{fig: polar} 
     \end{subfigure}
     \hfill
\caption{Illustration of the core principle of steerable convolutions in two dimensions. For each pixel location $\x$, the filter is matched against the input at all possible orientations $\theta$. (a) The patch grid $\fpatch_{\x,\theta}$ is generally not aligned with the underlying image grid. (b) To address this, both the patches centered at $\x$ and the filters are expressed in polar coordinates $(r,\theta)$.}
\end{figure}

Generalizing the pattern matching principle to the first layer of a steerable convolutional network leads to the expression in \eqref{eq: nsteer1}. In practice, when working with finite resolution images represented on a square grid, this operation becomes more involved. The added difficulty arises from the need to interpolate values from the original image grid to a grid at arbitrary rotations \(R \in \SO(d)\).
To simplify this process, one can transition from square filters to filters supported on a ball of radius \(h\), expressing both the input and filter in polar coordinates (see Figure \ref{fig: polar}). In this representation, any \(\y \in \RR^d\) can be written as \(\y = rR'\e\), for some \(r \geq 0\), \(R' \in \SO(d)\) and a fixed reference vector \(\e \in \RR^d\). The patch function in polar coordinates is then defined as \(\fpatch_\x(r, R') = \fin(\x + rR'\e)\).
Using this patch function, the expression in \eqref{eq: nsteer1} can be simplified to
\begin{equation}\label{eq: fpre first}
    \fpre(\x, R) = \int_0^h\int_{\SO(d)} \fpatch_\x(r,RR'^{-1})\wbar{w}(r,{R'}) r^{d-1} \, d\mu(R')\, dr,
\end{equation}
where \(\wbar{w}(r, {R'}^{-1}) = w(rR'\e)\). The equality in \eqref{eq: fpre first} is a direct consequence of Lemma \ref{lem: cartesian to polar} and \ref{lem: inverse haar}. Note that, for a fixed \(r\), the inner integral is a convolution over the group \(\SO(d)\). This observation naturally motivates a Fourier approach, as convolutions in the angular domain is more efficient in Fourier space.
Applying the convolution theorem, \(\fpre(\x, R)\) can be represented in Fourier space as
\begin{equation}\label{eq: fpre first fourier}
    \h\fpre(\x, \rho) = \int_0^h  \h {\fpatch_{\x}}(r,\rho)\,\h{\wbar{w}}(r,\rho) \,r^{d-1}\, dr,
\end{equation}
where the Fourier transform of the patch function \(\fpatch_\x\) is given by
\begin{equation}\label{eq: fpatch first fourier}
     \h {\fpatch_{\x}}(r,\rho) = \int_{\SO(d)} \fpatch_{\x}(r, R)\rho(R)d\mu(R).
\end{equation}

\begin{remark}
    The group \(\SE(d)\) itself is not compact. Without compactness, we cannot ensure that the irreps are finite dimensional. However, the rotation component \(\SO(d)\) within \(\SE(d)\) is indeed a compact group. In the derivation of steerable convolutional layers, we capitalize on this fact and employ the convolution theorem solely on the angular component. This results in a mixed Fourier approach, where convolution in the translation component is performed in the time domain, while the convolution in the angular component is carried out in Fourier domain.
\end{remark}

\subsection{Higher Layers}\label{sec: higher layers}

The output of \eqref{eq: nsteer1} is a function defined on both position \(\x\) and rotation \(R\), meaning that in subsequent layers, both the input field \(\fin\) and the filter \(w\) are functions defined on \(\SE(d)\). To proceed, we need to establish a compact support for the filter \(w\). This can be achieved by setting the support to \(\Omega := [-h, h]^d \times \SO(d)\), which is compact since \(\SO(d)\) itself is a compact group. Under this assumption, the equivariant linear map for higher layers takes the form
\begin{align}\label{eq: nsteer2}
     \fpre(\x, R) 
     &=  \langle {(\x, R)^{-1}}\cdot\fin, w\rangle\nonumber\\
     &= \int_{\Omega} \fin(\x+R\y, RR')w(\y, R')\,d\mu_{\SE(d)}(\y,R') \nonumber\\
     &= \int_{[-h,h]^d}\int_{\SO(d)} \fin(\x+R\y, RR')w(\y, R')\,d\mu(R')\,d\y.
\end{align}
The last equality follows from Lemma \ref{lemma: SE(d) haar}. Similar to the first layer, we will express \eqref{eq: nsteer2} in Fourier space. By leveraging the compactness of \(\SO(d)\), the convolution theorem can be applied to the angular component, leading to a Fourier space representation that simplifies computations. The details of this transformation and its implications for steerable convolutions are provided below.

We begin by defining the patch function at each position \(\x\). Since the input \(\fin\) already includes a rotational component, the patch function in polar coordinates involves two rotation arguments: \(\fpatch_{\x}(r, R', R'') = \fin(\x + rR''\e, R')\). Using this formulation, equation \eqref{eq: nsteer2} can be simplified to
\begin{equation}\label{eq: fpre higher}
    \fpre(\x, R) = \int_0^h \int_{\SO(d)}\int_{\SO(d)} \fpatch_\x(r, RR'^{-1}, RR''^{-1}) \wbar{w}(r,R',R'') r^{d-1}\,d\mu(R')\,d\mu(R'')\,dr,
\end{equation}
where \(\wbar{w}(r, R'^{-1}, R''^{-1}) = w(rR'\e, R'')\). Although equation \eqref{eq: fpre higher} bears a resemblance to a convolution over the group \(\SO(d) \times \SO(d)\), it is not quite the same thing. The inner product calculation involves integrating over \(R'\) and \(R''\), which are coupled through \(R\). As a result, applying the convolution theorem is not straightforward. To address this, we introduce ``shadow rotations'' \(R_1\) and \(R_2\), and define the quantity
 \begin{equation}\label{eq: q_x}
     q_\x(R_1, R_2) := \int_{0}^{h}\!\int_{\SO(d)}\int_{\SO(d)}
    \fpatch_{\x}(r,R_1R'^{-1}, R_2R''^{-1})\,\wbar{w}(r,R',R'')\,r^{d-1}\,d\mu(R')\,d\mu(R'')\,dr,
\end{equation}
which, for a fixed radius \(r\), can now be interpreted as a convolution on the group \(\SO(d) \times \SO(d)\). In Fourier space, this convolution is expressed as
\begin{equation}\label{eq: q_x hat expression}
    \h{q}_\x( \rho_1, \rho_2) = \int_0^h \h{\fpatch_{\x}}(r, \rho_1, \rho_2) \h{\wbar{w}}(r, \rho_1, \rho_2)\, r^{d-1} \, dr,
\end{equation}
where \(\h{\fpatch_{\x}}(r, \rho_1, \rho_2)\) is given by
\begin{equation}\label{eq: fpatch higher fourier}
    \h{\fpatch_{\x}}(r, \rho_1, \rho_2) = \int_{\SO(d)}\int_{\SO(d)} \fpatch_{\x}(r, R',R'') (\rho_1(R')\otimes \rho_2(R''))\,d\mu(R') \,d\mu(R'').
\end{equation}

Given that \(\fpre(\x, R) = q_\x(R, R)\), it remains to extract \(\h{{\fpre}}(\x, \rho)\) from \(\h q_{\x}(\rho_1,\rho_2)\). As Proposition \ref{prop: diagF} suggests, this operation can be carried out efficiently using the CG--decomposition of $\h q_\x$, yielding
\begin{equation}\label{eq: fpre higher fourier}
    \h{\fpre}(\x,\rho) = \substack{\textrm{CG} \\ \rho_1, \rho_2 \to \rho} \left(\h q_\x(\rho_1, \rho_2)\right) =  \int_0^h \substack{\textrm{CG} \\ \rho_1, \rho_2 \to \rho} \left( \h{\fpatch_{\x}}(r, \rho_1, \rho_2) \h{\wbar{w}}(r, \rho_1, \rho_2)\right) \,r^{d-1}\, dr,
\end{equation}
Recall the notation $\substack{\textrm{CG} \\ \rho_1, \rho_2 \to \rho}(\cdot)$, introduced in Remark \ref{remark: CG--matrices}, which denotes the CG--decomposition \eqref{eq: CG--decomposition}.

\begin{remark}\label{remark: steerable constraint}
    From Proposition \ref{prop: equivariant linear map}, we know that when an element \((\t,R)\in\mathrm{SE}(d)\) acts on the input map as \(\fin\mapsto(\t, R)^{-1} \cdot \fin = \fin(R\x+\t)\), the resulting pre-activation function transforms accordingly as
    \begin{equation*}
        \fpre(\x,R) \mapsto [(\t, R)^{-1} \cdot \fpre](\mathbf{x}, R') = \fpre(R\mathbf{x} + \t, RR').
    \end{equation*}
    Since the rotational component is expressed in Fourier space, the translation property of the Fourier transform ensures that \(\h{\fpre}\), for each irrep \(\rho\), transforms as
    \begin{equation}
        \h\fpre(\x, \rho) \mapsto [(\t,R)^{-1}\cdot \h\fpre](\x,\rho) = \rho(R)^\dagger \h\fpre(R\x+\t, \rho).
    \end{equation}
\end{remark}

\subsection{Comparison to Previous Derivations}\label{sec: prev work}
In this section, we will discuss the approach taken by previous works in deriving steerable convolutions. Previous works, such as those by \citet{cohen2017steerable, weiler2018learning, thomas2018tensor, weiler2019general, cesa2022a} introduce steerable convolutions as convolving ``steerable features'' with ``steerable filters''. First, let us define these concepts.

\begin{defn}\label{defn: steerable feature map}
    A feature map \(f:\mathbb{R}^d \to \mathbb{C}^{d_{\rho}}\) is said to be \(\rho\)-steerable if, under the action of any \((\t, R) \in \mathrm{SE}(d)\), it transforms according to \([(\t, R)^{-1} \cdot f](\mathbf{x}) = \rho(R)^{\dagger} f(R\mathbf{x} + \t)\).
\end{defn}

\begin{defn}\label{defn: steerable filter}
A filter \(K:\mathbb{R}^d \to \mathbb{R}^{d_\rho \times d_{\rho_1}}\) is called a \((\rho, \rho_1)\)-steerable filter if, for any \(R \in \mathrm{SO}(d)\), it satisfies the condition
\(K(R\mathbf{y}) = \rho(R) K(\mathbf{y}) \rho_1(R)^{\dagger}\).
\end{defn}

In essence, a \((\rho, \rho_1)\)-steerable filter ``steers'' \(\rho_1\)-steerable features to \(\rho\)-steerable features. Specifically, when a \(\rho_1\)-steerable feature map \(f\) is convolved with a \((\rho, \rho_1)\)-steerable filter \(K\), the output feature map is \(\rho\)-steerable:
\begin{align*}
    [K * (\t,R)^{-1}\cdot f](\x)
    &= \int_{\RR^d} K(\y) \rho_1(R)^\dagger f(R\x+\t - R\y,\rho_1)d\y
    = \int_{\RR^d} K(R^{-1}\y) \rho_1(R)^\dagger f(R\x+\t - \y,\rho_1)d\y\\
    &= \int_{\RR^d} \rho(R)^\dagger K(\y)  f(R\x+\t - \y,\rho_1)d\y
    = (\t,R)^{-1}\cdot [K * f](\x).
\end{align*}
The discussion in Remark \ref{remark: steerable constraint} establishes that the pre-activation function in Fourier space, $\widehat{\fpre}(\cdot, \rho)$, is $\rho$-steerable. The geometric derivation of steerable convolutions, presented in Section \ref{sec: steerable network}, shows that steerable feature maps can be understood as functions defined on the group $\mathrm{SE}(d)$, with the rotational component expressed in the Fourier domain. While the concept of steerable convolutions, when viewed through the lens of steerable features and filters, simplifies the derivation, it does not offer much insight into their underlying principles. Our derivation demonstrates that under the hood, we are essentially matching the same filter across all possible rotations. This can be understood as a generalization of the core principles of CNNs to the group \(\SE(d)\).
\section{Implementation}\label{sec: steerable implementation}

In section \ref{sec: steerable network}, we discussed how the pattern matching argument in CNNs can be replaced by a direct appeal to the standard method of generalizing equivariant linear maps \eqref{eq: equivariant linear map} to \(\SE(d)\), leading to the equations of steerable convolutions in first and higher layers. While this derivation provides a solid theoretical basis, it offers limited guidance for practical network design and implementation. In this section, we will build on this foundation to expand on how to implement steerable convolutional layers.

\subsection{Interpolation}\label{sec: interpolation}

In practical implementations, the input to a convolution layer are finite resolution rasterized images defined on a grid. These images can be thought of as functions on \( \ZZ^d \), which are zero outside a finite set. As discussed in Section \ref{sec: steerable network}, these images need to be interpolated from a Cartesian coordinate system to a spherical coordinate system. Standard interpolation techniques, such as linear and cubic interpolation, fundamentally operate by computing a weighted sum of values from a grid. Rather than focusing on a specific technique, we take a more general perspective and define interpolation in terms of a discrete kernel, as outlined below.
\begin{defn}\label{defn: interpolation}
    An interpolation kernel is a function \( \I: \RR^d \times \ZZ^d \to \RR \) that satisfies the following properties:
    \begin{enumerate}
        \item[(a)] For any function \( f: \ZZ^d \to \mathbb{C} \), the interpolated function \( \I[f]: \RR^d \to \mathbb{C} \) is given by \( \I[f](\x) = \sum_{\y \in \ZZ^d} \I(\x, \y) f(\y)\).
        \item[(b)] For any \( \x \in \RR^d \) and \( \y, \mathbf{z} \in \ZZ^d \), the kernel satisfies \( \I(\x + \mathbf{z}, \y + \mathbf{z}) = \I(\x, \y) \).
        \item[(c)] For any \( \y \in \ZZ^d \), the function \(\I(\cdot, \y)\) is \(\alpha\)-H\"older for \(\alpha \in [0,1]\), i.e., there exists constants \(M,C>0\) such that for all \(\x,\z\in \RR^d\), \(|\I(\x, \y)|\leq M\) and \(|\I(\x,\y) - \I(\z,\y)|\leq C\|\x-\z\|_2^\alpha\).
        \item[(d)] For any \( \x \in \RR^d \), the quantity \( \sum_{\y \in \ZZ^d} |\I(\x, \y)|\) is uniformly bounded above by some constant.
    \end{enumerate}
\end{defn}
Interpolation from a Cartesian grid to a spherical grid introduces errors that result in loss of equivariance. To quantify this loss, we define the following quantity.
\begin{defn}\label{defn: interpolation error}
    Suppose \(\G \) is a group acting on \(\RR^d\). For any \(g\in \G \), define
    \begin{equation}
        \Delta(g) := \sup_{\x\in \RR^d} \sup_{\y\in \ZZ^d} \left\vert \I(g\cdot \x, \y)  - \sum_{\z\in \ZZ^d}\I(\x,\z)\I(g\cdot \z, \y) \right\vert.
    \end{equation}
    The interpolation is said to be exact with respect to the action of \(g\), if \(\Delta(g) = 0\).
\end{defn}

In essence, \(\Delta(g)\) measures the maximum potential information loss incurred when evaluating an input function on the grid obtained by translating it by \(g\) through interpolation. More precisely, when we evaluate the \(g\)-translated version of a function \(f\in\mathcal{L}_1(\ZZ^d)\) at a point \(\x\), the deviation from the value obtained by directly evaluating the original function at \(g \cdot \x\) is bounded above by \(\Delta(g)\|f\|_1\):
\begin{align*}
     \Bigg|\I[f](g\cdot \x) - \I[g^{-1}\cdot \I[f]](\x)\Bigg| 
    =& \left\vert \sum_{\y\in \ZZ^d} \I(g\cdot \x,\y)f(\y) - \sum_{\z\in \ZZ^d} \I(\x,\z) [g^{-1}\cdot \I[f]](\z) \right\vert\\
    =& \left\vert \sum_{\y\in \ZZ^d} \I(g\cdot \x,\y)f(\y) - \sum_{\z\in \ZZ^d} \I(\x,\z) \I[f](g\cdot \z) \right\vert\\
    =& \left\vert \sum_{\y\in \ZZ^d} \I(g\cdot \x,\y)f(\y) - \sum_{\y,\z \in \ZZ^d} \I(\x,\z) \I(g\cdot \z, \y) f(\y) \right\vert\\
    =& \left\vert \sum_{\y\in \ZZ^d} \left(\I(g\cdot \x,\y) - \sum_{\z\in \ZZ^d}\I(\x,\z) \I(g\cdot \z, \y)\right)f(\y) \right\vert\\
    \leq& \Delta(g) \|f\|_1
\end{align*}
Thus, \(\Delta(g)\) provides a uniform upper bound on the interpolation error induced by the group action.

\begin{figure}
    \centering
    \includegraphics[scale=0.5]{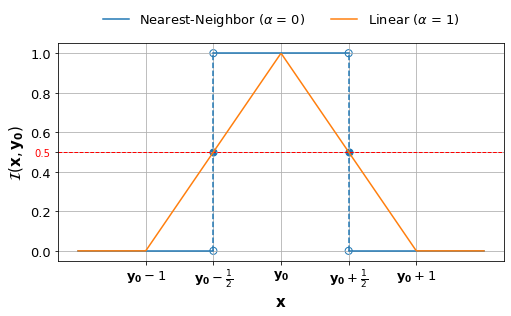}
    \caption{Illustration of the interpolation kernel \(\I\) at a fixed \(\y_0\) in one dimension. The nearest-neighbor interpolation is an example of a discontinuous kernel, hence corresponding to the case \(\alpha=0\) in Definition \ref{defn: interpolation}(c). In contrast, linear interpolation varies linearly between adjacent grid points, hence corresponding to the case \(\alpha=1\). In \(d\) dimensions, the kernel \(\I\) is constructed by multiplying the kernel values along each dimensions.}
    \label{fig: intepolation kernel}
\end{figure}

\subsection{Fourier Transform}\label{sec: SHT}

The interpolation into spherical coordinates is followed by applying the Fourier transform over the rotation group. At first glance, the idea of integrating over a (non-commutative) group might appear intimidating. However, on closer examination, we see that for a fixed radius \(r\), \(\fpatch_\x(r,\cdot)\) is actually a function on the \(d\) dimensional sphere \(\SS^{d-1}\). This observation allows us to use \emph{spherical harmonic} basis functions (see Appendix \ref{sec: background spherical harmonics} for definition) to simplify the integral.

In \(d\) dimensions, spherical harmonics form a basis for \(\mathcal{L}_2(\SS^{d-1})\), the space of all square integrable function on \(\SS^{d-1}\), and are closely related to the symmetric traceless irreps of \(\SO(d)\) (see Appendix \ref{sec: background spherical harmonics} for details). A spherical harmonic associated with an irrep \(\rho\) of \(\SO(d)\) is a function \(\Y^{(\rho)}:\SS^{d-1}\to \CC^{d_\rho}\) that satisfies
\begin{equation}\label{eq: spherical harmonics}
    \Y^{(\rho)}(Rs) = \rho(R)\Y^{(\rho)}(s)
\end{equation}
for all \(s \in \SS^{d-1}\) and \(R \in \SO(d)\). The relationship between \(\rho\) and \(\Y^{(\rho)}\) in the above equation simplifies the Fourier transform over the group, as formalized in Proposition \ref{prop: SHT}. 
\begin{prop}\label{prop: SHT}
     Fix \(\e\in \SS^{d-1}\). For any \(f\in \mathcal{L}_2(\SS^{d-1})\),
    \begin{equation*}
        \int_{\SO(d)} f(R\e)\rho(R)d\mu(R) = \begin{cases}
            \bigg(\int_{\SS^{d-1}}f(s) \Y^{(\rho)}(s) \,d\sigma(s)\bigg)\Y^{(\rho)}(\e)^\dagger &\rho \text{ is a symmetric traceless irrep},\\
            0 & \text{otherwise},
        \end{cases}
    \end{equation*}
    where \(\sigma\) denotes the normalized surface measure on \(\SS^{d-1}\), and the spherical harmonics are normalized such that \(\|\Y^{(\rho)}(s)\|_2 = 1\), for all \(s\in \SS^{d-1}\).
\end{prop}
The expression inside the parentheses is known as the \emph{spherical harmonic transform} of \(f\) (see Appendix \ref{sec: background spherical harmonics} for more details). This result significantly reduces complexity, as the surface measure on the sphere is far simpler to work with than the Haar measure on \(\SO(d)\). Furthermore, Proposition \ref{prop: SHT} highlights that non-symmetric irreps carry no information about \(\fpatch_\x\) and can be ignored in practice. Interestingly, in two and three dimensions, all irreps are symmetric traceless, so this distinction is not needed. From here on, we will assume the spherical harmonics are normalized such that \(\|\Y^{(\rho)}(s)\|_2 = 1\), for all \(s\in \SS^{d-1}\).

\begin{remark}
    Spherical harmonics are fundamental to the construction of steerable filters. Previous works, such as those by \citet{weiler20183d, thomas2018tensor} in three dimensions and \citet{cesa2022a} in general \(d\) dimension, introduced spherical harmonics as basis functions for designing steerable filters. Proposition \ref{prop: SHT} explains their natural emergence as a direct consequence of applying the Fourier transform on the group \(\SO(d)\), providing a deeper connection to the symmetry properties of the group.
\end{remark}

\subsection{Discretization}\label{sec: discritization}

The discussion in Sections \ref{sec: first layer} and \ref{sec: higher layers} demonstrated that steerable convolution layers involve integration over both the radial and rotation components. As discussed in Section \ref{sec: SHT}, the integration over the rotation group can be simplified to an integration over the \(d\) dimensional sphere $\SS^{d-1}$. In practice, these integrals are approximated numerically through discretization. For the radial part, the interval $(0, h]$ can be partitioned into $n_r$ equally spaced points. To approximate the spherical integration, a discretization of $\SS^{d-1}$ is necessary.

A point on $\SS^{d-1}$ can be parameterized by angular coordinates $\theta \in \Theta^{d-1} := [0,\pi]^{d-2} \times [0,2\pi)$ via the standard mapping
\begin{equation}\label{eq: sphere parameter}
    s_i(\theta) := 
    \begin{cases}
        \left(\prod_{j=1}^{i-2} \sin\theta_j\right)\cos\theta_i & 1 \leq i \leq d-1,\\
        \,\,\,\,\prod_{j=1}^{d-1} \sin\theta_j & i = d.
    \end{cases}
\end{equation}
With this parametrization, the integral of any bounded function $f: \SS^{d-1} \to \CC$ can be expressed as
\begin{align}
    &\int_{\SS^{d-1}} f(s)\, d\sigma(s) = \frac{1}{\mathcal{A}(\SS^{d-1})}\int_{\Theta^{d-1}} f(s(\theta)) \omega(\theta)\, d\theta, \label{eq: sphere integral}\\
    &\omega(\theta) := \prod_{i=1}^{d-2} (\sin \theta_i)^{d-1-i},\label{eq: quadrature}
\end{align}
where \(\sigma\) denotes the normalized surface measure on \(\SS^{d-1}\), and \(\mathcal{A}(\SS^{d-1}) := \int \omega(\theta)d\theta\) is the area of \(\SS^{d-1}\).
We use this parameterization to numerically approximate integrals over the sphere by constructing a uniform angular grid, obtained by dividing the intervals $[0,\pi]$ and $[0,2\pi)$ into $n_a$ equally spaced points. This yields the discretized grid
\begin{equation}\label{eq: sphere grid}
   \Theta_{n_a}^{d-1} := \left\{\left( \tfrac{\pi}{n_a} \left(a_1+\tfrac{1}{2}\right),\, \tfrac{\pi}{n_a} \left(a_2+\tfrac{1}{2}\right),\, \dots,\, \tfrac{\pi}{n_a} \left(a_{d-2}+\tfrac{1}{2}\right),\, \tfrac{2\pi}{n_a} a_{d-1} \right):\; a_i \in \{0,1,\dots,n_a-1\}\right\}.
\end{equation}
The points are shifted from the origin to avoid poles. Using this grid, the the spherical harmonic transform of any \(f\in \mathcal{L}_2(\SS^{d-1})\) can be approximated by a Riemann sum:
\begin{equation}\label{eq: sphere integral approx}
    \int_{\SS^{d-1}} f(s) \Y^{(\rho)}(s)\, d\sigma(s) 
    \approx \frac{2\pi^{d-1}}{n_a^{d-1}\mathcal{A}(\SS^{d-1})} \sum_{\theta \in \Theta_{n_a}^{d-1}} f(s(\theta))\Y^{(\rho)}(s(\theta)) \, \omega(\theta).
\end{equation}

\begin{remark}
    This uniform angular grid can be used to approximate any integral over the sphere, not only for computing spherical harmonic transforms. Nonetheless, the choice of quadrature weights can be optimized to improve the accuracy of spherical harmonic transform approximations. In three dimensions, for example, \citet{DRISCOLL1994202} proposed the quadrature
    \begin{equation}\label{eq: DH quadrature}
        \omega_{\textrm{DH}}(\theta) := \frac{4}{\pi}\sin\theta_1 \sum_{k=0}^{\frac{n_a}{2}-1} \left(\frac{\sin((2k+1)\theta_1)}{2k+1}\right)
    \end{equation}
    for the grid \(\Theta^{2}_{n_a}\), when \(n_a\) is even, that yields the exact spherical harmonic transform exactly for band limited functions. This was used by \citet{CohenSphericalICLR2018} and \citet{SphericalCNNNeurIPS2018} in the context of Spherical CNNs. The quadratures \eqref{eq: quadrature} and \eqref{eq: DH quadrature} are asymptotically equivalent as the summation in \eqref{eq: DH quadrature} converges to \(\frac{\pi}{4}\).
\end{remark}

\subsection{Implementation in general d dimensions}

The input to a steerable convolutional layer is first interpolated onto spherical coordinates at each grid location, followed by a spherical harmonic transformation and finally CG--decomposition (in higher layers). Since all these operations are linear operations, they can be combined into a single ``steerable filter''. This steerable filter, combined with learnable weights, can then be applied to the input to generate the pre-activation function. In the following sections, we formalize these concepts and outline a framework for implementing steerable convolutional layers.

\subsubsection{First Layer}

Let \(\fin: \ZZ^d \to \mathbb{C}\) denote the input function for a single channel. The output of the steerable convolution corresponding to an irrep \(\rho\), consists of \(d_\rho\) columns. For simplicity, in a neural network implementation, these columns are combined into the channel dimension, allowing us to represent the output as \(\h{\fpre}(\cdot, \rho) \in \CC^{d_\rho}\) for a single channel. Although \(\SO(d)\) has infinitely many irreps in theory, we restrict ourselves to a finite subset \(\mathcal{F}_{\SO(d)}\)in computation. With this setup, the theorem below provides a framework for implementing the first layer of a steerable convolutional network. As described in Section~\ref{sec: discritization}, the integrals are evaluated by discretizing \((0,h]\) into \(n_r\) grid points for the radial component, and using the grid \(\Theta_{n_a}^{d-1}\) for the spherical component.
\begin{thm}\label{thm: first layer}
    Suppose \(\fin:\ZZ^d \to \CC\) be an input map. The pre-activation \(\h\fpre\) for the first layer of steerable convolutional network can be computed as
    \begin{equation}\label{eq: implementation first layer}
        \h\fpre(\x, \rho) =  \sum_{\y\in \ZZ^d}\left(\sum_{r=1}^{n_r} w_r^{(\rho)}  M_{r}^{(\rho)}(\y)\right) \fin(\x+\y),
    \end{equation}
    where \({w}_r^{(\rho)} \in \CC\) are learnable weights, \(\rho\in \mathcal{F}_{\SO(d)}\), and \(M_{r}^{(\rho)}\in \CC^{d_\rho}\) is defined as
    \begin{equation}\label{eq: steerable filter basis first layer}
        M_{r}^{(\rho)}(\y) := \frac{r^{d-1}}{n_r^dn_a^{d-1}}\sum_{\theta\in\Theta_{n_a}^{d-1}} \I\left(\frac{rh}{n_r}s(\theta), \y\right)\Y^{(\rho)}(s(\theta)) \,\omega(\theta)
    \end{equation}
    for \(1\leq r\leq n_r\).
\end{thm}
The function \(M_r^{(\rho)}\) is computed during a precomputation phase \citep{cohen2017steerable, weiler20183d, weiler2018learning, cesa2022a}. The forward pass consists of assembling the convolution kernel by combining learned weights with precomputed filters, and then performing convolution for each Fourier component.  All these steps are parallelizable and can be efficiently executed on GPUs.

\subsubsection{Higher Layers}
The higher layers build upon the operations of the first layer by incorporating the CG--decomposition. As this decomposition is also a linear operation, it can similarly be integrated into the steerable filter during the precomputation stage. As a result, the higher layers follow a similar process to the first layer, as established by the following theorem. Recall the CG--matrices \(C^{(\rho, \rho_1, \rho_2)}\) introduced in Remark \ref{remark: CG--matrices}.
\begin{thm}\label{thm: higher layer classical}
    Let \(\fin(\cdot, \rho): \ZZ^d \to \mathbb{C}^{d_{\rho}}\) be the input map corresponding to the irrep \(\rho\in \mathcal{F}_{\SO(d)}\) and single input channel. Then, \(\h\fpre\) for the higher layers of steerable convolutions is given by
        \begin{equation}\label{eq: implementation higher layers}
            \h\fpre(\x, \rho) =  \sum_{\y \in \ZZ^d}  \sum_{\rho_1\in \mathcal{F}_{\SO(d)}}\left( \sum_{\rho_2\in \mathcal{F}_{\SO(d)}} \sum_{r=1}^{n_r} w_r^{(\rho, \rho_1, \rho_2)} \tilde{M}_{r}^{(\rho, \rho_1, \rho_2)}(\y) \right)\fin(\x + \y, \rho_1),
        \end{equation}
    where \(w_r^{(\rho, \rho_1, \rho_2)}\in \mathbb{C}\) are learnable weights, \(\rho\in \mathcal{F}_{\SO(d)}\), and the \(m^{\text{th}}\) row of \(\tilde{M}_{r}^{(\rho, \rho_1, \rho_2)}(\y)\in \CC^{d_\rho\times d_{\rho_1}}\) is defined as
        \begin{equation}\label{eq: steerable filter basis higher layers}
            \left[\tilde{M}_{r}^{(\rho, \rho_1, \rho_2)}(\y)\right]_{m, \cdot} := \frac{1}{|\mathcal{F}_{\SO(d)}|} M_r^{(\rho_2)}(\y) ^\top  \tilde{C}^{(\rho, \rho_1, \rho_2)}_m.
        \end{equation}
    Here, \(\tilde{C}_m^{(\rho, \rho_1, \rho_2)}\in \mathbb{C}^{d_{\rho_2} \times d_{\rho_1}}\) such that \(\left[C^{(\rho, \rho_1, \rho_2)}\right]_{\cdot, m}^* = \vec{\tilde{C}_m^{(\rho, \rho_1, \rho_2)}}\).
\end{thm}
Similar to the first layer, the steerable filter $\tilde{M}_{r}^{(\rho, \rho_1, \rho_2)}$ is computed during the precomputation phase. The first layer can be interpreted as a special case of steerable convolutions in  higher layer, where the input Fourier cutoff \(\rho_1\) is restricted to the constant representation.

\begin{figure}[t]
    \centering
    \includegraphics[width=\linewidth]{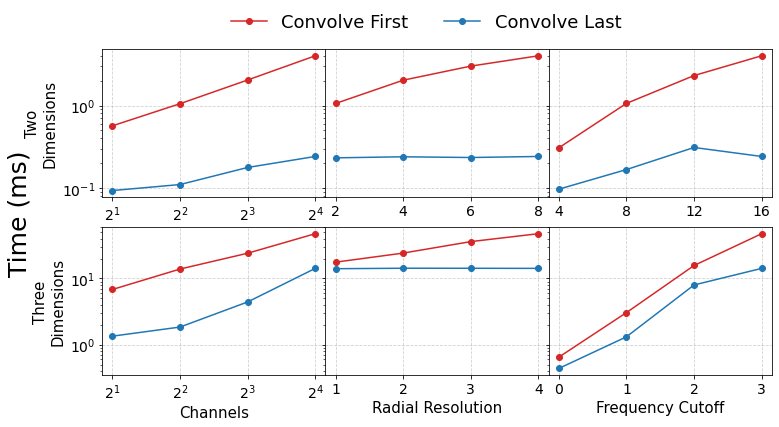}
    \caption{Comparison of the two implementation strategies for steerable convolutions discussed in Remark \ref{remark: time}. Each network consists of two convolutional layers, with runtime measured across varying parameter settings shown on the \(x\)-axis. When one parameter (e.g., “Channels”) was varied, the others (“Radial Resolution” and “Frequency Cutoff”) were fixed at their largest values shown in the plot. Runtimes correspond to inputs of size \(1\times28\times28\) for 2D and \(1\times32\times32\times32\) for 3D, averaged over \(100\) runs (excluding \(5\) burn-ins) and normalized per sample using batchsizes of \(10\) and \(5\) for 2D and 3D, respectively.}
    \label{fig: time}
\end{figure}

\begin{remark}\label{remark: constrained convolution}
    Theorems \ref{thm: first layer} and \ref{thm: higher layer classical} establish that, unlike conventional convolutions, which directly learn a full weight matrix of size \(\CC^{d_\rho \times d_{\rho_1}}\), steerable convolutions parameterize the kernel using a single learnable scalar \(w^{(\rho, \rho_1, \rho_2)}\) that is linearly combined with a fixed matrix \(\tilde{M}_{r}^{(\rho, \rho_1, \rho_2)}(\y) \in \CC^{d_\rho \times d_{\rho_1}}\). Thus steerable convolutions can be viewed as learning convolutional filters constrained to a predefined steerable basis \citep{cohen2017steerable, weiler20183d}.
\end{remark}

\begin{remark}\label{remark: time}
    Because matrix multiplication is associative, the operations in equations \eqref{eq: implementation first layer} and \eqref{eq: implementation higher layers} can be executed in two equivalent ways: either by first convolving the input with the precomputed steerable filters and then applying the learned weights, or by combining the weights with the filters beforehand and convolving the result with the inputs. In practice, the latter approach is both faster and more memory efficient, as demonstrated in Figure \ref{fig: time}, which compares runtimes across various parameter settings. The first method is typically slower because it requires storing large intermediate tensors whose size scales with the radial and angular resolutions, and performing weight multiplications at every spatial location. However, if memory is not a limiting factor and sparse convolution for interpolation and Fast Fourier Transform are used, the first approach has the potential to outperform the second in terms of speed. Implementing these optimizations would require writing custom CUDA kernels, which we leave as an interesting direction for future work.
\end{remark}

\begin{figure}[t]
    \centering
    \includegraphics[width=\linewidth]{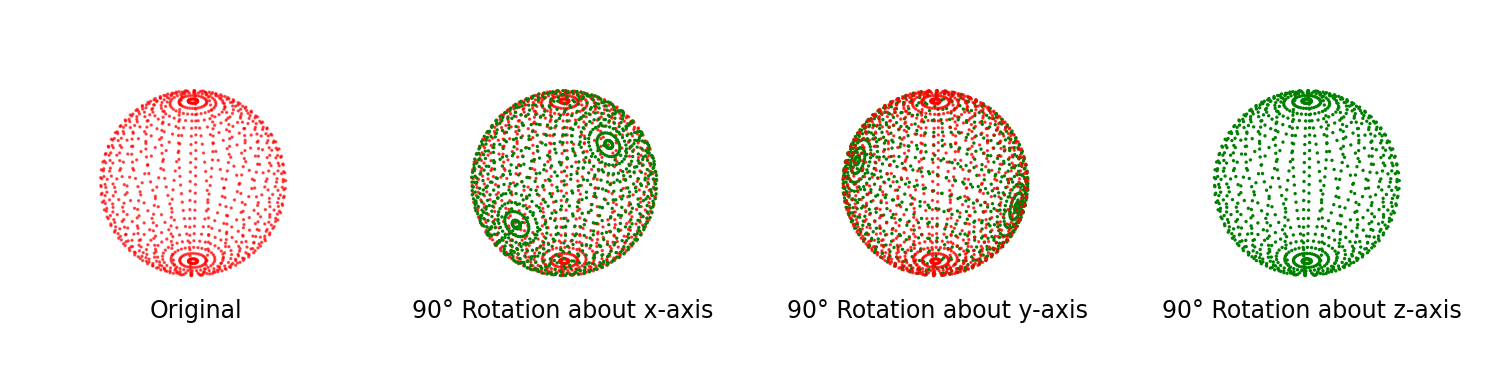}
    \caption{Illustration of a uniform grid on $\SS^2$. The polar and azimuthal angles are uniformly discretized into \(32\) grid points. Only discrete rotations about the $z$ axis preserve the grid structure. Rotations about the $x$ or $y$ axes generally do not map grid points back onto the grid, breaking rotational alignment.
    }
    \label{fig: align}
\end{figure}

\subsection{Loss of Equivariance}\label{sec: loss of equivariance}

While in theory the steerable layers are exactly equivariant, practical implementations deviate from this due to interpolation errors and approximation of integrals by discretization, which result in a loss of equivariance. In this section, we derive bounds that quantify this loss. For brevity, we focus on steerable convolutions in higher layers. However, the results are also applicable to the first layer, as it can be viewed as special cases.

Define the filter \(K^{(\rho,\rho_1)}:\ZZ^d \to \RR^{d_\rho \times d_{\rho_1}}\) as
\begin{equation}\label{eq: steerable filter}
    K^{(\rho,\rho_1)}(\y) := \sum_{\rho_2\in \mathcal{F}_{\SO(d)}}\sum_{r=1}^{n_r} w_r^{(\rho,\rho_1,\rho_2)} \tilde{M}_{r}^{(\rho,\rho_1,\rho_2)}(\y)
\end{equation}
Suppose $\fin(\cdot, \rho): \ZZ^d \to \CC^{d_\rho}$ is an input feature map, corresponding to irrep \(\rho\). According to Theorem \ref{thm: higher layer classical}, the output of the steerable convolution is computed via a discrete cross-correlation, given by
\begin{equation}
    \h{\fpre}(\x, \rho) = \sum_{\rho_1}\sum_{\y \in \ZZ^d} K^{(\rho,\rho_1)}(\y) \fin(\x + \y, \rho_1).
\end{equation}
Since $K^{(\rho,\rho_1)}$ is defined over a discrete grid, it does not strictly satisfy the definition of a steerable filter in the continuous sense (as given in Definition \ref{defn: steerable filter}). However, if a rotation $R \in \SO(d)$ results in exact interpolation ($\Delta(R) = 0$) and maps the spherical grid back onto itself while preserving the quadrature \(\omega\), then $K^{(\rho,\rho_1)}$ can be interpreted as a discrete counterpart to a steerable filter. This relationship is formalized in the following proposition.
\begin{prop}\label{prop: steerable filter}
    Let $K^{(\rho,\rho_1)} : \mathbb{Z}^d \to \mathbb{C}$ be as defined in \eqref{eq: steerable filter}. Suppose $R \in \SO(d)$ satisfies $\Delta(R) = 0$, and the action $s(\theta) \mapsto R\,s(\theta)$ defines a bijection on the discretized angular grid $\Theta_{n_a}^{d-1}$, with the quadrature $\omega$ being invariant under this action. Then, for any $\mathbf{y} \in \mathbb{Z}^d$, the filter satisfies
    \begin{equation*}
            R^{-1} \cdot K^{(\rho,\rho_1)}(\y) = \rho(R) K^{(\rho,\rho_1)}(\y) \rho_1(R)^\dagger,
    \end{equation*}
    where \(R^{-1} \cdot K^{(\rho,\rho_1)}(\y) = \sum_{\mathbf{z} \in \ZZ^d} K^{(\rho,\rho_1)}(\mathbf{z}) \I(R^{-1} \mathbf{z}, \y)\).
\end{prop}
For general transformations, however, the assumptions of Proposition \ref{prop: steerable filter} are not satisfied. Even in favorable cases, such as rotations by $90^\circ$ about the $x$ or $y$ axes in three dimensions, the transformed angular grid does not align with the original configuration (see Figure \ref{fig: align}). This leads to a loss of equivariance. The following theorem provides an upper bound on this error.
\begin{thm}\label{thm: interpolation error}
    Suppose \(f: \ZZ^d \to \mathbb{C}^{d_{\rho_1}}\) be a feature map such that under the action of any \((\t,R)\in \SE(d)\), it transforms according to
    \begin{equation*}
        f(\x)\mapsto [(\mathbf{t}, R)^{-1} \cdot f](\x) = \rho_1(R)^\dagger \I[f](R \x + \mathbf{t}).
    \end{equation*}
    Let \(K^{(\rho, \rho_1)}\) be as defined in \eqref{eq: steerable filter}, and assume that its weights are uniformly bounded above by some constant. Then, for any \((\mathbf{t}, R) \in \SE(d)\) and \(\x \in \ZZ^d\),
    \begin{equation*}
        \|(\mathbf{t}, R) \cdot [K^{(\rho, \rho_1)} \star (\mathbf{t}, R)^{-1} \cdot f](\x) - [K^{(\rho, \rho_1)} \star f](\x)\|_\infty \leq C  \bigg(\Delta(\mathbf{t}, R) + \Delta(R) + n_a^{-\alpha}\bigg) \sum_{\y\in \ZZ^d}\|f(\y)\|_2
    \end{equation*}
    for some constant \(C > 0\), which only depends on radius \(h\), dimension \(d\), choice of \(\I\) and the bounds on the weights.
\end{thm}

\begin{remark}
    The error $\Delta(\t, R)$ arises from applying the back and forth transformation $\x \mapsto R\x + \mathbf{t}$ to the grid. The term $\Delta(R)$ is introduced by the interpolation performed at each patch location of the rotated feature map. Naturally, if the transformation maps points back onto the integer grid, as in the case of integer translations or rotations by multiples of $90^\circ$ (about \(z\) axis), the terms $\Delta(\t, R)$ and $\Delta(R)$ vanish, and no error is introduced. These observations emphasize that interpolation errors depend on how the transformation aligns with the structure of the underlying grid.
\end{remark}

\begin{remark}

The term $n_a^{-\alpha}$ arises from approximating the spherical harmonic transform using a discrete grid, with $\alpha$ capturing the smoothness of the interpolation kernel. For nearest-neighbor interpolation ($\alpha = 0$), increasing the grid resolution does not reduce the error. In contrast, for linear interpolation ($\alpha = 1$), the approximation error decays at the rate $n_a^{-1}$, allowing this error to be controlled by using a finer grid.
\end{remark}

\subsection{Implementation in 2D and 3D}

In this section we will discuss some specific details about implementing steerable convolutions in two and three dimensions and contrast them to prior works.

\paragraph{2D steerable convolutions:}

Implementing steerable convolutions in two dimensions is relatively straightforward because all spherical harmonics are one dimensional and take the form $\Y^{(k)}(\theta) = e^{\iota k \theta}$, where $k \in \ZZ$. These functions also serve as the irreps $\rho_k(\theta)$ of $\SO(2)$. In this setting, the precomputed filter for the first layer becomes
\begin{equation}\label{eq: filter 2d}
    M_r^{(k)} =  \frac{r}{n_r^2n_a}\sum_{a=1}^{n_a} \I\left(\frac{rh}{n_r}s_a, \y\right) e^{\iota k\frac{2\pi a}{n_a}}, 
    \quad s_a = 
    \begin{bmatrix} 
        \cos\frac{2\pi a}{n_a} \\ 
        \sin\frac{2\pi a}{n_a} 
    \end{bmatrix},
\end{equation}
where $0 \leq k \leq n_a - 1$. Because the product of two irreps $\rho_{k_1}$ and $\rho_{k_2}$ results in $\rho_{k_1 + k_2}$, the CG--matrices are given by $C^{(k, k_1, k_2)} = \mathbbm{1}_{\{k = k_1 + k_2\}}$. As a result, the steerable filters in higher layers simplify to
\begin{align}
      \sum_{k_2} \sum_{r=1}^{n_r} w_r^{(k,k_1,k_2)} \tilde{M}_r^{(k,k_1,k_2)}(\y) 
    =& \frac{1}{|\mathcal{F}_{\SO(2)}|}  \sum_{r=1}^{n_r} \sum_{k_2} w_r^{(k,k_1,k_2)}\mathbbm{1}_{\{k=k_1+k_2\}} M_r^{(k_2)}(\y) \nonumber\\
    =& \frac{1}{|\mathcal{F}_{\SO(2)}|}\sum_{r=1}^{n_r}w_r^{(k,k_1)} M_r^{([k-k_1]_{n_a})}(\y), \label{eq: 2d paramter reduction}
\end{align}
where $w_r^{(k, k_1)} := w_r^{(k, k_1, k - k_1)}$, and $[\cdot]_{n_a}$ denotes modulo $n_a$ operation. Since each pair of irreps of $\SO(2)$ yields exactly one irrep, this structure naturally leads to a reduction in the number of learnable parameters, as seen in \eqref{eq: 2d paramter reduction}. Moreover, if the number of angular discretization points $n_a$ matches the highest Fourier frequency cutoff, the symmetry group can effectively be treated as the cyclic group $\ZZ_{n_a}$. This setting corresponds to the $\ZZ_{n_a}$ steerable filters discussed by \citet{weiler2019general}.

\paragraph{3D steerable convolutions:}
In three dimensions, the implementation becomes more intricate because the spherical harmonics, apart from the constant representation, are no longer one dimensional. Instead, the spherical harmonic take the form \(\Y^{(\ell)} : \mathbb{S}^2 \to \mathbb{C}^{2\ell+1}\), where \(\ell \in \ZZ_{\geq 0}\). Consequently, the precomputed filter for the first layer in three dimensions is given by
\begin{equation}\label{eq: filter 3d}
    M_r^{(\ell)}(\y) = \frac{r^2}{n_r^3n_a^2}  \sum_{a_1, a_2=1}^{n_a} \I\left(\frac{rh}{n_r}s_{a_1, a_2}, \y\right) \Y^{(\ell)}\left(s_{a_1,a_2 }\right)\sin\frac{\pi (a_1+0.5)}{n_a}, 
    \,\, s_{a_1, a_2} = 
    \begin{bmatrix}
        \sin \frac{\pi (a_1+0.5)}{n_a}\cos \frac{2\pi {a_2}}{n_a} \\ 
        \sin \frac{\pi (a_1+0.5)}{n_a} \sin \frac{2\pi a_2}{n_a} \\ 
        \cos \frac{\pi (a_1+0.5)}{n_a}
    \end{bmatrix}.
\end{equation}
Using the quadrature from \citet{DRISCOLL1994202} amounts to replacing the sine with \(\omega_{\text{DH}}\), defined in \eqref{eq: DH quadrature}. Unlike the 2D case, the tensor product of two irreps of \(\SO(3)\) can decompose into multiple irreps, and as a consequence, we no longer have parameter reduction seen in two dimensions. Therefore, the filters in the higher layers of 3D steerable convolutions retain the general structure described in \eqref{eq: steerable filter basis higher layers}.

\begin{remark}\label{remark: delta}
    The implementation of steerable convolutions across various works in the literature primarily differs in the definition of \(M_r^{(\rho)}\). For instance, in the 2D case, \citet{weiler2018learning} employed the filter
    \begin{equation*}
         M_r^{(k)} (\y) = e^{-\frac{(||\y||-r)^2}{2\tau^2}} \mathbbm{1}_{\{\y\neq0\}}\, e^{\iota k \angle\y},
    \end{equation*}
    while in the 3D case, \citet{weiler20183d} used
    \begin{equation*}
        M_r^{(\ell)}(\y) = e^{-\frac{(||\y||-r)^2}{2\tau^2}} \mathbbm{1}_{\{\y\neq0\}}\, \Y^{(\ell)}\left(\frac{\y}{||\y||}\right).
    \end{equation*}
    Here, \(\angle\y\) represents the angle that \(\y\) makes with the \(x\) axis, and \(\tau\) is a hyperparameter. Since these filter are directly defined using the Cartesian coordinate system, we will refer to these as Cartesian filters, to distinguish from the interpolation based filters.
\end{remark}
\begin{figure}
    \centering
    \includegraphics[width=\linewidth]{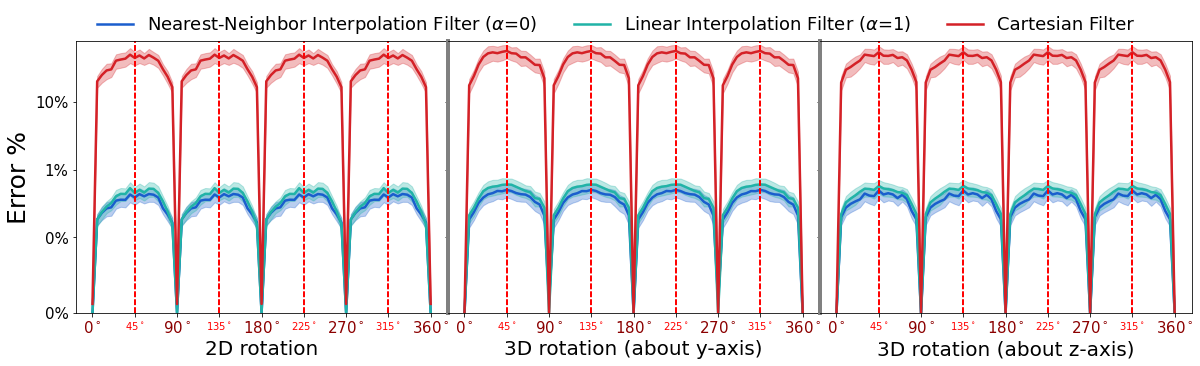}
    \caption{Simulation of equivariance error as a function of rotation in 2D and 3D. Each model consists of two convolutional layers separated by a normalization layer (Appendix \ref{sec: Normalization}) and followed by a flattening layer (Appendix \ref{sec: Flattening}), yielding rotation invariance. The error is defined as $\frac{\|\mathcal{M}(R\cdot f) - \mathcal{M}(f)\|_\infty}{\|f\|_1}$, where $\mathcal{M}$ denotes the model, $f$ the input, and $R$ the rotation. Model weights and inputs were sampled from a Gaussian distribution. Shaded regions indicate \(90\%\) Wald confidence intervals over \(100\) independent runs. A Fourier cutoff of \(4\) was used for the 2D model and \(0\) for the 3D model.}
    \label{fig: error}
\end{figure}

\section{Experiments}\label{sec: experiments}

We investigated the performance difference between Cartesian filters (Remark \ref{remark: delta}) and interpolation based filters in steerable convolutions by benchmarking their performance across four widely used datasets. For each dataset, we systematically varied the Fourier cutoff values to study their effect on model accuracy. Training was carried out both with and without rotation augmentation, and models were subsequently evaluated on rotated test samples to assess generalization to unseen orientations. Each experimental setup was repeated across five independent runs to ensure reliability. To further probe model robustness, we performed a sensitivity analysis by introducing Gaussian noise at multiple intensity levels into the test data and measuring how performance degraded under these perturbations. For each noise level and each trained model, evaluations were repeated over five runs. The following sections describe the datasets, training procedures, and detailed analyses of the experimental findings. 

\subsection{Datasets}

\begin{figure}[t]
    \centering
    \includegraphics[width=\linewidth]{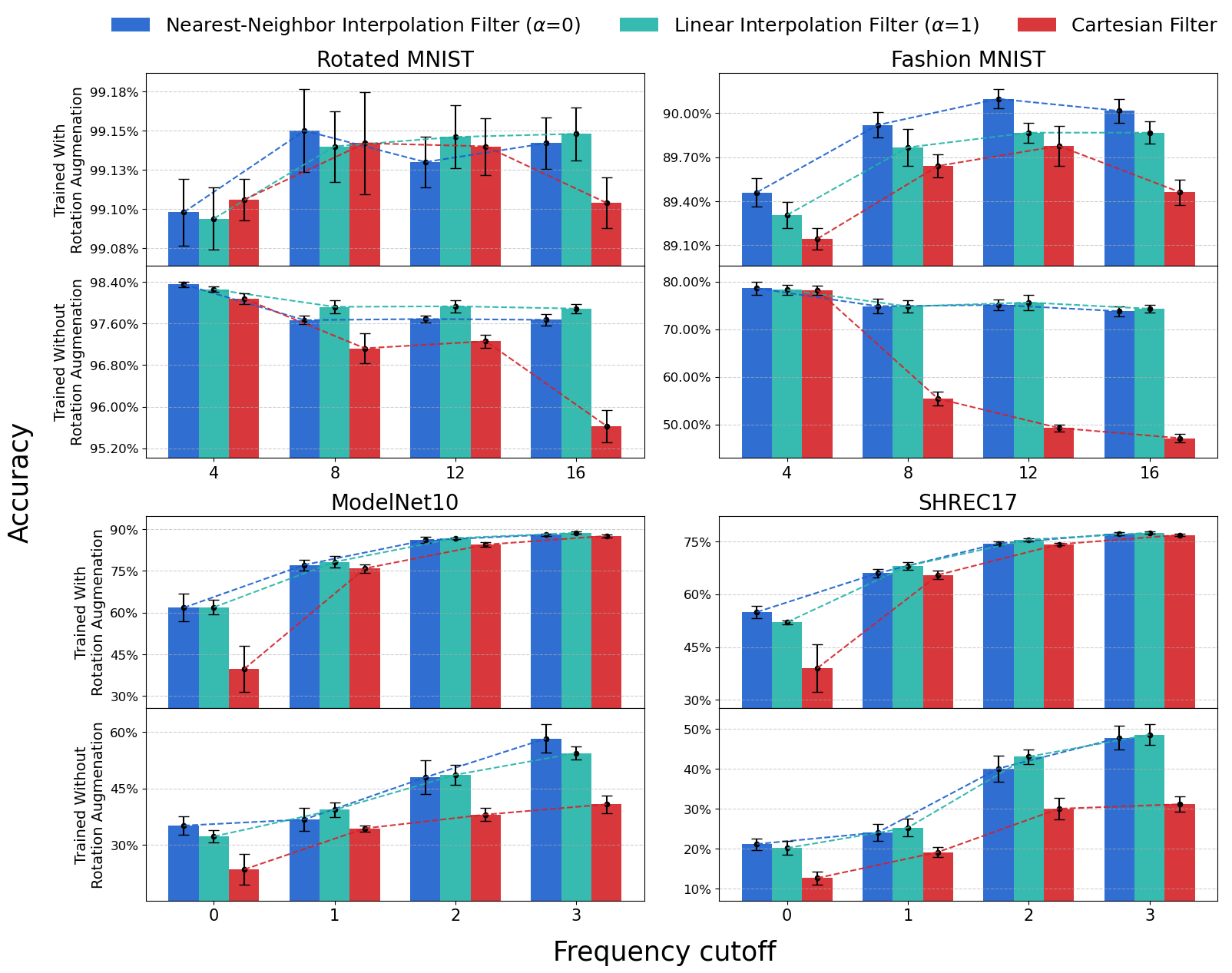}
   \caption{Performance comparison of Nearest Neighbor interpolation, Linear interpolation, and Cartesian filters across four benchmark datasets. Error bars indicate \(90\%\) Wald confidence intervals computed from five independent runs. The values used to generate this plot are provided in Table \ref{table: experiments}.}
    \label{fig: bar comparison}
\end{figure}

\paragraph{Rotated MNIST:}
The rotated MNIST dataset, a variant of the original MNIST dataset \citep{lecun2010mnist}, serves as a popular benchmark for assessing the performance of steerable architectures in 2D. It comprises grayscale images of size \(28 \times 28\) spanning 10 categories of handwritten digits \((0\)-\(9)\), with each image subjected to random rotations uniformly sampled from \(0^\circ\) to \(360^\circ\). This rotated dataset includes \(12,000\) training samples and \(50,000\) test samples. For training without rotation augmentation, we used a random set of \(12,000\) sampled from the original MNIST dataset.

\paragraph{Rotated Fashion MNIST:}

The Fashion MNIST dataset \citep{xiao2017fashion} contains grayscale images of size $28 \times 28$ spanning 10 clothing categories. Following the setup used for Rotated MNIST, we randomly selected $12,000$ samples for training. The remaining $58,000$ images were used to construct the test set, where each image was rotated into 16 orientations uniformly spaced between $0^\circ$ and $360^\circ$, yielding a total of $928,000$ test samples. Rotations were applied using linear interpolation. For the sensitivity analysis, we instead used the original $58,000$ test images, applying a random rotation to each and subsequently adding Gaussian noise.

\begin{figure}[t]
    \centering
    \includegraphics[width=\linewidth]{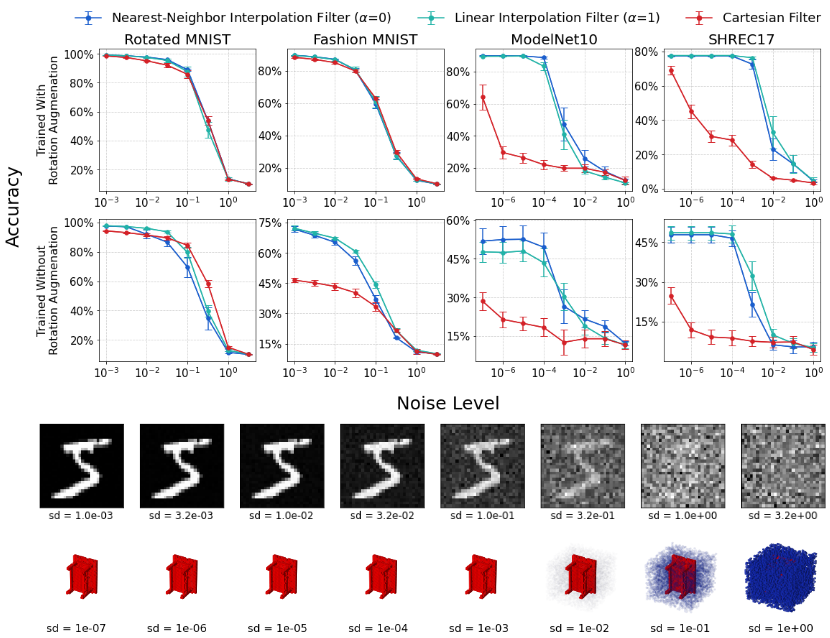}
    \caption{Effect of Gaussian noise perturbations on the performance of the model with highest frequency cutoff. For each noise level, every trained model was evaluated across \(5\) runs, with \(5\) models trained per configuration. Variance was estimated using $\text{Var}(X) = \text{Var}(\mathbb{E}[X|\mathcal{M}]) + \mathbb{E}[\text{Var}(X|\mathcal{M})]$, where $X$ denotes accuracy and $\mathcal{M}$ a trained model. Error bars indicate \(90\%\) Wald confidence intervals based on this variance. See Figure \ref{fig: sensitivity all} for the plots corresponding to all models.}
    \label{fig: sensitivity}
\end{figure}

\paragraph{ModelNet10:}
The ModelNet10 dataset \citep{wu20153d} is a 3D CAD collection spanning 10 object categories, with $3,991$ training samples and $908$ test samples. Each object is represented as a triangular mesh. For our experiments, we utilized the point cloud version of the dataset provided at \url{https://github.com/antao97/PointCloudDatasets}, where $2,048$ points were uniformly sampled from each object surface using the farthest point sampling algorithm. These point clouds were embedded into grids to obtain voxelized representations. To construct the test set, each example was rotated into $512$ orientations ($8$ uniformly spaced orientations between $0^\circ$ and $360^\circ$ about each of $y$ and $z$ axes, under the $y$-$z$-$y$ Euler parameterization of $\SO(3)$), yielding a total of $464,896$ test samples. For the sensitivity analysis, we instead used the original $908$ test objects, applied a uniformly sampled random rotation from $\SO(3)$, and added Gaussian noise.

\paragraph{SHREC17:}
The SHREC17 dataset \citep{shapenet2015} contains $51,162$ 3D models in .obj format, spanning 55 household object categories such as tables and chairs. It is divided into $35,764$ training samples, $5,133$ validation samples, and $10,265$ test samples. For our experiments, models were converted into $32 \times 32 \times 32$ voxel grids using the script available at \url{https://github.com/mariogeiger/obj2voxel}. The dataset defines two categories: a \emph{normal} set with upright objects and a \emph{perturbed} set with randomly rotated objects. We used the normal set for training without rotation augmentation and the perturbed set for testing. For the sensitivity analysis, Gaussian noise was added to the perturbed test samples.

\subsection{Training Details}

We used a single architecture for all 2D experiments and another for all 3D experiments. The 2D architecture was based on the model of \citet{weiler2019general} for Rotated MNIST. Models were trained with a batch size of $100$ for Fourier cutoffs of $4$, $8$, and $12$; for a cutoff of $16$, memory constraints required reducing the batch size to $60$. The 3D architecture followed the same overall design, with adjustments to account for the varying number of channels per Fourier component. Unlike the 2D case, where each Fourier cutoff contains the same number of channels, the channel count in 3D depends on the $\ell$-th Fourier component. For 3D experiments, batch sizes of $100$ were used for Fourier cutoffs of $0$, $1$, and $2$, and $85$ for cutoff $3$. All models were trained for $150$ epochs with the Adam optimizer \citep{adam}, using an initial learning rate of $5 \times 10^{-3}$ decayed by half every 20 epochs, and weight decay of $5 \times 10^{-4}$. For each dataset, models were trained both with and without rotation augmentation to assess performance under different conditions. Consistent with the approach of \citet{weiler20183d}, the number of radial grid points we chose $n_r$ to be $\max_{1 \leq i \leq d} \lfloor s_i / 2 \rfloor$, where $\lfloor \cdot \rfloor$ denotes the floor function and $s_1 \times s_2 \times \dots \times s_d$ specifies the filter size in Cartesian coordinates. For 2D Cartesian filters, we followed \citet{weiler2019general} and set $\tau=0.6$ for all radial profiles except the outermost ring, where $\tau=0.4$ was used. For 3D Cartesian filters, following \citet{weiler20183d}, we set $\tau=0.6$ for all radial profiles. We also drop the constant \(1/|\mathcal{F}(\SO(d))|\) in equation \eqref{eq: steerable filter basis higher layers} for the Cartesian filters to remain consistent with prior formulations in the literature.

\subsection{Results}

To examine equivariance error, we conducted a controlled simulation experiment. We used a minimal network comprising two convolutional layers separated by a normalization layer (Appendix \ref{sec: Normalization}) and followed by a flattening layer (Appendix \ref{sec: Flattening}), for both 2D and 3D steerable convolution settings. Equivariance error was quantified as the maximum difference in the output, normalized by the $\mathcal{L}_1$ norm of the input. Input resolutions were set to $1\times 28\times 28$ for 2D models and $1\times 32\times 32\times 32$ for 3D models, consistent with those used in the main experiments. For each simulation, both inputs and model weights were sampled from a Gaussian distribution. In the 2D case, equivariance error was computed across rotated inputs at increments of $5^\circ$ from $0^\circ$ to $360^\circ$. For the 3D case, rotations were applied separately about the $y$ and $z$ axes. Figure \ref{fig: error} compares the errors for networks using Cartesian versus interpolation based filters. Cartesian filters yielded substantially larger errors than interpolation based filters suggesting that the latter is better suited for generalizing to unseen rotations.

Figure \ref{fig: bar comparison} summarizes the performance of steerable networks using Cartesian and interpolation based filters. Across all four datasets and a range of Fourier cutoffs, interpolation based filters achieved comparable or superior performance. The performance gap was especially pronounced when models were trained without rotation augmentation, reflecting the higher equivariance error of Cartesian filters and their reduced ability to generalize to unseen orientations. Furthermore, as shown in Figure \ref{fig: sensitivity}, networks trained with Cartesian filters exhibited greater susceptibility to Gaussian noise perturbations, whereas interpolation based filters demonstrated increased robustness, with performance degrading more gradually under noise. These results indicate that interpolation based filters provide a more reliable alternative, particularly in noisy domains such as medical imaging. For brevity, Figure \ref{fig: sensitivity} shows the sensitivity analysis only for models with the highest Fourier cutoff. The corresponding plots for all Fourier cutoffs are provided in Figure \ref{fig: sensitivity all}.

Figure \ref{fig: bar comparison} illustrates the relationship between model accuracy and Fourier cutoff. Accuracy increases with cutoff up to a certain point, after which returns diminish, reflecting a bias variance tradeoff in the choice of cutoff. A persistent performance gap remains between models trained with and without rotation augmentation. Although such a gap should not theoretically exist, interpolation introduces equivariance errors that contribute to this disparity. Rotation augmentation enables the model to compensate for these inconsistencies, thereby improving performance (see Section \ref{sec: loss of equivariance}). The effect is more pronounced in 3D datasets, where accuracy declines more sharply.
\section{Conclusion}

In this work, we introduced a novel framework for deriving the equations of steerable convolutional networks, adopting a geometric perspective rather than an algebraic one. We established a general formulation of equivariant linear maps in a broad setting, encompassing existing results in the literature as special cases, and provided a new proof of this characterization. Building on this foundation, we systematically derived the equations of steerable convolutions using Fourier space arguments. This approach offers an intuitive methodology that not only facilitates efficient implementation but also enables the exploration of new architectural variations. The derivations further shed light on the natural emergence of spherical harmonics and the Clebsch--Gordan decomposition in the context of steerable convolutions. Furthermore, it motivated a novel implementation of steerable convolutions using interpolation kernels, which our experiments showed to outperform existing methods and to exhibit greater robustness to Gaussian noise perturbations. Collectively, these contributions advance the theoretical understanding of steerable networks and provide practical guidance for developing more robust and efficient architectures.

\bibliography{references}
\bibliographystyle{Format/tmlr/tmlr}

\clearpage
\appendix

\allowdisplaybreaks
\section{Background}

\subsection{Groups}\label{sec: background groups}
\paragraph{Definition:} A group \m{\G } is a non-empty set together with a binary operation (commonly denoted by ``\m{\cdot}'', that combines any two elements 
\m{a} and \m{b} of \m{\G } to form an element of 
\m{\G }, denoted \m{a\cdot b}, such that the following three requirements, known as group axioms, are satisfied:
\begin{itemize}
    \item \emph{Associativity} : For all \m{a,b,c \in \G }, one has \m{(a\cdot b)\cdot c = a\cdot(b \cdot c)}
    \item \emph{Identity Element} : There exists as element \m{e \in \G } such that for every \m{a \in \G }, one has \m{e\cdot a = a\cdot e = a}. Such an
    element is unique. It is called the \emph{identity element} of the group.
    \item \emph{Inverse} : For each \m{a \in \G }, there exists an element \m{b \in \G } such that \m{a\cdot b = b\cdot a = e}, where \m{e} is the identity element. For each \m{a}, the element \m{b} is unique and it is called the inverse of \m{a} and is commonly denoted by \m{a^{-1}}.
\end{itemize}
Examples include set of integers \m{\ZZ} with the addition operation and the set of non-zero reals \m{\RR \setminus \{0\}} with the multiplication operation.
From here on we will drop ``\m{\cdot}'', for simplicity. The group operation will be clear from the elements of the group concerned.

\paragraph{Group Homomorphism:} Given two groups \m{\G } and \m{H}, a function \m{\phi:\G \to H} is called a group homomorphism if \m{\phi(ab) = \phi(a)\phi(b)} for any \m{a,b\in \G }. If the map \m{\phi} is a bijection, it is called an \emph{isomorphism}. Furthermore, if \m{\G =H}, then an isomorphism is called \emph{automorphism}. The set of all automorphisms of a group \m{\G } with the operation of composition  form a group in itself and is denoted by \m{\textrm{Aut}(\G )}. 

\paragraph{Compact Groups:} A topological group is a topological space that is also a group such that the group operation and the inverse map are continuous. A compact group is a topological group whose topology realizes it as a compact topological space (see \citet{munkres1974topology} for definition of compact topological spaces).
Some classic examples of compact groups are the groups \m{\SO(d)} (the group of all real orthogonal matrices in \m{d} dimensions with determinant 1), \m{\textrm{U}(d)} (the group of all complex unitary matrices in \m{d} dimensions) and \m{\textrm{SU}(d)} (the group of all complex unitary matrices in \m{d} dimensions with determinant 1).

\emph{Special Orthonormal Group \(\SO(d)\):} This group comprises all real orthogonal matrices in \(d\) dimensions with a determinant of \(1\). These groups are associated with rotations in \(d\) dimensions. In two dimensions, the group \(\SO(2)\) can be parametrized by a single angle \(\theta\) corresponding to the rotation matrix
\begin{equation*}
    R(\theta) = 
    \begin{bmatrix}
        \cos(\theta) & -\sin(\theta) \\ 
        \sin(\theta) & \cos(\theta)
    \end{bmatrix}.
\end{equation*}
Here, \(R(\theta)\) signifies anti-clockwise rotation in the \(x\)-\(y\) plane by an angle \(\theta\in [0, 2\pi)\) radian. In three dimensions, the group $\SO(3)$ can be parameterized using the so called \emph{Euler angles}. Consider the rotation matrices $R_z(\alpha)$ and $R_y(\beta)$, defined by
\begin{equation*}
    R_z(\alpha) = 
    \begin{bmatrix}
        \cos(\alpha) & -\sin(\alpha) & 0\\
        \sin(\alpha) & \cos(\alpha) & 0\\
        0 & 0 & 1
    \end{bmatrix}, \quad
    R_y(\beta) = 
    \begin{bmatrix}
        \cos(\beta) & 0 & \sin(\beta)\\
        0 & 1 & 0\\
        -\sin(\beta) & 0 & \cos(\beta)
    \end{bmatrix},
\end{equation*}
which represent anti-clockwise rotations by angles $\alpha$ and $\beta$ about the $z$-axis and $y$-axis, respectively. Using the $z$-$y$-$z$ convention, any rotation matrix $R \in \SO(3)$ can be expressed as $R = R_z(\alpha) R_y(\beta) R_z(\gamma)$, where $\alpha, \gamma \in [0, 2\pi)$ and $\beta \in [0, \pi)$ are the Euler angles.

\paragraph{Semi-direct Product Groups:} Given two groups \m{N} and \m{H} and a group homomorphism \m{\phi : H\to \text{Aut}(N)}, we can construct a new group \m{N\rtimes_\phi H} defined as follows:
\begin{itemize}
    \item The underlying set is the Cartesian product \m{N\times H}.
    \item The group operation is given by \m{(n_1, h_1)(n_2, h_2) = (n_1\phi_{h_1}(n_2), h_1h_2)}.
\end{itemize}

\emph{Special Euclidean Group \m{{\text{SE}(d)}}:} The Special Euclidean group $\text{SE}(d)$ consists of all combinations of rotations and translations in $d$ dimensions. Translations in $\mathbb{R}^d$ form a group isomorphic to $\mathbb{R}^d$ itself, while rotations are represented by the group $\SO(d)$. We define a group homomorphism $\phi : \SO(d) \to \text{Aut}(\mathbb{R}^d)$ by setting $\phi(R) = (\mathbf{t} \mapsto R\mathbf{t})$, which describes how a rotation acts on a translation vector.
Using this, the group $\text{SE}(d)$ can be constructed as the semidirect product $\mathbb{R}^d \rtimes_\phi \SO(d)$. The group operation for two elements $(\mathbf{t}_1, R_1)$ and $(\mathbf{t}_2, R_2)$ in $\SE(d)$ is given by
\begin{equation*}
    (\mathbf{t}_1, R_1)(\mathbf{t}_2, R_2) = (\mathbf{t}_1 + R_1\mathbf{t}_2,\, R_1R_2).
\end{equation*}
The identity element of $\text{SE}(d)$ is $(\mathbf{0}, I)$, and the inverse of an element $(\mathbf{t}, R)$ is given by $(\mathbf{t}, R)^{-1} = (-R^{-1}\mathbf{t}, R^{-1})$.

\subsection{Group Actions}\label{sec: background group actions}
\paragraph{Definition:} If \m{\G } is a group with identity element \m{e}, and \m{X} is a set, then a (left) group action of \m{\G } on \m{X} is a function 
\m{\alpha: \G \times X \to X}, that satisfies the following two axioms for all \m{g, h\in \G } and \m{x\in X}:
\begin{itemize}
    \item \emph{Identity}: \m{\alpha(e, x) = x},
    \item \emph{Compatibility}: \m{\alpha(g, \alpha(h,x)) = \alpha(gh, x)}.
\end{itemize}
The group action is said to be continuous if \(\alpha\) is a continuous map. Often \m{\alpha(g,x)} is shortened to \m{g\cdot x}. Any group \m{\G } acts on itself by the group operation.  If \m{\G } acts on \m{X}, then it also naturally acts on any function \m{f} defined on \m{X}, as \m{(g\cdot f )(x) = f(g^{-1}\cdot x)}.

\emph{Action of \(\SE(d)\) on \(\RR^d\):} The special Euclidean group acts on a vector in \m{\RR^d} by first applying the rotation component followed by translation. For \m{\x \in \RR^d} and \m{(\t, R)\in \SE(d)},
    \begin{equation*}
        (\t, R)\cdot \x = R\x + \t
    \end{equation*}
gives us the action of \m{\SE(d)} on \m{\RR^d}.

\subsection{Group Representations}\label{sec: background group representations}
\paragraph{Definition:} A representation of a group \m{\G } is a group homomorphism from \m{\G } to \m{\text{GL}(\CC^n)} (group of invertible linear maps on \m{\CC^n}).
Here \m{n} is called the dimension of the representation, which can possibly be infinite. A representation is \emph{unitary} if \m{\rho} maps to unitary linear transformation of \m{\CC^n}.

\paragraph{Irreducible Representations:} If we have two representations, $\rho_1$ and $\rho_2$ of dimensions $n_1$ and $n_2$ respectively, then the two can be combined by a direct sum to give another representation of dimension $n_1+n_2$,
    \[\rho_1(g) \oplus \rho_2(g)  = \begin{bmatrix}
        \rho_1(g) & 0 \\ 0 & \rho_2(g)
    \end{bmatrix}.\]

\noindent A representation is said to be \emph{completely reducible} if it can be expressed as a direct sum of other representations after maybe a change of basis, i.e, 
    \[U\rho(g)U^{-1} = \bigoplus_i \rho_i(g)\]
where $U$ is a unitary change of basis matrix and the direct sum extends over some number of representations. 
However, for \emph{every} group there are a some representations which cannot be broken further into a direct sum of other representations. These are called the \emph{irreducible representations} or \emph{irreps} of the group. These irreps are the building blocks of the all other representations of the group, in the sense that any representation can be written as a direct sum of the irreps,
    \begin{equation*}
        \rho(g) = U\left[ \bigoplus_i \rho^{(i)}(g) \right]U^{-1}
    \end{equation*}
where again $U$ is a change of basis matrix and $\rho^{(i)}$ are the irreps. The Peter--Weyl Theorem (Theorem \ref{thm: peter thm}) by \citet{peter1927vollstandigkeit} tells us that for a compact group \m{\G }, any unitary representation \m{\rho} is completely reducible and splits into direct sum of irreducible \emph{finite dimensional unitary} representations of \m{\G }.

\paragraph{Irreducible Representations of \(\SO(2)\) and \(\SO(3)\):}\(\SO(d)\) being a compact group, all its irreps are finite dimensional unitary representations. In the case of $\SO(2)$, every irrep is one-dimensional and indexed by an integer. The group $\SO(2)$ can be parameterized by an angle $\theta \in [0, 2\pi)$, corresponding to an anti-clockwise rotation in the $x$-$y$ plane by $\theta$ radians. Under this parameterization, the irreps of $\SO(2)$ take the form
$$
\rho^{(k)}(\theta) = e^{ik\theta}, \qquad k \in \mathbb{Z}.
$$
The irreps of \(\SO(3)\) are indexed by positive integers $\ell \in \ZZ_{\geq 0}$, where the $\ell$'th representation is of dimension $2\ell+1$ and are given by the so called Wigner D-matrices \citep{wigner1932}:
    \[\rho^{(\ell)}(R) = D^{(\ell)}(R), \qquad \ell\in \ZZ_{\geq 0}.\]

\subsection{Fourier Transform}\label{sec: background fourier transform}

\paragraph{Haar Measure:} Theorem \ref{thm: haar} guarantees the existence of a unique canonical measure on locally compact groups, known as the (left) Haar measure. One key property of the Haar measure is that $\mu(U) > 0$ for every non-empty open subset $U \subseteq \G $. In the special case where $\G $ is compact, the total measure $\mu(\G )$ is finite and strictly positive. This allows us to uniquely determine a Haar measure by imposing the normalization condition $\mu(\G ) = 1$, which is especially useful in probability theory and harmonic analysis, where one often interprets the Haar measure as a uniform distribution over the group. Additionally, for compact groups such as $\SO(d)$, the Haar measure is unimodular, meaning it is invariant under both left and right translations, i.e., $\mu(gA) = \mu(Ag) = \mu(A)$ for all measurable subsets $A \subseteq \G $ and all $g \in \G $. This symmetry simplifies many computations and ensures that convolution and integration operations behave consistently with the group structure.

\paragraph{Fourier Transform on Compact groups:}A notable and useful property of compact groups is that the set of (isomorphism classes of) their irreps is countable \citep{robert1983introduction}. This fact underpins a powerful generalization of classical Fourier analysis to the setting of compact groups, enabling square-integrable functions to be decomposed into frequency components indexed by group representations. Let $f \in \mathcal{L}_2(\G )$, where $\mathcal{L}_2(\G )$ denotes the space of complex-valued, square-integrable functions on $\G $ with respect to the (normalized) Haar measure. For each irrep $\rho$, the Fourier transform of $f$ at $\rho$ is defined as
$$
\h{f}(\rho) = \int_\G f(g)\, \rho(g)\, d\mu(g),
$$
where $\h{f}(\rho)$ is a complex matrix of size $d_\rho \times d_\rho$. The inverse Fourier transform allows one to reconstruct the original function $f$ from its Fourier coefficients via the formula
$$
f(g) = \sum_{\rho } d_\rho\, \mathrm{Tr}\left( \h{f}(\rho)\, \rho(g)^\dagger \right),
$$
where the convergence is in the $\mathcal{L}_2$-sense. This formula is justified by the Peter--Weyl theorem (Theorem \ref{thm: peter thm}), which states that the matrix coefficients of all irreps form a complete orthonormal basis for $\mathcal{L}_2(\G )$. Consequently, any square integrable function on $\G $ can be expressed as a (generalized) Fourier series in terms of these basis functions.

\section{Spherical Harmonics}\label{sec: background spherical harmonics}

Spherical harmonics arise as a natural generalization of trigonometric functions and classical Fourier series to higher dimensional spheres. Just as the functions \(\sin\) and \(\cos\) provide a basis for square integrable functions on the unit circle \(\mathbb{S}^1\), spherical harmonics furnish an orthogonal basis for functions defined on the unit sphere \(\mathbb{S}^{d-1} \subset \mathbb{R}^d\). They play a central role in mathematical physics, harmonic analysis, and approximation theory, serving as the fundamental building blocks for problems with rotational symmetry. Applications range from quantum mechanics and geophysics to computer vision and machine learning, where spherical harmonics allow functions on spheres to be analyzed, expanded, and approximated in terms of simple, structured components.

From a mathematical perspective, spherical harmonics in \(d\) dimensions emerge from the study of harmonic polynomials. These are polynomials that satisfy the Laplace equation and thus capture a notion of ``balanced'' or ``smooth'' behavior in all directions. By restricting homogeneous harmonic polynomials to the unit sphere \(\mathbb{S}^{d-1}\), one obtains the spherical harmonics of degree \(\ell\). In this way, spherical harmonics can be viewed both as natural eigenfunctions of the Laplacian on the sphere and as the higher dimensional analogue of Fourier modes on the circle.

Formally, a polynomial \(P(\x_1, \dots, \x_d)\) is called \emph{homogeneous} of degree \(\ell\) if each of its monomials has total degree \(\ell\), that is,
\begin{equation*}
    P(\lambda \x_1, \dots, \lambda \x_d) = \lambda^\ell P(\x_1, \dots, \x_d), \quad \text{for all } \lambda \in \mathbb{R}.
\end{equation*}
For example, \(P(\x_1, \x_2) = \x_1^2 - \x_2^2\) is homogeneous of degree \(2\). A homogeneous polynomial is called \emph{harmonic} if it satisfies the Laplace equation,
\begin{equation*}
    \Delta P(\x_1, \dots, \x_d) = 0, \quad \Delta = \sum_{i=1}^d \frac{\partial^2}{\partial x_i^2}.
\end{equation*}

Let $\mathcal{H}_d^{(\ell)}$ denote the space of harmonic homogeneous polynomials of degree $\ell$ in $d$ variables. The corresponding spherical harmonics are obtained by restricting the elements of $\mathcal{H}_d^{(\ell)}$ to the unit sphere $\mathbb{S}^{d-1} \subset \mathbb{R}^d$. An orthonormal basis of this space will be denoted by $\Y^{(\ell)}_d$. The dimension of the space of spherical harmonics of degree $\ell$ is given by
\begin{equation*}
\operatorname{dim} \mathcal{H}_d^{(\ell)} = \frac{(2\ell + d - 2)(\ell + d - 3)!}{\ell!(d-2)!}.
\end{equation*}

To obtain explicit formulas for spherical harmonics in $d$ dimensions, it is convenient to use the standard angular parameterization of the unit sphere $\mathbb{S}^{d-1}$. For $\theta \in \Theta^{d-1} := [0,\pi]^{d-2} \times [0,2\pi)$, define $s(\theta) \in \mathbb{S}^{d-1}$ componentwise by
\begin{equation*}
    s_i(\theta) := 
    \begin{cases}
        \left(\prod_{j=1}^{i-2} \sin\theta_j\right)\cos\theta_i & 1 \leq i \leq d-1,\\
        \,\,\,\,\prod_{j=1}^{d-1} \sin\theta_j & i = d.
    \end{cases}
\end{equation*}
In these coordinates, one explicit choice of spherical harmonics $Y_d^{(\ell)}$ is given by
\begin{equation}\label{eq: spherical harmonics d dimensions}
      \left[\Y^{(\ell)}_d(s(\theta))\right]_{\mathbf{m}} 
    = \left[\mathcal{K}^{(\ell)}_{d}\right]_{\mathbf{m}}\frac{e^{\iota \mathbf{m}_{d-1} \theta_{d-1}}}{\sqrt{2\pi}}\prod_{j=1}^{d-2}(\sin \theta_j)^{\mathbf{m}_{j+1}} C_{\mathbf{m}_j - \mathbf{m}_{j+1}}^{\mathbf{m}_{j+1} + (d-j-1)/2}(\cos \theta_j)
\end{equation}
where the index $\mathbf{m} = (\mathbf{m}_1,\ldots,\mathbf{m}_{d-1})$ satisfies the condition $|\mathbf{m}_{d-1}| \leq \mathbf{m}_{d-2} \leq \cdots \leq \mathbf{m}_1 = \ell$, and the normalization constants \(\mathcal{K}^{(\ell)}_{d}\) are given by
\begin{equation*}
    \left[\mathcal{K}^{(\ell)}_{d}\right]_{\mathbf{m}} = \prod_{j=1}^{d-2}\frac{\Gamma\left(\mathbf{m}_{j+1}+\frac{d-j+1}{2}\right)}{2\mathbf{m}_{j+1}+d-j-1}\sqrt{\frac{2^{2\mathbf{m}_{j+1}+d-j-1}(2\mathbf{m}_j+d-j-1)(\mathbf{m}_j-\mathbf{m}_{j+1})!}{\pi(\mathbf{m}_j-\mathbf{m}_{j+1}+d-j-2)!}}.
\end{equation*}
Here, $C_n^{(\alpha)}$ denotes the Gegenbauer polynomial,
\begin{equation*}
    C_{n}^{(\alpha )}(z)=\sum _{k=0}^{\lfloor n/2\rfloor }(-1)^{k}{\frac {\Gamma (n-k+\alpha )}{\Gamma (\alpha )k!(n-2k)!}}(2z)^{n-2k},
\end{equation*}
where \(\alpha>0,\, n\in \ZZ_{\geq 0}\), and $\Gamma$ is the Gamma function \citep{cohl2023gegenbauer}. 

The constants $\mathcal{K}^{(\ell)}_{d}$ are chosen so that each basis coordinate function is normalized in $\mathcal{L}_2$ with respect to the surface measure on $\mathbb{S}^{d-1}$:
\begin{equation*}
    \int_{\Theta^{d-1}} \left\vert\left[\Y^{(\ell)}_d(s(\theta))\right]_{\mathbf{m}}\right\vert^2 \omega(\theta)\, d\theta = 1, \quad \omega(\theta) = \prod_{i=1}^{d-2} (\sin \theta_i)^{d-1-i}.
\end{equation*}
Additionally, with this normalization, the squared Euclidean norm of the vector $Y_d^{(\ell)}(s)$ is constant across the sphere. Specifically, for all $s \in \mathbb{S}^{d-1}$,
\begin{equation*}
    \|\Y^{(\ell)}_d(s)\|_2^2 = \frac{\operatorname{dim} \mathcal{H}_d^{(\ell)}}{\mathcal{A}(\mathbb{S}^{d-1})},
\end{equation*}
where $\mathcal{A}(\mathbb{S}^{d-1}) := \int_{\Theta^{d-1}} \omega(\theta)\, d\theta$ denotes the surface area of the unit sphere $\mathbb{S}^{d-1}$ (see Lemma \ref{lemma: spherical harmonics normalization} for proof).

In the special cases of two and three dimensions, the general expression \eqref{eq: spherical harmonics d dimensions} reduces to the familiar formulas. For $d=2$, corresponding to the unit circle $\mathbb{S}^{1}$, the spherical harmonics coincide with the Fourier modes:
\begin{equation*} 
    \Y^{(k)}_{2}(\theta)=\frac{1}{\sqrt{2\pi}}\,e^{\iota k\theta},\qquad k\in\mathbb{Z}.
\end{equation*}
For $d=3$, corresponding to the unit sphere $\mathbb{S}^{2}$, and using spherical coordinates $(\theta,\phi)$, the spherical harmonics take the classical form
\begin{equation*} 
    \big[\Y^{(\ell)}_{3}\big]_{m}(\theta,\phi)=\sqrt{\frac{2\ell+1}{4\pi}\,\frac{(\ell+m)!}{(\ell-m)!}}\; e^{\iota m\phi}\,P_{\ell}^{-m}(\cos\theta), \quad |m|\le \ell,\, \ell\in\mathbb{Z}_{\ge 0}.
\end{equation*}
where $P_{\ell}^{m}$ denotes the associated Legendre functions \citep{olver2010nist}.

\subsection{Spherical Harmonic Transform}

Every square integrable function $f \in \mathcal{L}_2(\mathbb{S}^{d-1})$ admits a spherical harmonic expansion that converges in the $\mathcal{L}_2$ sense:
\begin{equation*}\label{eq: ISHT background} 
    f(s) = \sum_{\ell=0}^{\infty} \Y^{(\ell)}_d(s)^\dagger \widehat f_{\ell} , \quad s \in \mathbb{S}^{d-1} 
\end{equation*}
Here the expansion coefficients $\widehat{f}_{\ell}$ are obtained as the inner products of $f$ with the basis functions, namely
\begin{equation*}\label{eq: SHT background} 
    \widehat f_{\ell} = \int_{\Theta^{d-1}} f(s(\theta)) {\Y^{(\ell)}_d(s(\theta))} \omega(\theta)\, d\theta. 
\end{equation*}

\subsection{Connection to irreps of SO(d)}

A distinguished subclass of representations of \( \SO(d) \) is realized on the space of symmetric traceless tensors.
For each integer \( \ell \ge 0 \), these spaces are defined as
\begin{equation*}
    V^{(\ell)} := \left\{T_{i_1 \dots i_\ell} \in (\mathbb{C}^d)^{\otimes \ell}\, \bigg\vert \, T_{i_{\sigma(1)} \dots i_{\sigma(\ell)}} = T_{i_1 \dots i_\ell}\ \forall\sigma\in S_\ell;\, \delta_{i_1 i_2} T_{i_1 i_2 i_3 \dots i_\ell} = 0 \right\}
\end{equation*}
where \( \delta_{ij} \) denotes the Kronecker delta. Elements of \( V^{(\ell)} \) are completely symmetric under permutation of indices and traceless under contraction with the Euclidean metric. The group \( \SO(d) \) acts naturally on \( (\mathbb{C}^d)^{\otimes \ell} \) by rotation of tensor indices, and this action preserves both symmetry and tracelessness.
Explicitly, the representation \(\rho^{(\ell)} : \SO(d) \to \GL(V^{(\ell)})\) is given by
\begin{equation*}
    \left[\rho^{(\ell)}[R](T)\right]_{i_1 \dots i_\ell} = \sum_{j_1=1}^{d} \cdots \sum_{j_\ell=1}^{d} R_{i_1 j_1} R_{i_2 j_2} \cdots R_{i_\ell j_\ell} T_{j_1 \dots j_\ell}, \quad R\in \SO(d).
\end{equation*}
Consequently, \( V^{(\ell)} \) is invariant under this action and carries an irrep of \( \SO(d) \). Since its carrier space consists of symmetric traceless tensors, the resulting irrep \( \rho^{(\ell)} \) are referred to as the \emph{symmetric traceless irrep} of degree \( \ell \).

The spaces of spherical harmonics \( \mathcal{H}_d^{(\ell)} \) furnish a functional realization of these representations. Under the natural action of \( \SO(d) \) on functions defined on the sphere,
\begin{equation*}
    [R \cdot f](s) = f(R^{-1}s),
\end{equation*}
the subspace \( \mathcal{H}_d^{(\ell)} \) remains invariant, and its basis functions \( Y_d^{(\ell)}(s) \) transform as
\begin{equation*}
    \Y^{(\ell)}_d(R^{-1}s) = \rho^{(\ell)}(R)^\dagger \Y^{(\ell)}_d(s),
\end{equation*}
Hence, spherical harmonics provide an explicit realization of the symmetric traceless irreps of \( \SO(d) \) as functions on the sphere. Moreover, the tensor product of two such symmetric traceless representations decomposes as a direct sum of symmetric and non-symmetric irreps \citep{georgi2000lie}:
\begin{equation*}
    \mathcal{H}_d^{(\ell_1)} \otimes \mathcal{H}_d^{(\ell_2)} \cong \bigoplus_{k=0}^{\min(\ell_1, \ell_2)} \mathcal{H}_d^{(\ell_1+\ell_2 - 2k)}   \bigoplus \text{ non-symmetric irreps}.
\end{equation*}
Each symmetric traceless irrep appearing in this decomposition occurs with multiplicity exactly one.
\section{Other Steerable Layers}\label{sec: implementataion}

Here we provide the details of other equivariant layers that we used in designing steerable neural networks.

\subsection{Non-linearity}\label{sec: CG nonlinearity}

All the operations we've discussed so far are linear in \(\fin\), but for a neural network layer, introducing nonlinearity is crucial. A common strategy in Fourier space networks is to first apply an inverse Fourier transform to bring \(\h \fpre\) back to the time domain. Here, a standard nonlinear activation function, such as ReLU, can be applied pointwise before transforming it back to Fourier space. This method allows flexibility in choosing any activation function. However, even with the efficiency of FFTs, the repeated forward and backward transforms can be computationally demanding, potentially becoming a bottleneck for the network. Furthermore, for groups such as $\SO(3)$, discretization on a uniform grid can introduce singularities, making FFT-based computations prone to numerical instability and errors. One strategy proposed by \citet{worrall2017harmonic} involves applying standard non-linearities, such as ReLU, to the norm of equivariant feature vectors at each spatial location. Another approach, used by \citet{weiler20183d, weiler2018learning}, first convolving the input with another steerable filter, then extracting the component corresponding to the constant representation, and finally multiplying it pointwise with the input.

In our experiments, we adopt the Clebsch--Gordan non-linearity used by \citet{SphericalCNNNeurIPS2018, anderson2019cormorant}. This involves computing the tensor product of the equivariant vector with itself at each location, followed by a CG--decomposition. In the 2D case, the output \(\h\fout\) at a given location \(\x\) is given by
\begin{equation*}
    \h\fout(\x, k) = \sum_{k'=0}^{k_{\max}} \eta_{k,k'}\widehat{\fpre}(\x, k') \widehat{\fpre}(\x, [k-k']_{n_a}), \quad 0 \leq k \leq k_{\max},
\end{equation*}
and in the 3D case, it is given by
\begin{equation*}
    \h\fout(\x, \ell) = \sum_{\ell_1, \ell_2=0}^{\ell_{\max}} \eta_{\ell, \ell_1, \ell_2} C_{\ell, \ell_1, \ell_2}^\dagger \left[ \widehat{\fpre}(\x, \ell_1) \otimes \widehat{\fpre}(\x, \ell_2) \right] C_{\ell, \ell_1, \ell_2}, \quad 0 \leq \ell \leq \ell_{\max}.
\end{equation*}
Here, \(k_{\max}\), \(\ell_{\max}\) are the respective Fourier frequency cutoffs, and \(\eta_{k,k'}\), \(\eta_{\ell, \ell_1, \ell_2}\) are learnable weights. The tensor product operation on large tensors is the most computationally demanding step in the network, especially for the 3D case. To optimize this, we utilize the \texttt{GElib} library \citep{gelib}, which efficiently implements this operation for the 3D case.

\subsection{Normalization}\label{sec: Normalization}
Under the action of any \((\t,R)\in \SE(d)\), the features extracted from steerable convolutions transform as
\begin{equation*}
    \h\fpre(\x, \rho)\mapsto [(\t,R)^{-1} \h\fpre](\x, \rho) =  \rho(R)^\dagger\h\fpre(R\x+\t, \rho).
\end{equation*}
Since the matrices \(\rho(R)\) are unitary, the \(\mathcal{L}_2\) norm of each Fourier component remains invariant under this transformation. 
Motivated by this property, we normalize each Fourier vector by its norm to achieve an equivariant form of normalization,
\begin{equation*}
    \h \fout(\x, \rho) = \frac{\h\fpre(\x, \rho)}{\sqrt{\sum_{\rho} ||\h\fpre(\x, \rho)||_2^2}}.
\end{equation*}
This normalization not only preserves equivariance but also enhances stability within the network by preventing numerical overflow and facilitating faster convergence during training.

\subsection{Pooling}\label{sec: implementataion pooling}
Feature maps produced by steerable convolutions store, at each location $\x$, an equivariant vector $\h\fpre(\x, \rho)$ corresponding to each irreducible representation $\rho$. Since the group action is inherently a linear operation, multiple such vectors at different locations can be averaged, and the resulting vector will still remain equivariant. Exploiting this property, we use average pooling (instead of the commonly used max pooling) to ensure that equivariance is preserved.
The operation of average pooling is expressed as
\begin{equation*}
    \h\fout(\x, \rho) = \frac{1}{|\mathcal{N}(\x)|} \sum_{\x' \in \mathcal{N}(\x)} \h\fpre(\x', \rho).
\end{equation*}
Here, \(\mathcal{N}(\x)\) represents the neighborhood of \(\x\). This operation aggregates the angular information from neighboring locations while preserving the equivariant properties of the input field.

\subsection{Flattening layer}\label{sec: Flattening}
For the flattening layer, we apply a similar approach to that used for normalization, leveraging the property that the norm of equivariant vectors remains invariant under group actions. Specifically, we compute the mean of the sum of the equivariant vectors across all spatial locations and then calculate its norm
\begin{equation*}
    \sqrt{\sum_{\rho} \left\| \frac{1}{|\mathcal{F}|} \sum_{\x \in \mathcal{F}} \h\fout(\x, \rho) \right\|_2^2},
\end{equation*}
where \(\mathcal{F}\) represents the set of all spatial locations. This operation makes the network \emph{invariant} to rotations and translations of the input by aggregating equivariant features into a single scalar value. The resulting output can then be passed on to subsequent linear layers for further processing.

\section{Supplementary Experimental Results}

In the following, we provide supplementary results to complement the main experiments.

\subsection{Comparison of filters}
We now present the experimental results underlying the bar plot shown in Figure \ref{fig: bar comparison}.

\begin{table}[h]
    \centering
    \setlength{\tabcolsep}{3pt}
    \begin{tabular}{c|c|c|ccc|c|cc}
    \hline
         \multicolumn{2}{c|}{\gray} & \gray Fourier & \multicolumn{3}{c|}{\gray Accuracy} & \gray Params & \multicolumn{2}{c}{\gray Avg Run Time}\\
         \cline{4-6}
         \cline{8-9}
          \multicolumn{2}{c|}{\multirow{-2}{*}{\gray Dataset}}  & \gray Cutoff & \gray NN &\gray Linear & \gray Cartesian & \gray (\(\sim\times 10^6\)) & \gray Training(ms) & \gray Inference(ms)\\
        \hline
        \multirow{16}{*}{\rotatebox{90}{\parbox{5cm}{\thead{Trained With\\Rotation Augmentation}}}}
        &\multirow{8}{*}{\thead{Rot MNIST\\Fashion MNIST}}
        &\multirow{2}{*}{$k=4$}   
                  & \(99.10_{\pm 0.02}\) & \(99.09_{\pm 0.02}\) & \(99.11_{\pm 0.01}\) & \multirow{2}{*}{0.56} & \multirow{2}{*}{1.26} & \multirow{2}{*}{0.45}\\
                 &&& \(89.46_{\pm 0.10}\) & \(89.31_{\pm 0.09}\)  & \(89.14_{\pm 0.07}\) &\\
                \cline{3-9}
        &&\multirow{2}{*}{$k=8$}   
                  & \(99.15_{\pm 0.03}\) & \(99.14_{\pm 0.02}\) & \(99.14_{\pm 0.03}\) & \multirow{2}{*}{2.23} & \multirow{2}{*}{2.63} & \multirow{2}{*}{0.83}\\
                 &&& \(89.92_{\pm 0.09}\) & \(89.77_{\pm 0.12}\)  & \(89.64_{\pm 0.08}\) &\\
                \cline{3-9}
        &&\multirow{2}{*}{$k=12$}   
                  & \(99.13_{\pm 0.02}\) & \(\mathbf{99.15_{\pm 0.02}}\) & \(99.14_{\pm 0.02}\) & \multirow{2}{*}{5.00} & \multirow{2}{*}{4.91} & \multirow{2}{*}{1.60}\\
                 &&& \(\mathbf{90.10_{\pm 0.07}}\) & \(89.87_{\pm 0.07}\)  & \(89.78_{\pm 0.14}\) &\\
                \cline{3-9}
        &&\multirow{2}{*}{$k=16$}   
                  & \(99.14_{\pm 0.02}\) & \(99.15_{\pm 0.02}\) & \(99.10_{\pm 0.02}\) & \multirow{2}{*}{8.88} & \multirow{2}{*}{9.03} & \multirow{2}{*}{2.55}\\
                &&& \(90.01_{\pm 0.08}\) & \(89.87_{\pm 0.08}\)  & \(89.46_{\pm 0.09}\) &\\
   \cline{2-9}
    \noalign{\vskip\doublerulesep
         \vskip-\arrayrulewidth}
    \cline{2-9}
        &\multirow{8}{*}{\thead{ModelNet10 \\ SHREC17}}
        &\multirow{2}{*}{$\ell=0$}  
            & \(61.78_{\pm 4.95}\)  & \(61.93_{\pm 2.62}\)  & \(39.71_{\pm 8.34}\) & \multirow{2}{*}{0.08} & \multirow{2}{*}{1.90} & \multirow{2}{*}{0.90}\\
            &&& \(54.97_{\pm 1.83}\)  & \(52.09_{\pm 0.54}\)  & \(38.99_{\pm 6.82}\) & \\
        \cline{3-9}
        &&\multirow{2}{*}{$\ell=1$}   
             & \(77.18_{\pm 1.93}\)  & \(78.24_{\pm 2.03}\)  & \(75.91_{\pm 1.54}\) & \multirow{2}{*}{0.12} & \multirow{2}{*}{4.21} & \multirow{2}{*}{1.55}\\
             &&& \(66.02_{\pm 1.13}\)  & \(68.10_{\pm 1.14}\)  & \(65.53_{\pm 1.15}\) & \\
        \cline{3-9}
        &&\multirow{2}{*}{$\ell=2$}   
             & \(86.06_{\pm 1.13}\)  & \(86.79_{\pm 0.25}\)  & \(84.46_{\pm 0.90}\) & \multirow{2}{*}{0.28} & \multirow{2}{*}{12.42} & \multirow{2}{*}{4.10}\\
             &&& \(74.51_{\pm 0.57}\)  & \(75.59_{\pm 0.43}\)  & \(74.25_{\pm 0.36}\) & \\
        \cline{3-9}
        &&\multirow{2}{*}{$\ell=3$}   
             & \(87.97_{\pm 0.52}\)  & \(\mathbf{88.67_{\pm 0.39}}\)  & \(87.58_{\pm 0.51}\) & \multirow{2}{*}{0.80} & \multirow{2}{*}{50.75} & \multirow{2}{*}{16.86}\\
             &&& \(77.24_{\pm 0.39}\)  & \(\mathbf{77.52_{\pm 0.28}}\)  & \(76.92_{\pm 0.24}\) & \\

    \hline
        \multicolumn{2}{c|}{\gray} & \gray Fourier & \multicolumn{3}{c|}{\gray Accuracy} & \gray Params  & \multicolumn{2}{c}{\gray Avg Run Time}\\
         \cline{4-6}
         \cline{8-9}
          \multicolumn{2}{c|}{\multirow{-2}{*}{\gray Dataset}}  & \gray Cutoff & \gray NN &\gray Linear & \gray Cartesian & \gray (\(\sim\times 10^6\)) & \gray Training(ms) & \gray Inference(ms)\\
    \hline
    \multirow{16}{*}{\rotatebox{90}{\parbox{5cm}{\thead{Trained Without\\Rotation Augmentation}}}}
        &\multirow{8}{*}{\thead{Rot MNIST\\Fashion MNIST}}
        &\multirow{2}{*}{$k=4$}   
                  & \(\mathbf{98.35_{\pm 0.05}}\) & \(98.26_{\pm 0.05}\) & \(98.08_{\pm 0.10}\) & \multirow{2}{*}{0.56} & \multirow{2}{*}{1.26} & \multirow{2}{*}{0.45}\\
                 &&& \(\mathbf{78.63_{\pm 1.42}}\) & \(78.29_{\pm 0.97}\)  & \(78.26_{\pm 0.99}\) &\\
                \cline{3-9}
        &&\multirow{2}{*}{$k=8$}   
                  & \(97.67_{\pm 0.08}\) & \(97.92_{\pm 0.13}\) & \(97.12_{\pm 0.29}\) & \multirow{2}{*}{2.23} & \multirow{2}{*}{2.63} & \multirow{2}{*}{0.83}\\
                 &&& \(74.86_{\pm 1.55}\) & \(74.82_{\pm 1.22}\)  & \(55.42_{\pm 1.42}\) &\\
                \cline{3-9}
        &&\multirow{2}{*}{$k=12$}   
                  & \(97.69_{\pm 0.07}\) & \(97.93_{\pm 0.12}\) & \(97.26_{\pm 0.13}\) & \multirow{2}{*}{5.00} & \multirow{2}{*}{4.91} & \multirow{2}{*}{1.60}\\
                 &&& \(75.13_{\pm 1.09}\) & \(75.64_{\pm 1.57}\)  & \(49.27_{\pm 0.75}\) &\\
                \cline{3-9}
        &&\multirow{2}{*}{$k=16$}   
                  & \(97.67_{\pm 0.11}\) & \(97.89_{\pm 0.08}\) & \(95.63_{\pm 0.31}\) & \multirow{2}{*}{8.88} & \multirow{2}{*}{9.03} & \multirow{2}{*}{2.55}\\
                &&& \(73.82_{\pm 1.02}\) & \(74.34_{\pm 0.78}\)  & \(47.14_{\pm 0.81}\) &\\
   \cline{2-9}
    \noalign{\vskip\doublerulesep
         \vskip-\arrayrulewidth}
    \cline{2-9}
        &\multirow{8}{*}{\thead{ModelNet10 \\ SHREC17}}
        &\multirow{2}{*}{$\ell=0$}  
            & \(35.13_{\pm 2.42}\)  & \(32.25_{\pm 1.63}\)  & \(23.46_{\pm 4.06}\) & \multirow{2}{*}{0.08} & \multirow{2}{*}{1.90} & \multirow{2}{*}{0.90}\\
            &&& \(21.20_{\pm 1.44}\)  & \(20.19_{\pm 1.77}\)  & \(12.68_{\pm 1.68}\) & \\
        \cline{3-9}
        &&\multirow{2}{*}{$\ell=1$}   
             & \(36.77_{\pm 3.04}\)  & \(39.32_{\pm 2.03}\)  & \(34.39_{\pm 0.84}\) & \multirow{2}{*}{0.12} & \multirow{2}{*}{4.21} & \multirow{2}{*}{1.55}\\
             &&& \(24.09_{\pm 2.15}\)  & \(25.30_{\pm 2.20}\)  & \(19.18_{\pm 1.18}\) & \\
        \cline{3-9}
        &&\multirow{2}{*}{$\ell=2$}   
             & \(47.95_{\pm 4.46}\)  & \(48.63_{\pm 2.67}\)  & \(38.06_{\pm 1.76}\) & \multirow{2}{*}{0.28} & \multirow{2}{*}{12.42} & \multirow{2}{*}{4.10}\\
             &&& \(40.13_{\pm 3.27}\)  & \(43.10_{\pm 1.84}\)  & \(30.05_{\pm 2.72}\) & \\
        \cline{3-9}
        &&\multirow{2}{*}{$\ell=3$}   
             & \(\mathbf{58.28_{\pm 3.76}}\)  & \(54.40_{\pm 1.80}\)  & \(40.79_{\pm 2.29}\) & \multirow{2}{*}{0.80} & \multirow{2}{*}{50.75} & \multirow{2}{*}{16.86}\\
             &&& \(47.84_{\pm 2.93}\)  & \(\mathbf{48.62_{\pm 2.59}}\)  & \(31.18_{\pm 1.91}\) & \\
        \hline
    \end{tabular}
    
    \caption{Comparison of Nearest Neighbor (NN) interpolation, Linear interpolation, and Cartesian filters on four benchmark datasets. Error bars show \(90\%\) Wald confidence intervals over five independent runs. Reported runtimes correspond to the average per-sample cost of one forward and backward pass during training and one forward pass during inference.}
    \label{table: experiments}
\end{table}

\newpage
\subsection{Sensitivity Analysis}
While Figure \ref{fig: sensitivity} presented the sensitivity analysis for models with the highest frequency cutoff, here we report the results for all models. In this expanded view, Figure \ref{fig: sensitivity} corresponds to the final column of the figure.
\begin{figure}[h]
    \centering
    \includegraphics[width=\linewidth]{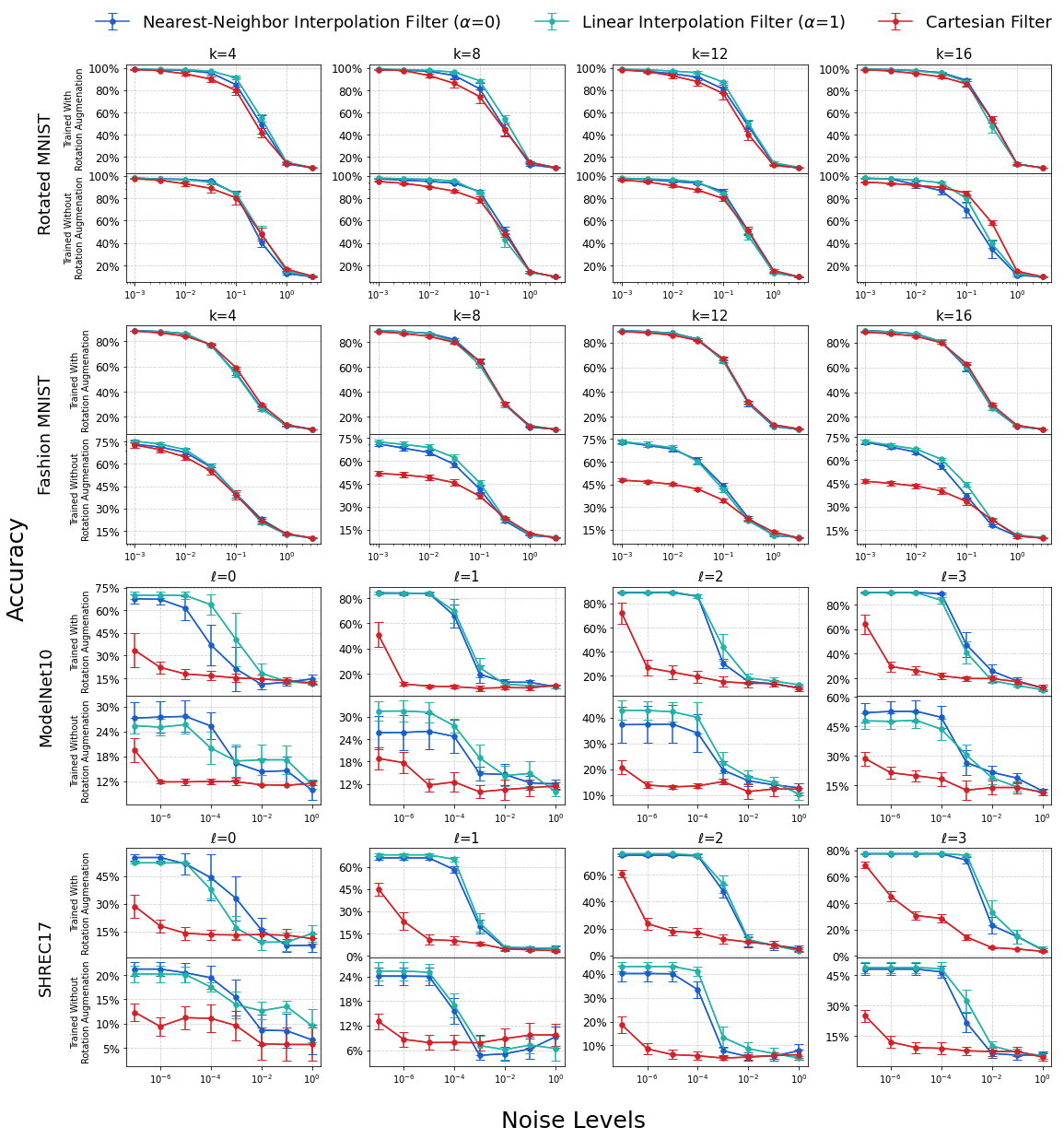}
    \caption{Effect of Gaussian noise perturbations on the performance of the model. For each noise level, every trained model was evaluated across \(5\) runs, with \(5\) models trained per configuration. Variance was estimated using $\text{Var}(X) = \text{Var}(\mathbb{E}[X|\mathcal{M}]) + \mathbb{E}[\text{Var}(X|\mathcal{M})]$, where $X$ denotes accuracy and $\mathcal{M}$ a trained model. Error bars indicate \(90\%\) Wald confidence intervals based on this variance.}
    \label{fig: sensitivity all}
\end{figure}
\section{Proofs}\label{sec: proofs}

\subsection{Proof of Theorem \ref{thme: equivariant linear map}}

\begin{proof} 
    Let \(e\) denote the identity element of \(\G \).  
    Because \(\phi\) is an equivariant, for every \(f\in C_{c}(\mathcal X)\) and \(g\in \G \), we have
    \begin{equation*}
        \phi[f](g)=\phi[f](g\!\cdot\!e)=g^{-1}\!\cdot\!\phi[f](e)=\phi[g^{-1}\!\cdot\!f](e).
    \end{equation*}
    Define \(\Lambda:C_{c}(\mathcal X)\to\CC\) such that \(\Lambda(f):=\phi[f](e)\). Since the operator \(\phi\) is bounded, there is a constant \(C>0\) such that for all \(f\in C_{c}(\mathcal X)\)
    \begin{equation*}
        |\Lambda(f)| \;=\;|\phi[f](e)|\;\le\;C\|f\|_{\infty}
    \end{equation*}
    Thus, \(\Lambda\) is a bounded linear functional on \(C_{c}(\mathcal X)\). Because \(\mathcal X\) is a locally compact Hausdorff space, \(C_{c}(\mathcal X)\) is sup‑norm dense in 
    \(C_{0}(\mathcal X)\), the space of continuous functions in \(\mathcal{X}\) vanishing at infinity. Boundedness therefore extends \(\Lambda\) uniquely to \(C_{0}(\mathcal X)\).  
    By the Riesz--Markov--Kakutani theorem (Theorem \ref{thm: riesz rep thm}), there exists a \emph{unique} complex Radon measure \(\lambda\) on \(\mathcal X\) such that for all \(f\in C_0(\mathcal{X})\), we have
    \begin{equation*}
        \Lambda(f)=\int_{\mathcal X} f\,d\lambda.
    \end{equation*}
    Now suppose \(\mathcal{X}\) is \(\sigma\)-compact and equipped with a \(\sigma\)-finite measure \(\sigma\).
    By the Hahn--Jordan decomposition theorem (Theorem \ref{thm: hahn decomp thm}) there exists four unique non‑negative Radon measures \(\lambda_1, \lambda_2, \lambda_3, \lambda_4\) such that
    \begin{equation*}
        \lambda \;=\;(\lambda_{1}-\lambda_{2}) \;+\; i\,(\lambda_{3}-\lambda_{4}).
    \end{equation*} 
    Since \(\mathcal X\) is \(\sigma\)-compact, and \(\lambda_j\) is finite on compact sets (by definition of Radon measure), each \(\lambda_{j}\) is \(\sigma\)-finite. By Lebesgue decomposition theorem (Theorem \ref{thm: leb decomp thm}), for each \(1\leq j\leq 4\), there exists a unique real valued functions \(w_j\in \mathcal{L}_1(\mathcal{X})\) and a unique measure \(\nu_j\), which are singular with respect to \(\sigma\), such that \(d\lambda_i = w_id\sigma + d\nu_i\). Set  
    \begin{align*}
        &w:=(w_{1}-w_{2})+i(w_{3}-w_{4}),\\
        &\nu:=(\nu_{1}-\nu_{2})+i(\nu_{3}-\nu_{4})
    \end{align*}
    so that \(d\lambda = w\,d\sigma + d\nu\). Now, for any \(f\in C_c(\mathcal{X})\), let \(\Omega(f) \subseteq \mathcal{X}\) denote the compact support of \(f\). For any \(g\in \G \), the support for the translate \(g^{-1}\cdot f\) is given by \(g^{-1}\cdot \Omega(f) \). Since the action of \(\G \) on \(\mathcal{X}\) is continuous, \(g^{-1}\cdot \Omega(f) \) is also compact, as it is the image of a compact set under a continuous map. Hence, for any \(f\in C_{c}(\mathcal X)\) and \(g\in \G \), we have \(g^{-1}\cdot f\in C_{c}(\mathcal X)\). Therefore,
    \begin{equation*}
        \phi[f](g) \;=\; \phi[g^{-1}\!\cdot\!f](e) \;=\;\Lambda(g^{-1}\!\cdot\!f) \;=\;\int_{\mathcal X} g^{-1}\!\cdot\!f\,d\lambda \;=\;\bigl\langle g^{-1}\!\cdot\!f,\;w\bigr\rangle  \;+\;\int_{\mathcal X} g^{-1}\!\cdot\!f\,d\nu.
    \end{equation*}
\end{proof}

\subsection{Proof of Proposition \ref{prop: diagF}}

\begin{proof} 
    Using the inverse Fourier transform on \(\G \times \G \), 
    \begin{equation*}
        f(g,g)\<=\sum_{\rho_1,\rho_2} \fr{d_{\rho_1} d_{\rho_2}}{\mu(\G )^2}\,\tr\sqbbig{\h f(\rho_1,\rho_2)\,\rho_1(g^{-1})\otimes \rho_2(g^{-1})}.
    \end{equation*}
    Since the trace of a matrix is invariant to similarity transformations, we can rewrite the trace term as 
    \begin{multline*}
        \tr\sqbBig{C_{\rho_1,\rho_2}^\dag\, \h f(\rho_1,\rho_2)\,C_{\rho_1,\rho_2}\cdot
        C_{\rho_1,\rho_2}^\dag\rho_1(g^{-1})\otimes \rho_2(g^{-1})C_{\rho_1,\rho_2}}=\\
        \tr\sqbbigg{\brbigg{C_{\rho_1,\rho_2}^\dag\, \h f(\rho_1,\rho_2)\,C_{\rho_1,\rho_2}} 
        \brbigg{\bigoplus_{\rho\in\Rcal} \bigoplus_{i=1}^{\kappa(\rho_1,\rho_2,\rho)} \rho(g^{-1})}}.
    \end{multline*}
    Now, since the second factor under the trace is block diagonal, the first one can be brought to similar form, giving
    \begin{equation*}
        \tr\sqbbigg{ \brbigg{\bigoplus_{\rho\in\Rcal} \bigoplus_{i=1}^{\kappa(\rho_1,\rho_2,\rho)} F_{\rho_1,\rho_2}^{\rho,i}} \brbigg{\bigoplus_{\rho\in\Rcal} \bigoplus_{i=1}^{\kappa(\rho_1,\rho_2,\rho)} \rho(g^{-1})}}= \sum_\rho\sum_{i=1}^{\kappa(\rho_1,\rho_2,\rho)} \tr\sqbbig{ F_{\rho_1,\rho_2}^{\rho,i}\, \rho(g^{-1})},
    \end{equation*}
    where 
    \begin{equation*}
        F_{\rho_1,\rho_2}^{\rho,i}=\sqbBig{C_{\rho_1,\rho_2}^\dag\, \h f(\rho_1,\rho_2)\,C_{\rho_1,\rho_2}}_{\rho,i}.
    \end{equation*}
    In summary, collecting terms indexed by the same \(\rho\)
    \begin{equation*}
        \diagf{f}(g)=f(g,g)=\sum_{\rho} \sum_{\rho_1,\rho_2} \sum_{i=1}^{\kappa(\rho_1,\rho_2,\rho)}\fr{d_{\rho_1} d_{\rho_2}}{\mu(\G )^2}\:
    \tr\sqbbig{ F_{\rho_1,\rho_2}^{\rho,i}\, \rho(g^{-1})}.
    \end{equation*}
    Comparing this to the form of the inverse Fourier transform on the diagonal of \(\G \times \G \), 
    \begin{equation*}
        \diagf{f}(g)=\sum_\rho \fr{d_\rho}{\mu(\G )}\: \tr\sqbbig{\h f(\rho)\,\rho(g^{-1})}
    \end{equation*}
    proves our proposition. 
\end{proof}

\subsection{Proof of Proposition \ref{prop: SHT}}

\begin{proof}
    First we need to show that the integral on the left hand side exists. Fix \(f\in \mathcal{L}_2(\SS^{d-1})\). Then, for any \(1\leq i,j\leq d_\rho\), we have
    \begin{align*}
         \int_{\SO(d)}|f(R\e)\rho_{ij}(R)|\,d\mu(R) &\leq \int_{\SO(d)}|f(R\e)|^2\,d\mu(R)\int_{\SO(d)}|\rho_{ij}(R)|^2\,d\mu(R)\\ &
         =\frac{1}{d_\rho}  \int_{\SO(d)}|f(R\e)|^2\,d\mu(R)= \int_{\SS^{d-1}}|f(s)|^2\,d\sigma(s) <\infty.
    \end{align*}
    The first inequality follows from Cauchy--Schwartz inequality. The next equality follows from the fact that \(\{\sqrt{d_\rho} \rho_{ij}\}\) forms an orthonormal basis of \(\mathcal{L}_2(\SO(d))\), which is a consequence of the Peter--Weyl Theorem (Theorem \ref{thm: peter thm}). Finally, the last equality follows from Lemma \ref{lemma: sphere integral}. Define
    \begin{align}
        \h f(\rho) 
        :=& \int_{\SS^{d-1}} f(s) \frac{\Y^{(\rho)}(s)}{\|\Y^{(\rho)}(s)\|_2}\,d\sigma(s) \nonumber \\
        =& \sqrt{\frac{\mathcal{A}(\SS^{d-1})}{d_\rho}}\int_{\SS^{d-1}} f(s) \Y^{(\rho)}(s)\,d\sigma(s)\nonumber \\
        &(\textit{using Lemma }\ref{lemma: spherical harmonics normalization})\nonumber\\
        =& \frac{1}{\sqrt{\mathcal{A}(\SS^{d-1})d_\rho}}\int_{\SS^{d-1}} f(s(\theta)) \Y^{(\rho)}(s(\theta))\omega(\theta)d\theta\label{eq: SHT proof}\\
        &(\textit{by equation }\eqref{eq: sphere integral})\nonumber
    \end{align}
    Since \(\Y^{(\rho)}\) form an orthonormal basis of \(\mathcal{L}_2(\SS^{d-1})\), we have
    \begin{align}
        f(s(\theta)) 
        =& \sum_{\rho}  \Y^{(\rho)}(s(\theta))^\dagger\left(\int_{\Theta^{d-1}} f(s(\theta')) \Y^{(\rho)}(s(\theta'))\omega(\theta')\,d\theta'\right)\nonumber \\
        &(\textit{by equations }\eqref{eq: ISHT background} \textit{ and }\eqref{eq: SHT background})\nonumber\\
        =&  \sum_{\rho}d_\rho  \sqrt{\frac{\mathcal{A}(\SS^{d-1})}{d_\rho}}\Y^{(\rho)}(s(\theta))^\dagger\left(\frac{1}{\sqrt{\mathcal{A}(\SS^{d-1})d_\rho}}\int_{\Theta^{d-1}} f(s(\theta')) \Y^{(\rho)}(s(\theta'))\omega(\theta')\,d\theta'\right)\nonumber \\
        =&  \sum_{\rho}d_\rho\frac{\Y^{(\rho)}(s(\theta))^\dagger}{\|\Y^{(\rho)}(s(\theta))\|_2}\left(\frac{1}{\sqrt{\mathcal{A}(\SS^{d-1})d_\rho}}\int_{\Theta^{d-1}} f(s(\theta')) \Y^{(\rho)}(s(\theta'))\omega(\theta')\,d\theta'\right)\nonumber \\
        =&  \sum_{\rho} d_\rho \frac{\Y^{(\rho)}(s(\theta))^\dagger}{\|\Y^{(\rho)}(s(\theta))\|_2}\h f(\rho)\label{eq: SHT series}\\
        &(\textit{by equation }\eqref{eq: SHT proof})\nonumber
    \end{align}
    Therefore,
    \begin{align*}
        \int_{\SO(d)}f(R\e)\rho_{ij}(R)\,d\mu(R)
        =& \int_{\SO(d)}\sum_{\rho'}d_{\rho'}\frac{\Y^{(\rho')}(R\e)^\dagger}{\|\Y^{(\rho')}(R\e)\|_2} \h f(\rho')\rho_{ij}(R)\,d\mu(R)\\
         & (\textit{by equation } \eqref{eq: SHT series}) \\
        =& \int_{\SO(d)}\sum_{\rho'} d_{\rho'}\frac{\Y^{(\rho')}(\e)^\dagger}{\|\Y^{(\rho')}(\e)\|_2} \rho'(R)^\dagger \h f(\rho')\rho_{ij}(R)\,d\mu(R) \hspace{1cm} \\
         & (\textit{by equation } \eqref{eq: spherical harmonics}) \\
        =& \int_{\SO(d)}\sum_{\rho'}d_{\rho'}\sum_{k,l=1}^{d_{\rho'}} \frac{\Y^{(\rho')}_l(\e)^*}{\|\Y^{(\rho')}(\e)\|_2}\h f_k(\rho')\rho'_{kl}(R)^*\rho_{ij}(R)\,d\mu(R)\\
        =& \sum_{\rho'} d_{\rho'}\sum_{k,l=1}^{d_{\rho'}}\frac{\Y^{(\rho')}_l(\e)^*}{\|\Y^{(\rho')}(\e)\|_2}\h f_k(\rho')\int_{\SO(d)}\rho'_{kl}(R)^*\rho_{ij}(R)\,d\mu(R)\\
        =& \sum_{\rho'}d_{\rho'}\sum_{k,l=1}^{d_{\rho'}} \frac{\Y^{(\rho')}_l(\e)^*}{\|\Y^{(\rho')}(\e)\|_2} \h f_k(\rho')  \frac{\delta_{\rho\rho'}\delta_{ik}\delta_{jl}}{d_\rho}\\
         & (since \textit{\(\{\sqrt{d_\rho} \rho_{ij}\}\) forms an orthonormal basis of \(\mathcal{L}_2(\SO(d))\)}) \\
        =&  \h f_i(\rho) \frac{\Y^{(\rho)}_j(\e)^*}{\|\Y^{(\rho)}(\e)\|_2}.
    \end{align*}
    The penultimate equality follows from the fact that \(\{\sqrt{d_\rho} \rho_{ij}\}\) forms a basis of \(\mathcal{L}_2(\SO(d))\). This completes the proof.
\end{proof}

\subsection{Proof of Theorem \ref{thm: first layer}}

\begin{proof} 
    Note that, for any \(\x\in \ZZ^d\), we have
    \begin{align}
         & \h \fpre(\x, \rho) \nonumber\\
        =& \int_0^h \h{\fpatch_\x}(r,\,\rho){\h{\wbar{w}}(r,\rho)} r^{d-1}\, dr \nonumber \\
         & (\textit{by equation } \eqref{eq: fpre first fourier})\nonumber
         \\
        =& \int_0^h \int_{\SO(d)}\fpatch_\x(r,R)\rho(R)\,d\mu(R){\h{\wbar{w}}(r,\rho)} r^{d-1}\, dr \nonumber \\
         & (\textit{by equation }\eqref{eq: fpatch first fourier}) \nonumber
         \\
        =& \int_0^h \int_{\SO(d)}\I[\fin](\x+rR\e)\rho(R)\,d\mu(R){\h{\wbar{w}}(r,\rho)} r^{d-1}\, dr \nonumber \\ 
         & (\textit{using definition of }\fpatch_\x)\nonumber
         \\
        =& \int_0^h \int_{\SO(d)}\sum_{\y\in \ZZ^d}\fin(\y)\I(rR\e + \x,\y)\rho(R)\,d\mu(R){\h{\wbar{w}}(r,\rho)} r^{d-1}\, dr \nonumber \\ 
         & (\textit{by Definition }\ref{defn: interpolation}(a))\nonumber
         \\
        =& \int_0^h \int_{\SO(d)}\sum_{\y\in \ZZ^d}\fin(\y)\I(rR\e,\y-\x)\rho(R)\,d\mu(R){\h{\wbar{w}}(r,\rho)} r^{d-1}\, dr \nonumber\\ 
         & (\textit{by Definition }\ref{defn: interpolation}(b))\nonumber
         \\
        =& \int_0^h \int_{\SO(d)}\sum_{\y\in \ZZ^d}\fin(\x+\y)\I(rR\e,\y)\rho(R)\,d\mu(R){\h{\wbar{w}}(r,\rho)} r^{d-1}\, dr\nonumber \\
         & (\textit{using the transformation }\y\mapsto\y+\x) \nonumber
         \\
        =& \sum_{\y\in \ZZ^d}\fin(\x+\y)\left[\int_0^h r^{d-1}\left(\int_{\SO(d)}\I(rR\e,\y)\rho(R)\,d\mu(R)\right){\h{\wbar{w}}(r,\rho)}\, dr\right]  \nonumber
        \\
        =& \sum_{\y\in \ZZ^d}\fin(\x+\y)\left[\int_0^h r^{d-1}\left(\int_{\mathbb{S}^{d-1}}\I(rs,\y)\Y^{(\rho)}(s)\,d\sigma(s)\right)\Y^{(\rho)}(\e)^\dagger {\h{\wbar{w}}(r,\rho)}\, dr\right] \nonumber\\
         & (\textit{using Proposition }\ref{prop: SHT} \textit{ and Definition }\ref{defn: interpolation}(c))\nonumber
         \\
         \approx & \sum_{\y\in \ZZ^d}\hspace{-4pt}\fin(\x+\y)\hspace{-4pt}\left[\frac{h}{n_r}\frac{2\pi^{d-1}}{n_a^{d-1}\mathcal{A}(\SS^{d-1})}\sum_{r=1}^{n_r} \left(\frac{rh}{n_r}\right)^{d-1}\hspace{-4pt}\left(\sum_{\theta\in\Theta_{n_a}^{d-1}}\I\left(\frac{rh}{n_r}s(\theta),\y\right)\hspace{-0.8pt}\Y^{(\rho)}(s(\theta))\,\omega(\theta)\right)\hspace{-3pt}\Y^{(\rho)}(\e)^\dagger {\h{\wbar{w}}\left(\frac{rh}{n_r},\rho\right)}\hspace{-3pt}\right] \nonumber\\
         & (\textit{by discretizing the radial component into \(n_r\) grid points and equation } \eqref{eq: sphere integral approx})\nonumber
         \\
         =& \sum_{\y\in \ZZ^d}\fin(\x+\y)\left[\sum_{r=1}^{n_r} \left(\frac{r^{d-1}}{n_r^dn_a^{d-1}}\sum_{\theta\in\Theta_{n_a}^{d-1}}\I\left(\frac{rh}{n_r}s(\theta),\y\right)\Y^{(\rho)}(s(\theta))\,\omega(\theta)\right)\frac{2\pi^{d-1}h^d}{\mathcal{A}(\SS^{d-1})}\Y^{(\rho)}(\e)^\dagger {\h{\wbar{w}}\left(\frac{rh}{n_r},\rho\right)}\right] \nonumber
         \\
         =& \sum_{\y\in \ZZ^d}\fin(\x+\y)\left[\sum_{r=1}^{n_r}M_{r}^{(\rho)}(\y) \left(\frac{2\pi^{d-1}h^d}{\mathcal{A}(\SS^{d-1})}\Y^{(\rho)}(\e)^\dagger {\h{\wbar{w}}\left(\frac{rh}{n_r},\rho\right)}\right)\right]. \label{eq: fpre first proof 1}\\
         & (\textit{using definition of \(M_r^{(\rho)}\) in equation }\eqref{eq: steerable filter basis first layer})\nonumber
    \end{align}
    Define \(w_r^{(\rho)} := \frac{2\pi^{d-1}h^d}{\mathcal{A}(\SS^{d-1})}\Y^{(\rho)}(\e)^\dagger {\h{\wbar{w}}\left(\frac{rh}{n_r},\rho\right)}\), for \(1\leq r\leq n_r\). Since \(\h{\wbar{w}}(\cdot,\rho)\in \CC^{d_\rho\times d_\rho}\) and \(\Y^{(\rho)}(\e)\in \CC^{d_\rho}\), it follows that \(w_r^{(\rho)}\in \CC^{1\times d_\rho}\). Now, assuming one input and output channel, \(w_r^{(\rho)}\) reduces to a complex scalar, and therefore, \eqref{eq: fpre first proof 1} becomes
    \begin{equation*}
        \h \fpre(\x, \rho)= \sum_{\y\in \ZZ^d}\left[\sum_{r=1}^{n_r}w_r^{(\rho)} M_{r}^{(\rho)}(\y)\right]\fin(\x+\y).
    \end{equation*}
\end{proof}

\subsection{Proof of Theorem \ref{thm: higher layer classical}}

\begin{proof}
    Note that, for any \(\x\in \ZZ^d\), we have
    \begin{align}
         &\h \fpre(\x, \rho) \nonumber
         \\
        =& \int_0^h \substack{\textrm{CG} \\ \rho_1, \rho_2\to \rho}\left(\h{\fpatch_\x}(r,\,\rho_1, \rho_2){\h{\wbar{w}}(r, \rho_1, \rho_2)}\right) r^{d-1}\, dr\nonumber\\
         &(\textit{by equation }\eqref{eq: fpre higher fourier}) \nonumber 
         \\
        =& \int_0^h \substack{\textrm{CG} \\ \rho_1, \rho_2\to \rho}\left(\left[\int_{\SO(d)}\int_{\SO(d)}\fpatch_\x(r,R,R')(\rho_1(R)\otimes \rho_2(R'))\,d\mu(R)d\mu(R')\right]{\h{\wbar{w}}(r, \rho_1, \rho_2)}\right) r^{d-1}\, dr \nonumber\\
        &(\textit{by equation }\eqref{eq: fpatch higher fourier}) \nonumber
        \\
        =& \int_0^h \substack{\textrm{CG} \\ \rho_1, \rho_2\to \rho}\left(\left[\int_{\SO(d)}\left(\int_{\SO(d)}\fpatch_\x(r,R,R')\rho_1(R)\,d\mu(R)\right)\otimes \rho_2(R')d\mu(R')\right]{\h{\wbar{w}}(r, \rho_1, \rho_2)}\right) r^{d-1}\, dr \nonumber
        \\
        =& \int_0^h \substack{\textrm{CG} \\ \rho_1, \rho_2\to \rho}\left(\left[\int_{\SO(d)}\I[\fin](\x+rR'\e, \rho_1) \otimes \rho_2(R')d\mu(R')\right]{\h{\wbar{w}}(r, \rho_1, \rho_2)}\right) r^{d-1}\, dr \nonumber\\
        &(\textit{using definition of }\fpatch_\x) \nonumber
        \\
        =& \int_0^h \substack{\textrm{CG} \\ \rho_1, \rho_2\to \rho}\left(\left[\int_{\SO(d)}\sum_{\y\in \ZZ^d}\fin(\y, \rho_1)\I(\x+rR'\e,\y)\otimes \rho_2(R')d\mu(R')\right]{\h{\wbar{w}}(r, \rho_1, \rho_2)}\right) r^{d-1}\, dr \nonumber\\
        &(\textit{by Definition }\ref{defn: interpolation}(a)) \nonumber
        \\
        =& \int_0^h \substack{\textrm{CG} \\ \rho_1, \rho_2\to \rho}\left(\left[\int_{\SO(d)}\sum_{\y\in \ZZ^d}\fin(\y, \rho_1)\I(rR'\e,\y-\x)\otimes \rho_2(R')d\mu(R')\right]{\h{\wbar{w}}(r, \rho_1, \rho_2)}\right) r^{d-1}\, dr \nonumber\\
        &(\textit{by Definition }\ref{defn: interpolation}(b)) \nonumber
         \\
        =& \int_0^h \substack{\textrm{CG} \\ \rho_1, \rho_2\to \rho}\left(\left[\int_{\SO(d)}\sum_{\y\in \ZZ^d}\fin(\x+\y, \rho_1)\I(rR'\e,\y)\otimes \rho_2(R')d\mu(R')\right]{\h{\wbar{w}}(r, \rho_1, \rho_2)}\right) r^{d-1}\, dr\nonumber \\
         & (\textit{using the transformation }\y\mapsto\y+\x) \nonumber 
        \\
        =& \int_0^h \substack{\textrm{CG} \\ \rho_1, \rho_2\to \rho}\left(\left[\sum_{\y\in\ZZ^d}\fin(\x+\y, \rho_1)\otimes r^{d-1}\int_{\SO(d)} \I(rR'\e,\y)\rho_2(R')d\mu(R') \right]{\h{\wbar{w}}(r, \rho_1, \rho_2)}\right) \, dr\nonumber
        \\
        =& \int_0^h \substack{\textrm{CG} \\ \rho_1, \rho_2\to \rho}\left(\left[\sum_{\y\in \ZZ^d}\fin(\x+\y, \rho_1)\otimes r^{d-1}\int_{\mathbb{S}^{d-1}}\I(rs,\y)\Y^{(\rho_2)}(s)\Y^{(\rho_2)}(\e)^\dagger \,d\sigma(s)\right]{\h{\wbar{w}}(r, \rho_1, \rho_2)}\right) \, dr \nonumber\\
         & (\textit{using Proposition }\ref{prop: SHT} \textit{ and Definition }\ref{defn: interpolation}(c))\nonumber
        \\
        \approx& \frac{h}{n_r}\sum_{r=1}^{n_r} \substack{\textrm{CG} \\ \rho_1, \rho_2\to \rho}\left(\left[\sum_{\y\in \ZZ^d}\fin(\x+\y, \rho_1)\otimes \left(\left(\frac{rh}{n_r}\right)^{d-1}\frac{2\pi^{d-1}}{n_a^{d-1}\mathcal{A}(\SS^{d-1})}\sum_{\theta\in\Theta_{n_a}^{d-1}} \I\left(\frac{rh}{n_r}s(\theta),\y\right)\Y^{(\rho_2)}(s(\theta)) \omega(\theta)\right)\right.\right.\nonumber\\
         &\hspace{10cm}\left.\left.\Y^{(\rho_2)}(\e)^\dagger\right]{\h{\wbar{w}}\left(\frac{rh}{n_r}, \rho_1, \rho_2\right)}\right)  \nonumber\\
         & (\textit{by discretizing the radial component into \(n_r\) grid points and equation } \eqref{eq: sphere integral approx})\nonumber
        \\
        =& \sum_{r=1}^{n_r} \substack{\textrm{CG} \\ \rho_1, \rho_2\to \rho}\left(\left[\sum_{\y\in \ZZ^d}\fin(\x+\y, \rho_1)\otimes \left(\frac{r^{d-1}}{n_r^dn_a^{d-1}}\sum_{\theta\in\Theta_{n_a}^{d-1}} \I\left(\frac{rh}{n_r}s(\theta),\y\right)\Y^{(\rho_2)}(s(\theta)) \omega(\theta)\right)\right.\right.\nonumber\\
         &\hspace{10cm}\left.\left.\frac{2\pi^{d-1}h^{d}}{\mathcal{A}(\SS^{d-1})}\Y^{(\rho_2)}(\e)^\dagger\right]{\h{\wbar{w}}\left(\frac{rh}{n_r}, \rho_1, \rho_2\right)}\right)  \nonumber\\
        =& \sum_{r=1}^{n_r} \substack{\textrm{CG} \\ \rho_1, \rho_2\to \rho}\left(\left[\sum_{\y\in\ZZ^d}\fin(\x+\y, \rho_1)\otimes M_r^{(\rho_2)}(\y) \frac{2\pi^{d-1}h^d|\mathcal{F}_{\SO(d)}|}{\mathcal{A}(\SS^{d-1})}\Y^{(\rho_2)}(\e)^\dagger \right]\h{\wbar{w}}\left(\frac{rh}{n_r}, \rho_1, \rho_2\right)\right) \nonumber\\
         & (\textit{using definition of \(M_r^{(\rho)}\) in equation } \eqref{eq: steerable filter basis first layer}) \nonumber
        \\
        =& \sum_{r=1}^{n_r} \sum_{\y\in\ZZ^d}\sum_{\rho_1, \rho_2}\frac{d_{\rho_1}d_{\rho_2}}{d_\rho}{C^{(\rho,\rho_1, \rho_2)}}^\dagger \Bigg(\fin(\x+\y, \rho_1)\otimes \nonumber \\ 
        &\hspace{3cm}\left.M^{(\rho_2)}_r(\y)\frac{2\pi^{d-1}h^d|\mathcal{F}_{\SO(d)}|}{\mathcal{A}(\SS^{d-1})}\Y^{(\rho_2)}(\e)^\dagger \right) \h{\wbar{w}}\left(\frac{rh}{n_r}, \rho_1, \rho_2\right) C^{(\rho,\rho_1, \rho_2)} \nonumber\\
         & (\textit{using equation } \eqref{eq: CG--decomposition} \textit{ and Remark } \ref{remark: CG--matrices}) \nonumber 
        \\
        =& \sum_{r=1}^{n_r} \sum_{\y\in\ZZ^d}\sum_{\rho_1, \rho_2}{C^{(\rho,\rho_1, \rho_2)}}^\dagger \left(\fin(\x+\y, \rho_1)\otimes M^{(\rho_2)}_r(\y) \right) \Bigg(I_{d_{\rho_1}}\otimes \nonumber \\ 
        &\hspace{3cm}\left. \frac{2\pi^{d-1}h^d|\mathcal{F}_{\SO(d)}|d_{\rho_1}d_{\rho_2}}{\mathcal{A}(\SS^{d-1})d_{\rho}}\Y^{(\rho_2)}(\e)^\dagger\right) \h{\wbar{w}}\left(\frac{rh}{n_r}, \rho_1, \rho_2\right) C^{(\rho,\rho_1, \rho_2)} \nonumber
        \\
        =& \sum_{\y\in \ZZ^d} \sum_{\rho_1, \rho_2 }\sum_{r=1}^{n_r}\left[\left(\fin(\x+\y, \rho_1)^\top \otimes M^{(\rho_2)}_r(\y)^\top  \right){C^{(\rho,\rho_1, \rho_2)}}^*\right]^\top \nonumber 
        \\ 
        & \hspace{3cm}\left[{C^{(\rho,\rho_1, \rho_2)}}^\top  \h{\wbar{w}}\left(\frac{rh}{n_r}, \rho_1, \rho_2\right)^\top  \left(I_{d_{\rho_1}}\otimes \frac{2\pi^{d-1}h^d|\mathcal{F}_{\SO(d)}|d_{\rho_1}d_{\rho_2}}{\mathcal{A}(\SS^{d-1})d_{\rho}}\Y^{(\rho_2)}(\e)^*\right) \right]^\top  \nonumber\\
        \approx& \sum_{\y\in \ZZ^d}\sum_{\rho_1, \rho_2\in \mathcal{F}_{\SO(d)}}\sum_{r=1}^{n_r}\left[\left(\fin(\x+\y, \rho_1)^\top \otimes M^{(\rho_2)}_r(\y)^\top  \right){C^{(\rho,\rho_1, \rho_2)}}^*\right]^\top \nonumber 
        \\ 
        & \hspace{3cm}\left[{C^{(\rho,\rho_1, \rho_2)}}^\top  \h{\wbar{w}}\left(\frac{rh}{n_r}, \rho_1, \rho_2\right)^\top  \left(I_{d_{\rho_1}}\otimes \frac{2\pi^{d-1}h^d|\mathcal{F}_{\SO(d)}|d_{\rho_1}d_{\rho_2}}{\mathcal{A}(\SS^{d-1})d_{\rho}}\Y^{(\rho_2)}(\e)^*\right) \right]^\top . \label{eq: fpre higher proof 1}\\
        &(\textit{truncating the sum over irreps to a finite set \(\mathcal{F}_{\SO(d)}\)})\nonumber
    \end{align}
    The penultimate  equality follows from the identity ${(\mathbf{A} \otimes \mathbf{B})(\mathbf{C} \otimes \mathbf{D}) = (\mathbf{AC}) \otimes (\mathbf{BD})}$, and the final equality follows from ${(\mathbf{A} \otimes \mathbf{B})^{T} = \mathbf{A}^{T} \otimes \mathbf{B}^{T}}$. Given that \({C^{(\rho,\rho_1, \rho_2)}_{\cdot,m}}^* = \vec{\tilde{C}^{(\rho,\rho_1,\rho_2)}_m}\) and using the identity, 
    \({\displaystyle \left(\mathbf {A} \otimes \mathbf {B} \right)\operatorname {vec} \left(\mathbf {V} \right)=\operatorname {vec} (\mathbf {B} \mathbf {V} \mathbf {A} ^{T})}\), for any  \(1\leq m\leq d_\rho\) and \(1\leq n \leq d_{\rho_1}\), we have
    \begin{align}
         &\left([\fin(\x+\y, \rho_1)]_{\cdot,n}^\top \otimes M^{(\rho_2)}_r(\y)^\top \right){C^{(\rho,\rho_1, \rho_2)}_{\cdot,m}}^* \nonumber\\
        =& \left([\fin(\x+\y, \rho_1)]_{\cdot,n}^\top \otimes M^{(\rho_2)}_r(\y)^\top \right)\vec{\tilde{C}^{(\rho,\rho_1,\rho_2)}_m}\nonumber\\
        =& \vec{ M^{(\rho_2)}_r(\y)^\top \tilde{C}^{(\rho,\rho_1,\rho_2)}_m [\fin(\x+\y, \rho_1)]_{\cdot,n}}\nonumber\\
        =& M^{(\rho_2)}_r(\y)^\top \tilde{C}^{(\rho,\rho_1,\rho_2)}_m [\fin(\x+\y, \rho_1)]_{\cdot,n}\nonumber\\
        =& \left[\tilde{M}_r^{(\rho,\rho_1,\rho_2)}\right]_{m,\cdot} [\fin(\x+\y, \rho_1)]_{\cdot,n}.\label{eq: fpre higher proof 2}\\
        &(\textit{using definition of \(\tilde{M}_r^{(\rho,\rho_1,\rho_2)}\) in equation }\eqref{eq: steerable filter basis higher layers})\nonumber
    \end{align}
    The last equality follows from the definition of \(\tilde{M}_r^{(\rho,\rho_1,\rho_2)}\). Define
    \begin{equation}\label{eq: fpre higher proof 3}
         {w_r^{(\rho,\rho_1,\rho_2)}} := {C^{(\rho,\rho_1, \rho_2)}}^\top  \h{\wbar{w}}\left(\frac{rh}{n_r}, \rho_1, \rho_2\right)^\top  \left(I_{d_{\rho_1}}\otimes \frac{2\pi^{d-1}h^d|\mathcal{F}_{\SO(d)}|d_{\rho_1}d_{\rho_2}}{\mathcal{A}(\SS^{d-1})d_{\rho}}\Y^{(\rho_2)}(\e)\right)\in \CC^{d_\rho\times d_{\rho_1}},
    \end{equation}
    where \(1\leq r\leq n_r\). Using equations \eqref{eq: fpre higher proof 1}, \eqref{eq: fpre higher proof 2} and \eqref{eq: fpre higher proof 3}, for \(1\leq m\leq d_\rho\), the \(m^{\text{th}}\) row of \(\h \fpre\) reduces to
    \begin{align}
         &\left[\h \fpre(\x, \rho)^\top  \right]_{\cdot, m}\nonumber
         \\
        =& \sum_{\y\in \ZZ^d} \sum_{\rho_1, \rho_2 \in \mathcal{F}_{\SO(d)}}\sum_{r=1}^{n_r}\sum_{n=1}^{d_{\rho_1}}\left[{C^{(\rho,\rho_1, \rho_2)}}^\top  \h{\wbar{w}}\left(\frac{rh}{n_r}, \rho_1, \rho_2\right)^\top  \left(I_{d_{\rho_1}}\otimes \frac{2\pi^{d-1}h^d|\mathcal{F}_{\SO(d)}|d_{\rho_1}d_{\rho_2}}{\mathcal{A}(\SS^{d-1})d_{\rho}}\Y^{(\rho_2)}(\e)\right) \right]_{\cdot, n}\nonumber\\
         &\hspace{7cm}\left[\left(\fin(\x+\y, \rho_1)^\top \otimes M^{(\rho_2)}_r(\y)^\top  \right){C^{(\rho,\rho_1, \rho_2)}}^*\right]_{n, m} \nonumber
        \\
        =& \sum_{\y\in \ZZ^d} \sum_{\rho_1, \rho_2 \in \mathcal{F}_{\SO(d)}}\sum_{r=1}^{n_r}\sum_{n=1}^{d_{\rho_1}}\left[{C^{(\rho,\rho_1, \rho_2)}}^\top  \h{\wbar{w}}\left(\frac{rh}{n_r}, \rho_1, \rho_2\right)^\top  \left(I_{d_{\rho_1}}\otimes \frac{2\pi^{d-1}h^d|\mathcal{F}_{\SO(d)}|d_{\rho_1}d_{\rho_2}}{\mathcal{A}(\SS^{d-1})d_{\rho}}\Y^{(\rho_2)}(\e)\right) \right]_{\cdot, n}\nonumber\\
         &\hspace{7cm}\left[\left([\fin(\x+\y, \rho_1)]^T_{\cdot, n}\otimes M^{(\rho_2)}_r(\y)^\top  \right){C^{(\rho,\rho_1, \rho_2)}}^*_{\cdot, m}\right] \nonumber
         \\
        =& \sum_{\y\in \ZZ^d} \sum_{\rho_1, \rho_2 \in \mathcal{F}_{\SO(d)}}\sum_{r=1}^{n_r}\sum_{n=1}^{d_{\rho_1}} \left[w_r^{(\rho,\rho_1,\rho_2)}\right]_{\cdot,n} \left[\tilde{M}_r^{(\rho,\rho_1,\rho_2)}\right]_{m,\cdot} [\fin(\x+\y, \rho_1)]_{\cdot,n}  .\label{eq: fpre higher proof 4}
    \end{align}
    Now if we assume one input and one output channel, the sum over \(n\) collapses and we can assume \(w_r^{(\rho,\rho_1,\rho_2)}\in \CC\) and \(\fin(\x+\y, \rho_1)\in \CC^{d_{\rho_1}}\). Then, equation \eqref{eq: fpre higher proof 4} reduces to 
    \begin{equation*}
        \h \fpre(\x, \rho) = \sum_{\y\in \ZZ^d} \sum_{\rho_1\in \mathcal{F}_{\SO(d)}}\left(\sum_{\rho_2 \in \mathcal{F}_{\SO(d)}}\sum_{r=1}^{n_r} w_r^{(\rho,\rho_1,\rho_2)} {\tilde{M}_r^{(\rho,\rho_1,\rho_2)}}(\y)\right) \fin(\x+\y, \rho_1).
    \end{equation*}
\end{proof}

\subsection{Proof of Proposition \ref{prop: steerable filter}}

\begin{proof}
    Fix \(\y\in \ZZ^d\). Note that, for any \(1\leq m\leq d_\rho\), we have
    \begin{align}
         &\left[R^{-1}\cdot K^{(\rho, \rho_1)}(\y)\right]_{m, \cdot} \nonumber \\
        =& \sum_{\z\in \ZZ^d} \left[K^{(\rho, \rho_1)}(\z)\right]_{m,\cdot}\I(R^{-1}\z, \y)\nonumber\\
        =& \sum_{\z\in \ZZ^d}\sum_{\rho_2 \in \mathcal{F}_{\SO(d)}} \sum_{r=1}^{n_r} w^{(\rho, \rho_1,\rho_2)}_r \left[{\tilde{M}}^{(\rho, \rho_1,\rho_2)}_{r} (\z)\right]_{m,\cdot} \I(R^{-1}\z, \y)  \nonumber\\
         & (\textit{by equation \eqref{eq: steerable filter}}) \nonumber \\
        =& \frac{1}{|\mathcal{F}_{\SO(d)}|}\sum_{\z\in \ZZ^d}\sum_{\rho_2 \in \mathcal{F}_{\SO(d)}} \sum_{r=1}^{n_r} w^{(\rho, \rho_1,\rho_2)}_r  M^{(\rho_2)}_r(\z)^\top  {\tilde{C}}^{(\rho, \rho_1,\rho_2)}_{m}\I(R^{-1}\z, \y)\nonumber\\
         & (\textit{using definition of } \tilde{M}^{(\rho, \rho_1,\rho_2)}_{r} \textit{ in equation } \eqref{eq: steerable filter basis first layer}) \nonumber \\
        =& \frac{1}{|\mathcal{F}_{\SO(d)}|}\sum_{\z\in \ZZ^d}\sum_{\rho_2 \in \mathcal{F}_{\SO(d)}} \sum_{r=1}^{n_r} w^{(\rho, \rho_1,\rho_2)}_r  \sum_{\theta \in \Theta_{n_a}^{d-1}} \frac{r^{d-1}}{n_r^dn_a^{d-1}} \Y^{(\rho_2)}(s(\theta))^\top  \I\left(\frac{rh}{n_r}s(\theta), \z\right) \omega(\theta)\I(R^{-1}\z, \y){\tilde{C}}^{(\rho, \rho_1,\rho_2)}_{m}\nonumber\\
        & (\textit{using definition of } M^{(\rho_2)}_r \textit{ in equation } \eqref{eq: steerable filter basis higher layers}) \nonumber \\
        =& \frac{1}{|\mathcal{F}_{\SO(d)}|}\sum_{\rho_2 \in \mathcal{F}_{\SO(d)}}\sum_{r=1}^{n_r} w^{(\rho, \rho_1,\rho_2)}_r \hspace{-3pt}\sum_{\theta \in \Theta_{n_a}^{d-1}}\hspace{-1.02pt} \frac{r^{d-1}}{n_r^dn_a^{d-1}} \Y^{(\rho_2)}(s(\theta))^\top  \hspace{-3pt} \left(\sum_{\z\in \ZZ^d}  \I\left(\frac{rh}{n_r}s(\theta), \z\right) \I(R^{-1}\z, \y)\right) \omega(\theta){\tilde{C}}^{(\rho, \rho_1,\rho_2)}_{m}.\label{eq: steerable kernel proof 1}
    \end{align}
    Using the fact that \(\Delta(R) = 0\), we conclude
    \begin{align}
        \sum_{\z\in \ZZ^d}  \I\left(\frac{rh}{n_r}s(\theta), \z\right) \I(R^{-1}\z, \y) 
        = \I\left(\frac{rh}{n_r}R^{-1}s(\theta), \y\right).\label{eq: steerable kernel proof 2}
    \end{align}
    Plugging \eqref{eq: steerable kernel proof 2} into \eqref{eq: steerable kernel proof 1}, we get
    \begin{align*}
         &\left[R^{-1}\cdot K^{(\rho, \rho_1)}(\y)\right]_{m,\cdot}\\
        =& \frac{1}{|\mathcal{F}_{\SO(d)}|}\sum_{\rho_2 \in \mathcal{F}_{\SO(d)}} \sum_{r=1}^{n_r}  w^{(\rho, \rho_1,\rho_2)}_r  \sum_{\theta \in \Theta_{n_a}^{d-1}} \frac{r^{d-1}}{n_r^dn_a^{d-1}}\Y^{(\rho_2)}(s(\theta))^\top  \ \I\left(\frac{rh}{n_r}R^{-1}s(\theta), \y\right) \omega(\theta){\tilde{C}}^{(\rho, \rho_1,\rho_2)}_{m} \\
        =& \frac{1}{|\mathcal{F}_{\SO(d)}|}\sum_{\rho_2 \in \mathcal{F}_{\SO(d)}} \sum_{r=1}^{n_r} w^{(\rho, \rho_1,\rho_2)}_r  \sum_{\theta \in \Theta_{n_a}^{d-1}} \frac{r^{d-1}}{n_r^dn_a^{d-1}} \Y^{(\rho_2)}(Rs(\theta))^\top  \I\left(\frac{rh}{n_r}s(\theta), \y\right) \omega(\theta){\tilde{C}}^{(\rho, \rho_1,\rho_2)}_{m} \\
         & (\textit{since }s(\theta)\mapsto Rs(\theta) \textit{ is a bijection on }\Theta_{n_a}^{d-1} \textit{ and \(\omega\) is unchanged by this action})  \\
        =& \frac{1}{|\mathcal{F}_{\SO(d)}|}\sum_{\rho_2 \in \mathcal{F}_{\SO(d)}} \sum_{r=1}^{n_r}   w^{(\rho, \rho_1,\rho_2)}_r   \sum_{\theta \in \Theta_{n_a}^{d-1}} \frac{r^{d-1}}{n_r^dn_a^{d-1}} \Y^{(\rho_2)}(s(\theta))^\top  \rho_2(R)^\top   \I\left(\frac{rh}{n_r}s(\theta), \y\right) \omega(\theta){\tilde{C}}^{(\rho, \rho_1,\rho_2)}_{m}   \\
         & (\textit{by equation }\eqref{eq: spherical harmonics}) \\
        =& \frac{1}{|\mathcal{F}_{\SO(d)}|}\sum_{\rho_2 \in \mathcal{F}_{\SO(d)}} \sum_{r=1}^{n_r}  w^{(\rho, \rho_1,\rho_2)}_r  M^{(\rho_2)}_r(\y)^\top \rho_2(R)^\top {\tilde{C}}^{(\rho, \rho_1,\rho_2)}_{m}  \\
         & (\textit{using definition of } M^{(\rho_2)}_r \textit{ in equation } \eqref{eq: steerable filter basis first layer}) \\
        =& \left[\rho(R)K^{(\rho, \rho_1)}(\y)\rho_1(R)^\dagger\right]_{m,\cdot}  \\
         & (\textit{using Lemma }\ref{lemma: steerable filter})
    \end{align*}
    This completes the proof.
\end{proof}

\subsection{Proof of Theorem \ref{thm: interpolation error}}

\begin{proof}
    Note that, for any \(\x\in \ZZ^d\), we have
    \begin{align}
         &(\t,R) \cdot[K^{(\rho, \rho_1)}\star (\t,R)^{-1}\cdot f](\x)\nonumber\\
        =& \rho(R) \I[K^{(\rho, \rho_1)}\star f]((\t,R)^{-1}\cdot \x ) \nonumber\\
        =& \rho(R) \sum_{\z\in \ZZ^d} \I((\t,R)^{-1}\cdot \x , \z)[K^{(\rho, \rho_1)}\star (\t,R)^{-1}\cdot f](\z) \nonumber\\
         & (\textit{by Definition }\ref{defn: interpolation}(a))\nonumber \\
        =& \rho(R) \sum_{\y, \z\in \ZZ^d}\I((\t,R)^{-1}\cdot \x , \z) K^{(\rho, \rho_1)}(\y) [(\t,R)^{-1}\cdot f](\y+\z)\nonumber\\
         & (\textit{using definition of cross-correlation on }\ZZ^d) \nonumber \\
        =& \rho(R) \sum_{\y, \z\in \ZZ^d}\I((\t,R)^{-1}\cdot \x , \z) K^{(\rho, \rho_1)}(\y) \rho_1(R)^\dagger \I[f]((\t,R)\cdot (\y+\z))\nonumber\\
        =& \rho(R) \sum_{\k, \y, \z\in \ZZ^d}\I((\t,R)^{-1}\cdot \x , \z) K^{(\rho, \rho_1)}(\y) \rho_1(R)^\dagger \I((\t,R)\cdot (\y+\z), \k)f(\k)\nonumber\\
         & (\textit{by Definition }\ref{defn: interpolation}(a))\nonumber \\
        =& \rho(R) \sum_{\k, \y, \z\in \ZZ^d}\I((\t,R)^{-1}\cdot \x , \z-\y) K^{(\rho, \rho_1)}(\y) \rho_1(R)^\dagger \I((\t,R)\cdot \z, \k)f(\k)\nonumber\\
         & (\textit{using the transformation } \z\mapsto \z-\y) \nonumber \\
        =& \rho(R) \sum_{\k, \y, \z\in \ZZ^d}\I((\t,R)^{-1}\cdot \x+\y , \z) K^{(\rho, \rho_1)}(\y) \rho_1(R)^\dagger \I((\t,R)\cdot \z, \k)f(\k)\nonumber\\
        & (\textit{by Definition }\ref{defn: interpolation}(b))\nonumber \\
        =& \sum_{\y\in \ZZ^d}  \rho(R)K^{(\rho, \rho_1)}(\y) \rho_1(R)^\dagger \left(\sum_{\k, \z\in \ZZ^d}\I((\t,R)^{-1}\cdot \x+\y , \z)\I((\t,R)\cdot \z, \k)f(\k)\right)\nonumber\\
        =& \sum_{\y\in \ZZ^d}  \rho(R)K^{(\rho, \rho_1)}(\y) \rho_1(R)^\dagger \tilde{f}(\x,\y), \label{eq: difference 1}
    \end{align}
    where \(\tilde{f}(\x,\y)\) is defined to be the expression inside the brackets. Using Lemma \ref{lemma: steerable filter}, we have
        \begin{equation}
            [\rho(R)K^{(\rho, \rho_1)}(\y) \rho_1(R)^\dagger]_{m, \cdot} =  \frac{1}{|\mathcal{F}_{\SO(d)}|}\sum_{\rho_2 \in \mathcal{F}_{\SO(d)}} \sum_{r=1}^{n_r}w_r^{(\rho,\rho_1,\rho_2)}M^{(\rho_2)}_r(\y)^\top \rho_2(R)^\top \tilde{C}^{(\rho,\rho_1,\rho_2)}_m.  \label{eq: steered kenel}
        \end{equation}
    Suppose the weights \(w_r^{(\rho, \rho_1, \rho_2)}\) are uniformly bounded above by \(C_w>0\). Now, using \eqref{eq: difference 1} and \eqref{eq: steered kenel}, we have
    \begin{align}
         &\left\vert \left\vert (\t,R)\cdot[K\star (\t,R)^{-1}\cdot f](\x) - [K\star f](\x)\right\vert\right\vert_\infty\nonumber
         \\
        =& \sup_{1\leq m\leq d_\rho}\frac{1}{|\mathcal{F}_{\SO(d)}|}\left\vert  \sum_{\y\in \ZZ^d}  \sum_{\rho_2 \in             \mathcal{F}_{\SO(d)}}\sum_{r=1}^{n_r}w_r^{(\rho,\rho_1,\rho_2)}\left(M_{r}^{(\rho_2)}(\y)^\top \rho_2(R)^\top \tilde{C}^{(\rho,\rho_1,\rho_2)}_m \tilde{f}(\x,\y) \right.\right. \nonumber\\
        & \hspace{5cm}\left.- M^{(\rho_2)}_r(\y)^\top \tilde{C}^{(\rho,\rho_1,\rho_2)}_m f(\x+\y)\right) \Bigg\vert\nonumber
        \\
        \leq & \sup_{1\leq m\leq d_\rho}\frac{1}{|\mathcal{F}_{\SO(d)}|}\sum_{\rho_2 \in \mathcal{F}_{\SO(d)}}\sum_{r=1}^{n_r}|w_r^{(\rho,\rho_1,\rho_2)}| \left\vert \sum_{\y\in \ZZ^d} M_{r}^{(\rho_2)}(\y)^\top \rho_2(R)^\top \tilde{C}^{(\rho,\rho_1,\rho_2)}_m \tilde{f}(\x,\y)  \right. \nonumber\\
        & \hspace{5cm} - \sum_{\y\in \ZZ^d} M^{(\rho_2)}_r(\y)^\top \tilde{C}^{(\rho,\rho_1,\rho_2)}_m f(\x+\y)\Bigg\vert \nonumber\\
        & (\textit{using triangle inequality}) \nonumber 
        \\
        \leq & \sup_{1\leq m\leq d_\rho}\frac{C_w}{|\mathcal{F}_{\SO(d)}|}\sum_{\rho_2 \in \mathcal{F}_{\SO(d)}}\sum_{r=1}^{n_r}\left\vert \sum_{\y\in \ZZ^d} M_{r}^{(\rho_2)}(\y)^\top \rho_2(R)^\top  \tilde{C}^{(\rho,\rho_1,\rho_2)}_m\tilde{f}(\x,\y)\right. \nonumber\\
        & \hspace{5cm} - \sum_{\y\in \ZZ^d} M^{(\rho_2)}_r(\y)^\top \tilde{C}^{(\rho,\rho_1,\rho_2)}_m f(\x+\y)\Bigg\vert .\label{eq: bound}\\
        &(\textit{since \(|w_r^{(\rho,\rho_1,\rho_2)}|\leq C_w\)})\nonumber
    \end{align}
    The first term of the difference in \eqref{eq: bound} can be further written as
    \begin{align}
          &  M_{r}^{(\rho_2)}(\y)^\top \rho_2(R)^\top \tilde{C}^{(\rho,\rho_1,\rho_2)}_m\tilde{f}(\x,\y)\nonumber\\
         =& \frac{r^{d-1}}{n_r^dn_a^{d-1}}\sum_{\y\in \ZZ^d}\sum_{\theta\in\Theta_{n_a}^{d-1}} \Y^{(\rho_2)}(s(\theta))^\top  \rho_2(R)^\top  \tilde{C}^{(\rho,\rho_1,\rho_2)}_m\I\left(\frac{rh}{n_r}s(\theta), y\right)\tilde{f}(\x,\y)\,\omega(\theta)\nonumber\\
          & (\textit{using definition of } M^{(\rho_2)}_r \textit{ in equation } \eqref{eq: steerable filter basis first layer})\nonumber
          \\
         =& \frac{r^{d-1}}{n_r^dn_a^{d-1}}\sum_{\y\in \ZZ^d}\sum_{\theta\in\Theta_{n_a}^{d-1}} \left(\rho_2(R) \Y^{(\rho_2)}(s(\theta))\right)^\top  \tilde{C}^{(\rho,\rho_1,\rho_2)}_m \I\left(\frac{rh}{n_r}s(\theta), y\right)\tilde{f}(\x,\y)\,\omega(\theta)\nonumber
         \\
         =& \frac{r^{d-1}}{n_r^dn_a^{d-1}}\sum_{\y\in \ZZ^d}\sum_{\theta\in\Theta_{n_a}^{d-1}} \Y^{(\rho_2)}(Rs(\theta))^\top  \tilde{C}^{(\rho,\rho_1,\rho_2)}_m\I\left(\frac{rh}{n_r}s(\theta), y\right)\tilde{f}(\x,\y) \,\omega(\theta).\label{eq: difference 2}\\
          & (\textit{by equation }\eqref{eq: spherical harmonics}).\nonumber 
    \end{align}
    Define intermediate terms
    \begin{align*}
       & T_1 :=  \frac{r^{d-1}}{n_r^dn_a^{d-1}}\sum_{\theta\in\Theta_{n_a}^{d-1}} \Y^{(\rho_2)}(Rs(\theta))^\top  \tilde{C}^{(\rho,\rho_1,\rho_2)}_m \left(\sum_{\y,\k\in \ZZ^d} \I\left(\frac{rh}{n_r}s(\theta), \y\right) \I(\x+R\y, \k)f(\k)\right) \,\omega(\theta),\\
       & T_2 :=  \frac{r^{d-1}}{n_r^dn_a^{d-1}}\sum_{\theta\in\Theta_{n_a}^{d-1}} \Y^{(\rho_2)}(Rs(\theta))^\top \tilde{C}^{(\rho,\rho_1,\rho_2)}_m  \left(\sum_{\k\in \ZZ^d} \I\left(\frac{rh}{n_r}Rs(\theta), \k\right) f(\x+\k)\right) \,\omega(\theta).
    \end{align*}
    Using the expression in \eqref{eq: difference 2}, the difference between first term of the difference in \eqref{eq: bound} and \(T_1\) can be bounded by
    \begin{align}
         & \left\vert M_{r}^{(\rho_2)}(\y)^\top \rho_2(R)^\top \tilde{C}^{(\rho,\rho_1,\rho_2)}_m\tilde{f}(\x,\y) -  T_1\right\vert\nonumber
         \\
        =& \left\vert \frac{r^{d-1}}{n_r^dn_a^{d-1}}\sum_{\theta\in\Theta_{n_a}^{d-1}} \Y^{(\rho_2)}(Rs(\theta))^\top  \tilde{C}^{(\rho,\rho_1,\rho_2)}_m\sum_{\y\in \ZZ^d} \I\left(\frac{rh}{n_r}s(\theta), \y\right)  \right.\nonumber\\
        & \hspace{3cm}\left(\sum_{\k, \z\in \ZZ^d}\I((\t,R)^{-1}\cdot \x+\y , \z)\I((\t,R)\cdot \z, \k)f(\k) - \mathcal{I}(\x+R\y, \k)\right)f(\k) \, \omega(\theta) \bigg\vert \nonumber 
        \\
        \leq& \frac{r^{d-1}}{n_r^dn_a^{d-1}}\sum_{\theta\in\Theta_{n_a}^{d-1}} \|\Y^{(\rho_2)}(Rs(\theta))\|_2 \|\tilde{C}^{(\rho,\rho_1,\rho_2)}_m\|_2\left\vert \left\vert \sum_{\y\in \ZZ^d} \I\left(\frac{rh}{n_r}s(\theta), \y\right)  \right.\right.\nonumber\\
        & \hspace{3cm} \left(\sum_{\k, \z\in \ZZ^d}\I((\t,R)^{-1}\cdot \x+\y , \z)\I((\t,R)\cdot \z, \k)f(\k) - \mathcal{I}(\x+R\y, \k)\right)f(\k) \bigg\vert \bigg\vert_2 |\omega(\theta)|  \nonumber 
        \\
        & (\textit{using triangle inequality}) \nonumber \\
        =& \frac{r^{d-1}}{n_r^dn_a^{d-1}}\sum_{\theta\in\Theta_{n_a}^{d-1}} \left\vert \left\vert \sum_{\y\in \ZZ^d} \I\left(\frac{rh}{n_r}s(\theta), \y\right) \left(\sum_{\k, \z\in \ZZ^d}\I((\t,R)^{-1}\cdot \x+\y , \z)\I((\t,R)\cdot \z, \k)f(\k)\right.\right.\right. \nonumber\\ 
        &\hspace{10cm} \left.\left.\left.- \mathcal{I}(\x+R\y, \k)\right)f(\k) \right\vert \right\vert_2 |\omega(\theta)|  \nonumber \\
        &(\textit{since \(\Y^{(\rho_2)}\) is normalized to have norm \(1\), and \(C^{(\rho,\rho_1, \rho_2)}\) is a unitary matrix})\nonumber
        \\
        \leq& \frac{r^{d-1}}{n_r^d}\sup_{\theta\in\Theta_{n_a}^{d-1}}  \left\vert \left\vert \sum_{\y\in \ZZ^d} \I\left(\frac{rh}{n_r}s(\theta), \y\right) \sum_{\k\in \ZZ^d}\left(\sum_{\z\in \ZZ^d}\I((\t,R)^{-1}\cdot \x+\y , \z)\I((\t,R)\cdot \z, \k)-  \mathcal{I}(\x+R\y, \k)\right)f(\k) \right\vert \right\vert_2   \nonumber \\
        &(\textit{since \(|\omega(\theta)|\leq 1\), and \(|\Theta_{n_a}^{d-1}| = n_a^{d-1}\)})\nonumber
        \\
        \leq& \frac{r^{d-1}}{n_r^d}\left(\sup_{\theta\in\Theta_{n_a}^{d-1}} \sum_{\y\in \ZZ^d}\left\vert \I\left(\frac{rh}{n_r}s(\theta), \y\right)\right\vert \right)\nonumber \\
        &\hspace{4cm} \left\vert \left\vert \sum_{\k\in \ZZ^d}\left(\sum_{\z\in \ZZ^d}\I((\t,R)^{-1}\cdot \x+\y , \z)\I((\t,R)\cdot \z, \k)-  \mathcal{I}(\x+R\y, \k)\right)f(\k) \right\vert \right\vert_2    \nonumber\\
        & (\textit{using triangle inequality}) \nonumber
        \\
        \leq& \frac{r^{d-1}}{n_r^d}\left(\sup_{\theta\in\Theta_{n_a}^{d-1}} \sum_{\y\in \ZZ^d}\left\vert \I\left(\frac{rh}{n_r}s(\theta), \y\right)\right\vert \right)\nonumber \\
        &\hspace{3cm}\left(\sup_{\k\in \ZZ^d}\left\vert  \sum_{\z\in \ZZ^d}\I((\t,R)^{-1}\cdot \x+\y , \z)\I((\t,R)\cdot \z, \k)-  \mathcal{I}(\x+R\y, \k)\right\vert \right)\sum_{\k\in \ZZ^d}\|f(\k)\|_2     \nonumber \\
        & (\textit{using triangle inequality}) \nonumber 
        \\
        \leq& \frac{r^{d-1}}{n_r^d}\left(\sup_{\theta\in\Theta_{n_a}^{d-1}} \sum_{\y\in \ZZ^d}\left\vert \I\left(\frac{rh}{n_r}s(\theta), \y\right)\right\vert\right)  \Delta(\t,R)\sum_{\k\in \ZZ^d}\|f(\k)\|_2   \nonumber \\
        & (\textit{by Definition }\ref{defn: interpolation error})  \nonumber
        \\
        \leq& \frac{r^{d-1}}{n_r^d}\,C_\I  \Delta(\t,R)\sum_{\k\in \ZZ^d}\|f(\k)\|_2.   \label{eq: bound 1} \\
        &(\textit{for some constant \(C_\I>0\), by Definition }\ref{defn: interpolation}(d))\nonumber
    \end{align}
    Note that, we can simplify \(T_1\) as
    \begin{align}
        T_1
        =&  \frac{r^{d-1}}{n_r^dn_a^{d-1}}\sum_{\theta\in\Theta_{n_a}^{d-1}} \Y^{(\rho_2)}(Rs(\theta))^\top  \tilde{C}^{(\rho,\rho_1,\rho_2)}_m \sum_{\k, \y\in \ZZ^d} \I\left(\frac{rh}{n_r}s(\theta), \y\right)\I(\x+R\y, \k) f(\k) \,\omega(\theta) \nonumber
        \\
        =&  \frac{r^{d-1}}{n_r^dn_a^{d-1}}\sum_{\theta\in\Theta_{n_a}^{d-1}} \Y^{(\rho_2)}(Rs(\theta))^\top  \tilde{C}^{(\rho,\rho_1,\rho_2)}_m\sum_{\k,\y\in \ZZ^d} \I\left(\frac{rh}{n_r}s(\theta), \y\right)\I(R\y, \k-\x) f(\k) \,\omega(\theta) \nonumber\\
        & (\textit{by Definition } \ref{defn: interpolation}(b)) \nonumber 
        \\
        =&  \frac{r^{d-1}}{n_r^dn_a^{d-1}}\sum_{\theta\in\Theta_{n_a}^{d-1}} \Y^{(\rho_2)}(Rs(\theta))^\top \tilde{C}^{(\rho,\rho_1,\rho_2)}_m \sum_{\k,\y\in \ZZ^d} \I\left(\frac{rh}{n_r}s(\theta), \y\right)\I(R\y, \k)f(\x+\k) \,\omega(\theta). \label{eq: difference 3}\\
        & (\textit{using the transformation } \k\mapsto \k+\x) \nonumber
    \end{align}
    Using the expression in \eqref{eq: difference 3}, the difference between \(T_1\) and \(T_2\) can be bounded by
    \begin{align}
        &|T_1-T_2| \nonumber
        \\
        =& \left\vert \frac{r^{d-1}}{n_r^dn_a^{d-1}}\sum_{\theta\in\Theta_{n_a}^{d-1}} \Y^{(\rho_2)}(Rs(\theta))^\top  \tilde{C}^{(\rho,\rho_1,\rho_2)}_m\sum_{\k\in \ZZ^d}\left(\sum_{\y\in \ZZ^d} \I\left(\frac{rh}{n_r}s(\theta), \y\right)\I(R\y, \k) - \right.\right. \nonumber\\
        &\hspace{10cm} \left.\left.  \I\left(\frac{rh}{n_r}Rs(\theta), \k\right)\right)f(\x+\k)\omega(\theta)\right\vert\nonumber
        \\
        \leq& \frac{r^{d-1}}{n_r^dn_a^{d-1}}\sum_{\theta\in\Theta_{n_a}^{d-1}}\| \Y^{(\rho_2)}(Rs(\theta))\|_2 \|\tilde{C}^{(\rho,\rho_1,\rho_2)}_m\|_2\left\vert \left\vert \sum_{\k\in \ZZ^d}\left(\sum_{\y\in \ZZ^d} \I\left(\frac{rh}{n_r}s(\theta), \y\right)\I(R\y, \k)\right.\right.\right. \nonumber\\
        &\hspace{10cm} \left.\left.\left.-  \I\left(\frac{rh}{n_r}Rs(\theta),  \k\right)\right)f(\x+\k)\right\vert\right\vert_2\nonumber |\omega(\theta)|\\
        & (\textit{using triangle inequality}) \nonumber
        \\
        =& \frac{r^{d-1}}{n_r^dn_a^{d-1}}\sum_{\theta\in\Theta_{n_a}^{d-1}}\left\vert \left\vert \sum_{\k\in \ZZ^d}\left(\sum_{\y\in \ZZ^d} \I\left(\frac{rh}{n_r}s(\theta), \y\right)\I(R\y, \k) -  \I\left(\frac{rh}{n_r}Rs(\theta),  \k\right)\right)f(\x+\k)\right\vert\right\vert_2\nonumber |\omega(\theta)|\\
        &(\textit{since \(\Y^{(\rho_2)}\) is normalized to have norm \(1\), and \(C^{(\rho,\rho_1, \rho_2)}\) is a unitary matrix})\nonumber
        \\
        \leq& \frac{r^{d-1}}{n_r^d}\sup_{\theta\in\Theta_{n_a}^{d-1}} \left\vert \left\vert \sum_{\k\in \ZZ^d}\left(\sum_{\y\in \ZZ^d} \I\left(\frac{rh}{n_r}s(\theta), \y\right)\I(R\y, \k) -  \I\left(\frac{rh}{n_r}Rs(\theta), \k\right)\right)f(\x+\k)\right\vert\right\vert_2\nonumber\\
        &(\textit{since \(|\omega(\theta)|\leq 1\), and \(|\Theta_{n_a}^{d-1}| = n_a^{d-1}\)})\nonumber
        \\
        \leq& \frac{r^{d-1}}{n_r^d}\sup_{\theta\in\Theta_{n_a}^{d-1}} \sup_{\k\in \ZZ^d}\left\vert \sum_{\y\in \ZZ^d} \I\left(\frac{rh}{n_r}s(\theta), \y\right)\I(R\y, \k) -  \I\left(\frac{rh}{n_r}Rs(\theta), \k\right)\right\vert \sum_{\k\in \ZZ^d}\|f(\x+\k)\|_2\nonumber\\
        & (\textit{using triangle inequality}) \nonumber\\
         \leq& \frac{r^{d-1}}{n_r^d}\sup_{\theta\in\Theta_{n_a}^{d-1}} \sup_{\k\in \ZZ^d}\left\vert \sum_{\y\in \ZZ^d} \I\left(\frac{rh}{n_r}s(\theta), \y\right)\I(R\y, \k) -  \I\left(\frac{rh}{n_r}Rs(\theta), \k\right)\right\vert\sum_{\k\in \ZZ^d}\|f(\k)\|_2\nonumber\\
        &(\textit{using the transformation }\k\mapsto \k-\x)\nonumber\\
        \leq& \frac{r^{d-1}}{n_r^d}\Delta(R)\sum_{\k\in \ZZ^d}\|f(\k)\|_2.\label{eq: bound 2}\\
        & (\textit{by Definition } \ref{defn: interpolation error}) \nonumber
    \end{align}
    Finally, the difference between \(T_2\) and the second term of the difference in \eqref{eq: bound} can be bounded by
    \begin{align}
          & \left\vert T_2 -  \sum_{\y\in \ZZ^d} M^{(\rho_2)}_r(\y)^Tf(\x+\y)\right\vert\nonumber
          \\
        = & \left\vert \frac{r^{d-1}}{n_r^dn_a^{d-1}}\sum_{\theta\in\Theta_{n_a}^{d-1}} \Y^{(\rho_2)}(Rs(\theta))^\top  \tilde{C}^{(\rho,\rho_1,\rho_2)}_m\left(\sum_{\y\in \ZZ^d} \I\left(\frac{rh}{n_r}Rs(\theta), \y\right) f(\x+\y)\right) \,\omega(\theta) \right. \nonumber\\
        &\hspace{10cm} \left. - \sum_{\y\in \ZZ^d} M^{(\rho_2)}_r(\y)^\top \tilde{C}^{(\rho,\rho_1,\rho_2)}_mf(\x+\y)\right\vert\nonumber
        \\
        = & \left\vert \frac{r^{d-1}}{n_r^dn_a^{d-1}} \sum_{\y\in \ZZ^d}\left( \sum_{\theta\in\Theta_{n_a}^{d-1}}\Y^{(\rho_2)}(Rs(\theta))^\top \I\left(\frac{rh}{n_r}Rs(\theta), \y\right)\omega(\theta) \right. \right. \nonumber\\
        &\hspace{7cm} \left.\left. - \Y^{(\rho_2)}(s(\theta))^\top \I\left(\frac{rh}{n_r}s(\theta),\y\right)\omega(\theta) \right)\tilde{C}^{(\rho,\rho_1,\rho_2)}_mf(\x+\y) \,\right\vert\nonumber
        \\
        \leq& \frac{r^{d-1}}{n_r^dn_a^{d-1}} \sup_{\y\in \ZZ^d}\left\vert\left\vert \sum_{\theta\in\Theta_{n_a}^{d-1}}\Y^{(\rho_2)}(Rs(\theta))^\top \I\left(\frac{rh}{n_r}Rs(\theta), \y\right)\omega(\theta) \right. \right. \nonumber\\
        &\hspace{5.5cm} \left.\left. - \Y^{(\rho_2)}(s(\theta))^\top \I\left(\frac{rh}{n_r}s(\theta),\y\right)\omega(\theta) \right\vert\right\vert_2 \|\tilde{C}^{(\rho,\rho_1,\rho_2)}_m\|_2\sum_{\y\in \ZZ^d}\|f(\x+\y)\|_2\nonumber\\
        & (\textit{using triangle inequality}) \nonumber
        \\
        =& \frac{r^{d-1}}{n_r^dn_a^{d-1}} \sup_{\y\in \ZZ^d}\left\vert\left\vert \sum_{\theta\in\Theta_{n_a}^{d-1}}\Y^{(\rho_2)}(Rs(\theta))^\top \I\left(\frac{rh}{n_r}Rs(\theta), \y\right)\omega(\theta)\right. \right. \nonumber\\
        &\hspace{7cm} \left.\left. - \Y^{(\rho_2)}(s(\theta))^\top \I\left(\frac{rh}{n_r}s(\theta),\y\right)\omega(\theta) \right\vert\right\vert_2\sum_{\y\in \ZZ^d}\|f(\x+\y)\|_2\nonumber\\
        &(\textit{since \(C^{(\rho,\rho_1, \rho_2)}\) is a unitary matrix})\nonumber
        \\
        =& \frac{r^{d-1}}{n_r^dn_a^{d-1}} \sup_{\y\in \ZZ^d}\left\vert\left\vert \sum_{\theta\in\Theta_{n_a}^{d-1}}\Y^{(\rho_2)}(Rs(\theta))^\top \I\left(\frac{rh}{n_r}Rs(\theta), \y\right)\omega(\theta)\right. \right. \nonumber\\
        &\hspace{7cm} \left.\left. - \Y^{(\rho_2)}(s(\theta))^\top \I\left(\frac{rh}{n_r}s(\theta),\y\right)\omega(\theta) \right\vert\right\vert_2\sum_{\y\in \ZZ^d}\|f(\y)\|_2.\label{eq: difference 4}\\
        &(\textit{using the transformation }\y\mapsto \y-\x)\nonumber
    \end{align}
    Since \(\Y^{(\rho_2)}\) is differentiable, it is \(1\)-H\"older. For any fixed \(\y\in \ZZ^d\), \(\I(\cdot, \y)\) is assumed to be \(\alpha\)-H\"older. Therefore, by Lemma \ref{lemma: Holder}, \(\Y^{(\rho_2)}(\cdot)\I(\cdot, \y)\) for any \(\y \in \ZZ^d\) is \(\alpha\)-H\"older. Using Lemma \ref{lemma: sphere epsilon net}, for any \(\y\in\ZZ^d\), we have
    \begin{equation*}
        \frac{1}{n_a^{d-1}}\left\vert\left\vert \sum_{\theta\in\Theta_{n_a}^{d-1}}\Y^{(\rho_2)}(Rs(\theta))^\top \I\left(\frac{rh}{n_r}Rs(\theta), \y\right)\omega(\theta) - \sum_{\theta\in\Theta_{n_a}^{d-1}}\Y^{(\rho_2)}(s(\theta))^\top \I\left(\frac{rh}{n_r}s(\theta),\y\right)\omega(\theta) \right\vert\right\vert_2 \leq C'_{\I}\left(\frac{rh}{n_r}\right)^\alpha d^{\alpha/2}n_a^{-\alpha},
    \end{equation*}
    for some \(C'_{\I}>0\) depending only on the choice of interpolation kernel \(\I\). Therefore, from \eqref{eq: difference 4}, we have
    \begin{align}
        & \left\vert \left\vert T_2 -  \sum_{\y\in \ZZ^d} M^{(\rho_2)}_r(\y)^\top \tilde{C}^{(\rho,\rho_1,\rho_2)}_m f(\x+\y)\right\vert\right\vert_2\nonumber
        \\
        \leq& \frac{r^{d-1}}{n_r^dn_a^{d-1}} \sup_{\y\in \ZZ^d}\left\vert\left\vert \sum_{\theta\in\Theta_{n_a}^{d-1}}\Y^{(\rho_2)}(Rs(\theta))^\top \I\left(\frac{rh}{n_r}Rs(\theta), \y\right)\omega(\theta) \right.\right.\nonumber\\
        &\hspace{5cm} \left.\left.- \Y^{(\rho_2)}(s(\theta))^\top \I\left(\frac{rh}{n_r}s(\theta),\y\right)\omega(\theta) \right\vert\right\vert_2 \sum_{\y\in \ZZ^d}\|f(\y)\|_2\nonumber
        \\
        \leq & \frac{r^{d-1}}{n_r^d} C'_{\I}\left(\frac{rh}{n_r}\right)^\alpha d^{\alpha/2}n_a^{-\alpha}\sum_{\y\in \ZZ^d}\|f(\y)\|_2\nonumber
        \\
        \leq & \frac{r^{d-1}}{n_r^d} C'_{\I}h^{\alpha}n_a^{-\alpha}\sum_{\y\in \ZZ^d}\|f(\y)\|_2.\label{eq: bound 3}
    \end{align}
    Plugging \eqref{eq: bound 1}, \eqref{eq: bound 2} and \eqref{eq: bound 3} into \eqref{eq: bound}, we get
    \begin{align*}
        &\left\vert \left\vert  (\t,R)^{-1}\cdot[K^{(\rho,\rho_1)}\star(\t,R)\cdot f](\x) -  [K^{(\rho,\rho_1)}\star f](\x)\right\vert\right\vert_\infty
        \\
        \leq&  \sup_{1\leq m\leq d_\rho}\frac{C_w}{|\mathcal{F}_{\SO(d)}|}\sum_{\rho_2 \in \mathcal{F}_{\SO(d)}}\sum_{r=1}^{n_r}\left\vert \sum_{\y\in \ZZ^d} M_{r}^{(\rho_2)}(\y)^\top \rho_2(R)^\top  \tilde{C}^{(\rho,\rho_1,\rho_2)}_m\tilde{f}(\x,\y) - \sum_{\y\in \ZZ^d} M^{(\rho_2)}_r(\y)^\top \tilde{C}^{(\rho,\rho_1,\rho_2)}_m f(\x+\y)\right\vert \nonumber
        \\
         \leq&  \sup_{1\leq m\leq d_\rho}\frac{C_w}{|\mathcal{F}_{\SO(d)}|}\sum_{\rho_2 \in \mathcal{F}_{\SO(d)}}\sum_{r=1}^{n_r}\left\vert \sum_{\y\in \ZZ^d} M_{r}^{(\rho_2)}(\y)^\top \rho_2(R)^\top  \tilde{C}^{(\rho,\rho_1,\rho_2)}_m\tilde{f}(\x,\y)- T_1\right\vert + |T_1-T_2|\\
         &\hspace{8cm} + \left\vert T_2 -\sum_{\y\in \ZZ^d} M^{(\rho_2)}_r(\y)^\top \tilde{C}^{(\rho,\rho_1,\rho_2)}_m f(\x+\y)\right\vert \nonumber
         \\
        \leq& \frac{C_w}{|\mathcal{F}_{\SO(d)}|}\sum_{\rho_2 \in \mathcal{F}_{\SO(d)}}\sum_{r=1}^{n_r} \frac{r^{d-1}}{n_r^d}\bigg(C_\I\Delta(\t,R) + \Delta(R) + C'_\I  h^{\alpha} d^{\alpha/2} n_a^{-\alpha} \bigg)\sum_{\y\in \ZZ^d}\|f(\y)\|_2\nonumber
        \\
        \leq& C_w \left( \frac{1}{n_r}\sum_{r=1}^{n_r} \left(\frac{r}{n_r}\right)^{d-1}  \right)\max\left(C_\I,1,C'_{\I} h^{\alpha} d^{\alpha/2}\right)\bigg(\Delta(\t,R) + \Delta(R) + n_a^{-\alpha} \bigg)\sum_{\y\in \ZZ^d}\|f(\y)\|_2
        \\
        \leq& C_w \left( \frac{1}{n_r}\sum_{r=1}^{n_r} 1  \right)\max\left(C_\I,1,C'_{\I} h^{\alpha} d^{\alpha/2}\right)\bigg(\Delta(\t,R) + \Delta(R) + n_a^{-\alpha} \bigg)\sum_{\y\in \ZZ^d}\|f(\y)\|_2
        \\
        =& C_w  \max\left(C_\I,1,C'_{\I}h^{\alpha} d^{\alpha/2}\right)\bigg(\Delta(\t,R) + \Delta(R) + n_a^{-\alpha} \bigg)\sum_{\y\in \ZZ^d}\|f(\y)\|_2.
    \end{align*}
    By setting
    \begin{equation*}
        C :=  C_w \max\left(C_\I,1,C'_{\I}h^{\alpha} d^{\alpha/2}\right),
    \end{equation*}
    for any \(\x \in \ZZ^d\), we have
    \begin{equation*}
        \left\vert \left\vert  (\t,R)\cdot[K^{(\rho,\rho_1)}\star(\t,R)^{-1}\cdot f](\x) -  [K^{(\rho,\rho_1)}\star f](\x)\right\vert\right\vert_\infty \leq C \bigg(\Delta(\t,R)+\Delta(R) + n_a^{-\alpha}\bigg)\, \sum_{\y\in \ZZ^d}\|f(\y)\|_2.
    \end{equation*}
    Taking the supremum over all \(\x \in \ZZ^d\) completes the proof.
\end{proof}
\section{Technical Lemmas}\label{sec: technical}

\begin{lem}\label{lemma: sphere integral}
    Let \(f:\SS^{d-1}\to \CC\) be an integrable function with respect to the surface measure \(\sigma\). Then,
    \begin{equation*}
        \int_{\SS^{d-1}} f(s) \, d\sigma(s) = \int_{\SO(d)} f(R\e)d\mu_{\SO(d)}(R)
    \end{equation*}
    where \(\e\) is a fixed vector and \(\mu_{\SO(d)}\) is the Haar measure on \(\SO(d)\).
\end{lem}

\begin{proof}
    Define the map \(\phi:\SO(d)\to \SS^{d-1}\), by \(\phi(R) = R\e\). Note that \(\phi\) is a continuous function, hence measurable. Define a measure \(\nu\) on \(\SS^{d-1}\) such that for any measurable set \(E\subseteq \SS^{d-1}\),
    \begin{equation*}
        \nu(E) = \int_{\SO(d)} \chi_{E}(\phi(R)) \, d\mu(R).
    \end{equation*}
    Note that \(\nu\) is a valid measure (the pull back measure) on \(\SS^{d-1}\). For any \(R'\in \SO(d)\),
    \begin{align*}
        \chi_{{R'}^{-1}\cdot E}(\phi(R)) = 1
        &\iff  \phi(R)\in {R'}^{-1}\cdot E
        \iff R\e\in {R'}^{-1}\cdot E
        \iff R'R\e \in E\\
        &\iff \phi(R'R)\in E
        \iff \chi_{E}(\phi(R'R)) = 1
    \end{align*}
    Due to left invariance of \(\mu_{\SO(d)}\), \(\nu({R'}^{-1}\cdot E) = \nu(E)\) for any measurable \(E\) and \(R'\in \SO(d)\). Hence, \(\nu\) is a measure on \(\SS^{d-1}\) invariant to the left action of \(\SO(d)\).

    Pick a measurable open \(E\subseteq \SS^{d-1}\) such that, \(\sigma(E) > 0\). If \(\nu(E) = 0\), then by invariance of the measure under rotation, \(\nu(R\cdot E) = 0\) for all \(R\in \SO(d)\). Since the action of \(\SO(d)\) on \(\SS^{d-1}\) is transitive and \(E\) has positive measure under \(\sigma\), any point in \(\SS^{d-1}\) can be covered by an open set of the form \(R\cdot E\) for some \(R\in \SO(d)\). Using compactness, we only need finitely many on these open sets to cover \(\SS^{d-1}\) and all of them have measure \(0\) under \(\nu\). This implies \(\nu \equiv 0\). But, by construction, \(\nu(\SS^{d-1}) = \mu_{\SO(d)}(\SO(d))>0\). This is a contradiction. Hence, \(\sigma(E) > 0\) implies \(\nu(E)>0\).
    
    The above discussion shows that, \(\nu\) is absolutely continuous with respect to \(\sigma\), and hence the Radon-Nikodym derivative exists and using the invariance property of both the measures, we have
    \begin{equation*}
        \frac{d\nu}{d\sigma}(R\e) = \frac{d\nu}{d\sigma}(\e) \quad \quad \forall\, R\in \SO(d).
    \end{equation*}
    A measurable function on \(\SS^{d-1}\), which is invariant to all rotations, can only be the constant function. Hence \(\exists\) \(c>0\) such that, \(d\sigma = cd\nu\). Furthermore, if \(\sigma\) and \(\mu_{\SO(d)}\) are normalized to integrate to \(1\), we have \(c=1\). Therefore, for any measurable set \(E\), we have
    \begin{equation*}
        \int_{\SS^{d-1}}\chi_E(s)\,d\sigma(s) = \sigma(E) = \nu(E) =  \int_{\SO(d)} \chi_{E}(\phi(R)) \, d\mu_{\SO(d)}(R) = \int_{\SO(d)} \chi_{E}(R\e) \, d\mu_{\SO(d)}(R).
    \end{equation*}
     Since it is true for indicator functions, it holds for all measurable functions. This completes the proof.
\end{proof}

\begin{lem}\label{lem: group convolution theorem}
    Suppose \(\G \) be a compact group and \(f,w\in \mathcal{L}_2(\G )\). Then
    \begin{equation*}
        \widehat{f\ast w}(\rho)=\h f(\rho)\,\h w(\rho).
    \end{equation*}
\end{lem}

\begin{proof} 
    Note that,
    \begin{align*}
        \widehat{f\ast w}(\rho) 
        =& \int_\G f\ast w (g) \rho(g) \, d\mu(g)\\
        =& \int_\G\int_{\G } f(g{g'}^{-1})w({g'})  \rho(g) \, d\mu(g') \, d\mu(g) \\
        =& \int_\G\int_{\G } f(g{g'}^{-1})w({g'})  \rho(g{g'}^{-1}{g'}) \, d\mu(g') \, d\mu(g) \\
        =& \int_\G\int_{\G } f(g{g'}^{-1})w(g')  \rho(g{g'}^{-1})\rho(g') \, d\mu(g') \, d\mu(g) \\
        =& \int_\G \int_{\G }f(g{g'}^{-1}) \rho(g{g'}^{-1}) \,\, w(g') \rho(g') \, d\mu(g') \, d\mu(g) \\
        =& \int_\G \int_{\G }f((g{g'})g'^{-1}) \rho((g{g'})g'^{-1})\,\,  w(g') \rho(g') \, d\mu(g') \, d\mu(gg') \\
         & (\textit{using the transformation }g\mapsto g{g'})\\
        =& \int_\G f(g) \rho(g) \, d\mu(g)  \int_{\G } w(g') \rho(g') \, d\mu(g') \\
        =& \h f(\rho)\,\h w(\rho).
    \end{align*}
    The penultimate equality follows from from unimodularity of Haar measure on compact groups, i.e, the Haar measure is both left and right invariant.
\end{proof}

\begin{lem}\label{lem: cartesian to polar}
    Let \(f:\RR^d\to \CC\) be an integrable function. Then
    \begin{equation*}
        \int_{\RR^d} f(\y) \, d\y= \int_{0}^\infty\int_{\SO(d)} f(rR\e)\,r^{d-1}\,d\mu_{\SO(d)}(R)\,dr
    \end{equation*}
    where \(\e\in \RR^d\) is a fixed unit vector.
\end{lem}

\begin{proof}
    Suppose \(\sigma\) is the surface measure of a sphere in \(d\) dimensions. If we apply the transformation \(\y=rs\) with \(r\geq0\) and \(s\in \SS^{d-1}\), then \(d\y = r^{d-1}\,d\sigma(s)\,\,dr\). Therefore,
    \begin{align*}
        \int_{\RR^d} f(\y) \, d\y
        =\int_0^\infty\int_{\SS^{d-1}} f(rs) \, r^{d-1}\,d\sigma(s)\,dr.
    \end{align*}
     Using Lemma \ref{lemma: sphere integral}, we have
    \begin{align*}
        \int_{\RR^d} f(\y) \, d\y
        =\int_0^\infty\int_{\SS^{d-1}} f(rs) \, r^{d-1}\,d\sigma(s)\,dr
        = \int_0^\infty \int_{\SO(d)} f(rRe) \, r^{d-1}\,d\mu_{\SO(d)}(R)\,dr.
    \end{align*}
\end{proof}

\begin{lem}\label{lem: inverse haar}
    For any \(f\in \mathcal{L}_1(\SO(d))\),
    \begin{equation*}
        \int_{\SO(d)} f(R)d\mu_{\SO(d)}(R) = \int f(R^{-1})d\mu_{\SO(d)}(R).
    \end{equation*}
\end{lem}

\begin{proof}
    Choose a measurable set \(E\subseteq \SO(d)\). Let \(\chi_E(R)\) be a simple function which is an indicator function on this set. Define a new measure,
    \begin{equation*}
        \nu(E) : = \int_{\SO(d)} \chi_E(R^{-1}) d\mu_{\SO(d)}(R) = \mu_{\SO(d)}(E^{-1}).
    \end{equation*}
    Now note that, \(\nu(hE) = \mu_{\SO(d)}(E^{-1}h^{-1}) = \mu(E^{-1})\).
    This follows from unimodularity of Haar measure on compact groups, i.e, the Haar measure is both left and right invariant. This shows that, \(\nu\) is also a left invariant measure. By uniqueness of Haar measure, we conclude \(\mu_{\SO(d)}(E)  = c\nu(E)\) for some \(c>0\). By setting \(E = \SO(d)\), we get \(\mu_{\SO(d)}\equiv \nu\). Therefore, the result holds for any indicator function, and hence holds for any measurable function.
\end{proof}

\begin{lem}\label{lemma: SE(d) haar}
    Let \(f:\SE(d)\to \CC\) be an integrable function. Then
    \begin{equation*}
        \int_{\SE(d)} f(\y,\, R) \,d\mu_{\SE(d)}(\y,R) = \int_{\RR^d}\int_{\SO(d)} f(\y, R) \,d\mu_{\SO(d)}(R) \,d\y ,
    \end{equation*}
    where \(\mu_{\SE(d)}\) and \(\mu_{\SO(d)}\) are Haar measures on \(\SE(d)\) and \(\SO(d)\), respectively.
\end{lem}

\begin{proof}
    Choose a measurable set \(E\subseteq \SE(d)\). Let \(\chi_E(\y,R)\) be an indicator function on this set. Define a measure 
    \begin{equation*}
        \nu(E) = \int_{\RR^d}\int_{\SO(d)}  \chi_E(\y,R) \,d\mu_{\SO(d)}(R) \,d\y .
    \end{equation*}
    Note that for any \((\t,R)\in \SE(d)\),
    \begin{align*}
        \nu( (\t,R')^{-1}\cdot E) 
        &=\int_{\RR^d}\int_{\SO(d)} (\t,R')^{-1}\cdot \chi_E(\y,R)   \,d\mu_{\SO(d)}(R) \,d\y \\
        &=\int_{\RR^d}\int_{\SO(d)}  \chi_E((\t,R')(\y,R)) \,d\mu_{\SO(d)}(R) \,d\y \\
        &=\int_{\RR^d}\int_{\SO(d)}  \chi_E(R'\y+\t,R'R) \,d\mu_{\SO(d)}(R) \,d\y \\
        &=\int_{\RR^d}\int_{\SO(d)}  \chi_E(\y,R'R) \,d\mu_{\SO(d)}(R) \,d\y  \\
        &=\int_{\RR^d}\int_{\SO(d)}  \chi_E(\y,R) \,d\mu_{\SO(d)}(R) \,d\y .
    \end{align*}
    The penultimate equality follows from the invariance of Lebesgue measure under rotation and translations and the last equality follows from the left invariance of Haar measure on \(\SO(d)\). This shows that \(\nu\) is a left invariant radon measure on \(\SE(d)\). Using uniqueness of Haar measure on \(\SE(d)\), there exists a constant \(c>0\) such that \(\mu_{\SE(d)} = c\nu\). Now, if we assume the integrals are normalized to the same value on a compact set, we have \(c=1\). Therefore,
    \begin{align*}
        \int_{\SE(d)} \chi_E(\y,R) \,d\mu_{\SE(d)}(\y,R) = \mu_{\SE(d)}(E) = \nu(E) = \int_{\RR^d}\int_{\SO(d)} \chi_E(\y,R) \,d\mu_{\SO(d)}(R) \,d\y 
    \end{align*}
    Since it holds for indicator functions, the equality holds for any measurable function \(f\). This completes the proof.
\end{proof}

\begin{lem}\label{lemma: homogeneous}
    Let \(\G \) be a \(\sigma\)-compact, locally compact group and \(\mathcal{X}\) be a homogeneous space (the action is continuous and transitive). Assume that \(\mathcal{X}\) is first countable, Hausdorff space. Then, \(\mathcal{X}\) is also \(\sigma\)-compact and locally compact.
\end{lem}

\begin{proof}
    Fix \(x_0\in \mathcal{X}\) and define
    \begin{equation*}
        H = \{g\in \G : g\cdot x_0 = x_0\}.
    \end{equation*}  
    It is easy to see, \(H\) is a subgroup of \(\G \). Consider the map \(\alpha:\G \to \mathcal{X}\), defined as \(g\mapsto g\cdot x_0\). Since, the action continuous, we conclude \(\alpha\) is a continuous map. Note that \(H\) is the inverse image of \(\{x_0\}\) under \(\alpha\). Since \(\mathcal{X}\) is Hausdorff, singleton are closed, hence \(H\) is closed subgroup of \(\G \). Let \(\pi:\G \to \G /H\), defined by\(\pi(g) =  gH\) be the quotient map. By definition, \(\pi\) is continuous. This implies that since \(\G \) is a \(\sigma\)-compact, locally compact space, so is \(\G /H\). Furthermore, since \(H\) is closed and \(\G \) is Hausdorff, \(\G /H\) is also Hausdorff.
    
    Consider the map \(q: \G /H \to \mathcal{X}\), with \(q(gH) = g\cdot x_0\). Since the \(\G \) acts transitively on \(\mathcal{X}\), given any \(x\in \mathcal{X}\), there exists \(g\in \G \) such that \(g\cdot x_0 = x\), and consequently \(q(gH) = x\). Now, for any \(g_1,g_2\in \G \), we have
    \begin{align*}
        q(g_1H) = q(g_2H) 
        &\implies g_1\cdot x_0 = g_2\cdot x_0 
        \implies g_2^{-1}g_1\cdot x_0 = x_0 \\ 
        &\implies g_2^{-1}g_1 \in H 
        \implies g_1\in g_2H 
        \implies g_1H = g_2H.
    \end{align*}
    This shows \(q\) is a bijection. Note that, \(\alpha = q \circ \pi\). By the property of quotient maps that for any function \(f:\G /H\to X\), \(f \circ \pi\) is continuous iff \(f\) is continuous, we conclude \(q\) is continuous. Finally, since \(q\) is a map from a \(\sigma\)-compact, locally compact Hausdorff space to a Hausdorff first countable space, \(q\) is also closed, and hence a homeomorphism. \(\G /H\) is \(\sigma\)-compact, locally compact space homeomorphic to a space \(\mathcal{X}\). Hence \(\mathcal{X}\) is also \(\sigma\)-compact, locally compact.
\end{proof}

\begin{lem}\label{lemma: spherical harmonics normalization}
    For the spherical harmonics defined in \eqref{eq: spherical harmonics d dimensions}, the squared Euclidean norm is constant across the sphere. Specifically, for all $s \in \mathbb{S}^{d-1}$,
    \begin{equation}
        \|\Y^{(\ell)}_d(s)\|_2^2 = \frac{\operatorname{dim} \mathcal{H}_d^{(\ell)}}{\mathcal{A}(\mathbb{S}^{d-1})},
    \end{equation}
    where $\mathcal{A}(\mathbb{S}^{d-1}) := \int_{\Theta^{d-1}} \omega(\theta)\, d\theta$ denotes the surface area of the unit sphere $\mathbb{S}^{d-1}$.
\end{lem}
\begin{proof}
    From equation (6.1) of \citet{cohl2023gegenbauer}, the squared Euclidean norm of the spherical harmonics satisfies
    \begin{equation}\label{eq: spherical harmonics normalization cohl}
       \|\Y^{(\ell)}_d(s)\|_2^2 = \frac{(2\ell+d-2)\Gamma(d/2)}{2(d-2)\pi^{d/2}} C_{\ell}^{(d/2-1)}(1),
    \end{equation}
    Furthermore, from equation (4.7.3) of \citet{szeg1939orthogonal}, the Gegenbauer polynomial at unity is given by
    \begin{equation}\label{eq: gegen at 1}
        C^{(\alpha)}_n = \binom{n+2\alpha-1}{n}.
    \end{equation}
    The closed form expression for the surface area of the unit sphere $\mathbb{S}^{d-1}$ is
    \begin{equation}\label{eq: sphere area}
        \mathcal{A}(\SS^{d-1}) =\frac{2\pi^{d/2}}{\Gamma(d/2)}.
    \end{equation}
    Combining \eqref{eq: spherical harmonics normalization cohl}, \eqref{eq: gegen at 1}, and \eqref{eq: sphere area}, we obtain
    \begin{align*}
    \|\Y^{(\ell)}_d&(s)\|_2^2 
    = \frac{(2\ell+d-2)\Gamma(d/2)}{2(d-2)\pi^{d/2}} C_{\ell}^{(d/2-1)}(1) 
    = \frac{(2\ell+d-2)\Gamma(d/2)}{2(d-2)\pi^{d/2}} \binom{\ell+d-3}{\ell} \\
    &= \frac{(2\ell+d-2)\Gamma(d/2)}{2(d-2)\pi^{d/2}} \frac{(\ell+d-3)!}{(d-3)!\ell!} 
    = \frac{(2\ell+d-2)(\ell+d-3)!}{\ell!(d-2)!} \, \frac{\Gamma(d/2)}{2\pi^{d/2}}
    =  \frac{\operatorname{dim} \mathcal{H}_d^{(\ell)}}{\mathcal{A}(\SS^{d-1})}
    \end{align*}
    as claimed.
\end{proof}

\begin{lem}\label{lemma: steerable filter}
    Let \(K^{(\rho,\rho_1)}\) be as defined in \eqref{eq: steerable filter}. Then for any \(R\in \SO(d)\), 
    \begin{equation*}
        \left[\rho(R) K^{(\rho,\rho_1)}(\y) \rho_1(R)^\dagger \right]_{m, \cdot} = \frac{1}{|\mathcal{F}_{\SO(d)}|} \sum_{\rho_2}\sum_{r=1}^{n_r} w_r^{(\rho,\rho_1,\rho_2)} M^{(\rho_2)}_r(\y)^\top  \rho_2(R)^\top  \tilde{C}^{(\rho,\rho_1,\rho_2)}_m .
    \end{equation*}
\end{lem}

\begin{proof} 
    For \(1\leq m\leq d_{\rho}\),
    \begin{align}
         & w_r^{(\rho,\rho_1,\rho_2)} M^{(\rho_2)}_r(\y)^\top  \rho_2(R)^\top  \tilde{C}^{(\rho,\rho_1,\rho_2)}_m \nonumber \\
        =& w_r^{(\rho,\rho_1,\rho_2)}M^{(\rho_2)}_r(\y)^\top  \rho_2(R)^\top  \tilde{C}^{(\rho,\rho_1,\rho_2)}_m  \rho_1(R) \rho_1(R)^\dagger\nonumber\\
        =&  \left(\rho_1(R)^* \otimes M^{(\rho_2)}_r(\y){w_r^{(\rho,\rho_1,\rho_2)}}^\top \right)\vec{ \rho_2(R)^\top  \tilde{C}^{(\rho,\rho_1,\rho_2)}_m \rho_1(R)}\nonumber\\
        =& \left(\rho_1(R)^* \otimes M^{(\rho_2)}_r(\y){w_r^{(\rho,\rho_1,\rho_2)}}^\top \right)\left(\rho_1(R)^\top  \otimes \rho_2(R)^\top \right)\vec{\tilde{C}^{(\rho,\rho_1,\rho_2)}_m}\nonumber\\
        =& \left(\rho_1(R)^* \otimes M^{(\rho_2)}_r(\y){w_r^{(\rho,\rho_1,\rho_2)}}^\top \right)\left(\rho_1(R^{-1})^* \otimes \rho_2(R^{-1})^*\right)\vec{\tilde{C}^{(\rho,\rho_1,\rho_2)}_m}\nonumber\\
        =&  \left(\rho_1(R)^* \otimes M^{(\rho_2)}_r(\y){w_r^{(\rho,\rho_1,\rho_2)}}^\top \right)\left(\rho_1(R^{-1})^* \otimes \rho_2(R^{-1})^*\right)\left[{C^{(\rho,\rho_1,\rho_2)}}^*\right]_{\cdot,m}\nonumber\\
        =&  \left(\rho_1(R)^* \otimes M^{(\rho_2)}_r(\y){w_r^{(\rho,\rho_1,\rho_2)}}^\top \right){C^{(\rho,\rho_1,\rho_2)}}^*\left[\rho(R^{-1})^*\right]_{\cdot,m}\nonumber \\
        =&  \left(\rho_1(R)^* \otimes M^{(\rho_2)}_r(\y){w_r^{(\rho,\rho_1,\rho_2)}}^\top \right){C^{(\rho,\rho_1,\rho_2)}}^*\left[\rho(R)^\top \right]_{\cdot, m}. \label{eq: rotated M rho}
    \end{align}
    The first, fourth and last equality follows from the fact that \(\rho_1(R)\) is unitary. The second and third equality is a result of the identity \(\left(\mathbf {A} \otimes \mathbf {B} \right)\operatorname {vec} \left(\mathbf {V} \right)=\operatorname {vec} (\mathbf {B} \mathbf {V} \mathbf {A} ^{T})\). The fifth equality follows from the definition of \(\tilde{C}^{(\rho,\rho_1,\rho_2)}_m\) (defined in Theorem \ref{thm: higher layer classical}). Finally, the penultimate equality follows from \eqref{eq: CGdef}. Now, for any \(1\leq k\leq d_{\rho}\), we have
    \begin{align}
        &\left[ \left(\rho_1(R)^* \otimes M^{(\rho_2)}_r(\y){w_r^{(\rho,\rho_1,\rho_2)}}^\top \right){C^{(\rho,\rho_1,\rho_2)}}^*\right]_{\cdot,k} \nonumber \\
        =& \left(\rho_1(R)^* \otimes M^{(\rho_2)}_r(\y){w_r^{(\rho,\rho_1,\rho_2)}}^\top \right)\left[{C^{(\rho,\rho_1,\rho_2)}}^*\right]_{\cdot,k}\nonumber\\
        =&  \left(\rho_1(R)^* \otimes M^{(\rho_2)}_r(\y){w_r^{(\rho,\rho_1,\rho_2)}}^\top \right)\vec{{\tilde{C}_k}^{\rho,\rho_1,\rho_2}}\nonumber\\
        =& w_r^{(\rho,\rho_1,\rho_2)}M^{(\rho_2)}_r(\y)^\top {\tilde{C}_k}^{\rho,\rho_1,\rho_2}\rho_1(R)^\dagger\nonumber\\
        =& w_r^{(\rho,\rho_1,\rho_2)}\left[{\tilde{M}_{r}}^{(\rho,\rho_1,\rho_2)}(\y)\right]_{k,\cdot} \rho_1(R)^\dagger\label{eq: rotated M rho 2}.
    \end{align}
    The penultimate equality uses the same tensor product identity as before, and the last equality follows from the definition of \(\tilde{M}_{r}^{(\rho,\rho_1,\rho_2)}\) (defined in Theorem \ref{thm: higher layer classical}). Using \eqref{eq: rotated M rho} and \eqref{eq: rotated M rho 2}, for any \(1\leq m\leq d_{\rho}\), we get
    \begin{align*}
         & \frac{1}{|\mathcal{F}_{\SO(d)}|}\sum_{\rho_2}\sum_{r=1}^{n_r}w_r^{(\rho,\rho_1,\rho_2)}M^{(\rho_2)}_r(\y)^\top  \rho_2(R)^\top  \tilde{C}^{(\rho,\rho_1,\rho_2)}_m  \\
         =& \frac{1}{|\mathcal{F}_{\SO(d)}|}\sum_{\rho_2}\sum_{r=1}^{n_r}\sum_{k=1}^{d_{\rho}} \left[ \left(\rho_1(R)^* \otimes M^{(\rho_2)}_r(\y){w_r^{(\rho,\rho_1,\rho_2)}}^\top \right){C^{(\rho,\rho_1,\rho_2)}}^*\right]_{\cdot,k} \left[\rho(R)^\top \right]_{k, m}\\
        =& \frac{1}{|\mathcal{F}_{\SO(d)}|}\sum_{\rho_2}\sum_{r=1}^{n_r}\sum_{k=1}^{d_{\rho}} w_r^{(\rho,\rho_1,\rho_2)} M^{(\rho_2)}_r(\y)^\top  {\tilde{C}_k}^{\rho,\rho_1,\rho_2} \rho_1(R)^\dagger \left[\rho(R)^\top \right]_{k, m}\\
        =& \sum_{\rho_2}\sum_{r=1}^{n_r}\sum_{k=1}^{d_{\rho}}w_r^{(\rho,\rho_1,\rho_2)}\left[{\tilde{M}_{r}}^{(\rho,\rho_1,\rho_2)}(\y)\right]_{k,\cdot} \rho_1(R)^\dagger\left[\rho(R)^\top \right]_{k, m}\\
        =& \sum_{k=1}^{d_{\rho}}\left[\rho(R)\right]_{m,k} \left[K^{(\rho,\rho_1)}(\y)\right]_{k,\cdot}\rho_1(R)^\dagger\\
        =& \left[\rho(R) K^{(\rho,\rho_1)}(\y) \rho_1(R)^\dagger\right]_{m,\cdot}.
    \end{align*}
    This completes the proof.
\end{proof}

\begin{lem}\label{lemma: Holder}
Let \( f_1, f_2 : X \to \RR^p \) be H\"older continuous functions on a space \(X \subset \RR^{d} \) with exponents \(\alpha_1, \alpha_2\in [0,1]\), respectively. Then the pointwise product \( f_1 f_2 \) is also H\"older continuous with exponent \(\min(\alpha_1, \alpha_2)\).
\end{lem}

\begin{proof} 
    The functions \(f_1\) and \(f_2\) are H\"older continuous means there exists constant \(M_1, M_2, C_1, C_2\) such that, for all \(x,y\in X\),
    \begin{align*}
        & |f_i(x)| \leq M_i\\
        & \|f_i(x)-f_i(y)\|_2 \leq C_i\|x-y\|_2^{\alpha_i}
    \end{align*}
    for \(i=1,2\). From the boundedness of \(f_1\) and \(f_2\) it follows that the pointwise product \(f_1f_2\) is also bounded by \(M_1M_2\). Note that, for any \(x,y\in X\), we have
    \begin{align*}
         \|f_1(x)f_2(x) - f_1(y)f_2(y)\|_2
         =& \|f_1(x)\big(f_2(x) - f_2(y)\big) + f_2(y)\big(f_1(x) - f_1(y)\big)\|_2\\
         \leq& \|f_1(x)\|\,\, \|f_2(x) - f_2(y)\| + \|f_2(y)\|\,\,\|f_1(x) - f_1(y)\|_2\\
         \leq& M_1C_2\|x-y\|_2^{\alpha_2} + M_2C_1\|x-y\|_2^{\alpha_1}\\
         \leq& M_1C_2\|x-y\|_2^{\alpha_2} + M_2C_1\|x-y\|_2^{\alpha_1}\\
         \leq& \max\left(M_1C_2 + M_2C_1, 2M_1M_2\right)\|x-y\|_2^{\min(\alpha_1, \alpha_2)}.
    \end{align*}
    This completes the proof.
\end{proof}

\begin{lem}\label{lemma: sphere epsilon net}
    Let \(f:\SS^{d-1}\to \CC^p\) be an  \(\alpha\)-H\"older continuous function. Let \(\Theta^{d-1}_{n_a}\) and \(\omega\) be as defined in \eqref{eq: sphere grid} and \eqref{eq: sphere integral}, respectively. Then, there exists a constant \(C>0\), depending only on \(f\), such that for any \(R\in \SO(d)\),
    \begin{equation*}
        \frac{1}{n_a^{d-1}}\left\vert\left\vert \sum_{\theta\in \Theta_{n_a}^{d-1}} f(s(\theta))\,\omega(\theta) - \sum_{\theta\in \Theta_{n_a}^{d-1}} f(Rs(\theta))\,\omega(\theta) \right\vert\right\vert_2 \leq Cd^{\alpha}n_a^{-\alpha}
    \end{equation*}
\end{lem}

\begin{proof} 
    Define the sets \(\{\mathcal{V}_\theta\}_{\theta\in \Theta_{n_a}^{d-1}}\) as
    \begin{equation}\label{eq: V theta}
            \mathcal{V}_\theta = \bigg\{\theta'\in [0,\pi]^{d-2}\times [0,2\pi): |\theta'_i - \theta_i| \leq \frac{\pi}{2n_a} \text{ for } 1\leq i\leq d-2, \text{ and } |\theta'_{d-1} - \theta_{d-1}| \leq \frac{\pi}{n_a} \bigg\}.
    \end{equation}

    The sets \(\{\mathcal{V}_\theta\}_{\theta\in \Theta_{n_a}^{d-1}}\) form a partition of the parameter space \([0,\pi]^{d-2}\times [0,2\pi)\) upto measure zero, and 
    \begin{equation}\label{eq: area element}
        \delta := \int_{\mathcal{V}_\theta} d\theta' = \frac{2\pi^{d-1}}{n_a^{d-1}}.
    \end{equation}
    Since the parameterization \(s(\theta)\) in \eqref{eq: sphere parameter} is a differentiable function of \(\theta\) and \(f\) is H\"older continuous, \(f(s(\theta))\) is \(\alpha\)-H\"older in \(\theta\). The function \(\omega\) is also differentiable and bounded in \(\theta\), hence it is \(1\)-H\"older. By Lemma \ref{lemma: Holder}, \(f(s(\theta))w(\theta)\) is \(\alpha\)-H\"older in \(\theta\). Now note that,
    \begin{align}
           &\left\vert\left\vert \sum_{\theta \in \Theta^{d-1}_{n_a}} f(s(\theta))\omega(\theta) \delta -  \int_{\SS^{d-1}} f(s)\,d\sigma(s') \right\vert\right\vert_2\nonumber\\
          =&\left\vert\left\vert \sum_{\theta \in \Theta^{d-1}_{n_a}} f(s(\theta))\omega(\theta) \delta -  \int_{\Theta} f(s(\theta'))\omega(\theta')\,d\theta' \right\vert\right\vert_2\nonumber\\
           &(\textit{by equation }\eqref{eq: sphere integral approx})\nonumber\\
         =&\left\vert\left\vert \sum_{\theta \in \Theta^{d-1}_{n_a}} f(s(\theta))\omega(\theta) \delta- \sum_{\theta \in \Theta^{d-1}_{n_a}}\int_{\mathcal{V}_\theta}f(s(\theta'))\omega(\theta')\,d\theta' \right\vert\right\vert_2\nonumber\\
         & (\textit{since }\mathcal{V}_\theta \textit{ form a partition of }\SS^{d-1}\textit{ upto a measure zero set}) \nonumber\\
         =&\left\vert\left\vert \sum_{\theta \in \Theta^{d-1}_{n_a}} \int_{\mathcal{V}_\theta}f(s(\theta))\omega(\theta)\,d\theta' -  \sum_{\theta \in \Theta^{d-1}_{n_a}} \int_{\mathcal{V}_\theta}f(s(\theta'))\omega(\theta')\,d\theta' \right\vert\right\vert_2\nonumber\\
        =& \left\vert\left\vert \sum_{\theta \in \Theta^{d-1}_{n_a}}  \int_{\mathcal{V}_\theta}f(s(\theta))\omega(\theta) - f(s(\theta'))\omega(\theta')\,d\theta' \right\vert\right\vert_2\nonumber\\
        \leq& \sum_{\theta \in \Theta^{d-1}_{n_a}}\int_{\mathcal{V}_\theta}\left\vert\left\vert f(s(\theta))\omega(\theta) - f(s(\theta'))\omega(\theta')\right\vert\right\vert_2\,d\theta'\nonumber\\
        \leq& \sum_{\theta \in \Theta^{d-1}_{n_a}}\int_{\mathcal{V}_\theta}\frac{C}{2}\left\vert\left\vert \theta - \theta' \right\vert \right\vert_2^{\alpha} \,d\theta'\nonumber\\
         & (\textit{for some constant \(C>0\), since \(f(s(\theta))\omega(\theta)\) is \(\alpha\)-H\"older continuous in \(\theta\)}) \nonumber\\
        \leq& \frac{C}{2}\sum_{\theta \in \Theta^{d-1}_{n_a}} \int_{\mathcal{V}_\theta} \frac{\pi^\alpha d^{\alpha/2}}{n_a^{\alpha}} d\theta'\nonumber\\
         & (\textit{using definition of \(\mathcal{V}_\theta\) in equation } \eqref{eq: V theta}) \nonumber\\
        =& \frac{C d^{\alpha/2}}{2n_a^{\alpha}} \sum_{\theta \in \Theta^{d-1}_{n_a}} \frac{2\pi^{d-1}}{n_a^{d-1}}\nonumber\\
         &(\textit{by equation }\eqref{eq: area element})\nonumber \\
         =& \frac{C d^{\alpha/2}}{2n_a^{\alpha}} n_a^{d-1} \frac{2\pi^{d-1}}{n_a^{d-1}}\nonumber\\
         &(\textit{since } |\Theta^{d-1}_{n_a}| = n_a^{d-1})\nonumber\\
        =& \frac{\pi^{d-1}Cd^{\alpha/2}}{n_a^{\alpha}}\label{eq: sphere approx 1}
    \end{align}
    Since \(f\) is bounded, \([R^{-1}\cdot f]\) is also bounded and hence integrable. Therefore, 
    \begin{align}
         & \left\vert\left\vert \sum_{\theta \in \Theta^{d-1}_{n_a}} f(Rs(\theta))\,\omega(\theta)\delta  - \int_{\SS^{d-1}} f(Rs')\, d\sigma(s') \right\vert\right\vert_2\nonumber\\
        =& \left\vert\left\vert \sum_{\theta \in \Theta^{d-1}_{n_a}} [R^{-1}\cdot f](s(\theta))\,\omega(\theta)\delta  - \int_{\SS^{d-1}} [R^{-1}\cdot f](s')\, d\sigma(s') \right\vert\right\vert_2\nonumber\\
        \leq&\frac{\pi^{d-1}Cd^{\alpha/2}}{n_a^{\alpha}}\label{eq: sphere approx 2}
    \end{align}
    Combining \eqref{eq: sphere approx 1} and \eqref{eq: sphere approx 2} we get
    \begin{align*}
          &\left\vert\left\vert \frac{1}{n_a^{d-1}}\sum_{\theta \in \Theta^{d-1}_{n_a}} f(s(\theta))\omega(\theta) - \frac{1}{n_a^{d-1}}\sum_{\theta \in \Theta^{d-1}_{n_a}} f(Rs(\theta))\omega(\theta) \right\vert\right\vert_2\\
          =&\frac{1}{2\pi^{d-1}}\left\vert\left\vert \sum_{\theta \in \Theta^{d-1}_{n_a}} f(s(\theta))\omega(\theta)\delta - \sum_{\theta \in \Theta^{d-1}_{n_a}} f(Rs(\theta))\omega(\theta)\delta \right\vert\right\vert_2\\
         =&\frac{1}{2\pi^{d-1}}\left\vert\left\vert \sum_{\theta \in \Theta^{d-1}_{n_a}} f(s(\theta))\omega(\theta)\delta - \int_{\SS^{d-1}}f(s')\, d\sigma(s') + \int_{\SS^{d-1}}f(s(\theta'))\, d\sigma(s') - \sum_{\theta \in \Theta^{d-1}_{n_a}} f(Rs(\theta))\omega(\theta)\delta \right\vert\right\vert_2\\
         \leq& \frac{1}{2\pi^{d-1}}\left\vert\left\vert \sum_{\theta \in \Theta^{d-1}_{n_a}} f(s(\theta))\omega(\theta) - \int_{\SS^{d-1}}f(s')\, d\sigma(s')\right\vert\right\vert_2 +\frac{1}{2\pi^{d-1}}\left\vert\left\vert\int_{\SS^{d-1}}f(s')\, d\sigma(s') - \sum_{\theta \in \Theta^{d-1}_{n_a}} f(Rs(\theta))\omega(\theta) \right\vert\right\vert_2\\
         =& \frac{1}{2\pi^{d-1}}\left\vert\left\vert \sum_{\theta \in \Theta^{d-1}_{n_a}} f(s(\theta))\omega(\theta) - \int_{\SS^{d-1}}f(s')\, d\sigma(s')\right\vert\right\vert_2 +\frac{1}{2\pi^{d-1}}\left\vert\left\vert\int_{\SS^{d-1}}f(Rs')\, d\sigma(s') - \sum_{\theta \in \Theta^{d-1}_{n_a}} f(Rs(\theta))\omega(\theta) \right\vert\right\vert_2\\
         &(\textit{since the surface measure \(\sigma\) is invariant to rotations})\\
         \leq&  C d^{\alpha/2}n_a^{-\alpha}.
    \end{align*}
    This completes the proof.
\end{proof}
\section{Auxilliary}
\renewcommand{\thethm}{F\arabic{thm}}
\renewcommand{\theHthm}{F\arabic{thm}}
\setcounter{thm}{0}

\begin{thm}[Haar Theorem]\label{thm: haar}
    There is, up to a positive multiplicative constant, a unique countably additive, nontrivial measure \({\displaystyle \mu }\) on the Borel subsets of \({\displaystyle \G }\) satisfying the following properties:
    \begin{itemize}
        \item The measure \({\displaystyle \mu }\) is left-translation-invariant: \({\displaystyle \mu (gS)=\mu (S)}\) for every \({\displaystyle g\in \G }\) and all Borel sets \({\displaystyle S\subseteq \G }\).
        \item The measure \({\displaystyle \mu }\) is finite on every compact set: \({\displaystyle \mu (K)<\infty }\) for all compact \({\displaystyle K\subseteq \G }\).
        \item The measure \({\displaystyle \mu }\) is outer regular on Borel sets \({\displaystyle S\subseteq \G }\): \({\displaystyle \mu (S)=\inf\{\mu (U):S\subseteq U,U{\text{ open}}\}.}\)
        \item The measure \({\displaystyle \mu }\) is inner regular on open sets \({\displaystyle U\subseteq \G }\): \({\displaystyle \mu (U)=\sup\{\mu (K):K\subseteq U,K{\text{ compact}}\}}\).
    \end{itemize}
\end{thm}

\begin{proof}
    See \S 2.2 of \citet{folland2016course}.
\end{proof}

\begin{thm}[Peter--Weyl Theorem] \label{thm: peter thm}
    Let \(\G \) be a compact group, and let \(\mathcal{L}_2(\G )\) denote the space of square integrable complex valued functions on \(\G \), with respect to the normalized Haar measure.  Then, the matrix coefficients of the irreducible unitary representations of \(\G \) span a dense subspace of \(\mathcal{L}_2(\G )\). Moreover, they form an orthonormal basis of \(\mathcal{L}_2(\G )\) under the \(\mathcal{L}_2\) inner product.
\end{thm}

\begin{proof}
    See \S 5.2 of \citet{folland2016course}.
\end{proof}

\begin{thm}[Riesz--Markov--Kakutani Representation Theorem] \label{thm: riesz rep thm}
    Let \(X\) be a locally compact Hausdorff space. For any bounded linear functional  \(\psi\) on \(C_0(X)\), there is a unique complex Radon measure \(\mu\) on X such that
    \begin{equation*}
        {\displaystyle \psi (f)=\int _{X}f(x)\,d\mu (x),\quad \forall f\in C_{0}(X).}
    \end{equation*}
\end{thm}

\begin{proof}
    See Theorem 7.17 in Chapter 7 of \citet{folland1999real}.
\end{proof}

\begin{thm}[Hahn--Jordan Decomposition Theorem] \label{thm: hahn decomp thm}
    Let \(\mu\) be a finite complex measure on a measurable space. Then, there exist four mutually singular finite positive measures \(\mu_1,\;\mu_2,\;\mu_3,\;\mu_4\) such that \(\mu \;=\; (\mu_1 - \mu_2) \;+\; i\,(\mu_3 - \mu_4)\).
\end{thm}

\begin{proof}
    See Theorem 4.1.5, Corollary 4.1.6 and equation (4) of Chapter 4.1 on page 118 of \citet{cohn2013measure}.
\end{proof}

\begin{thm} [Lebesgue Decomposition Theorem] \label{thm: leb decomp thm}
    Let \(\mu\) and \(\nu\) be \(\sigma\)-finite measures on a measurable space. Then \(\nu\) can be uniquely decomposed into \(\nu = \nu_c + \nu_s\) where \(\nu_c\) is absolutely continuous with respect to \(\mu\) and \(\nu_s\) is singular with respect to \(\mu\).
\end{thm}

\begin{proof}
    See Theorem C in \S 32 of \citet{halmos2013measure}.
\end{proof}

\end{document}